\documentclass[accepted]{article}

\newif\ifshowappendix
\showappendixtrue 

\usepackage{hyperref}
\hypersetup{
    colorlinks=true,
    linkcolor=blue,
    filecolor=blue,
    citecolor=blue,
    urlcolor=blue,
}

\usepackage{xr-hyper}
 \PassOptionsToPackage{numbers,sort&compress}{natbib}
\usepackage{neurips_2025}

\usepackage[table]{xcolor}         
\definecolor{Magenta}{rgb}{1,0,1}
\usepackage[utf8]{inputenc} 
\usepackage[T1]{fontenc}    
\usepackage{hyperref}       
\usepackage{url}            
\usepackage{booktabs}       
\usepackage{amsfonts}       
\usepackage{nicefrac}       
\usepackage{microtype}      
\usepackage{xcolor} 
\usepackage{booktabs}
\usepackage{graphicx} 
\usepackage{amsmath,amsthm,amssymb, mathtools}
\theoremstyle{plain} 
\newtheorem{theorem}{Theorem}[section] 
\newtheorem{lemma}[theorem]{Lemma}
\newtheorem{proposition}[theorem]{Proposition}
\newtheorem{assumption}[theorem]{Assumption}
\newtheorem{corollary}[theorem]{Corollary}

\theoremstyle{definition}
\newtheorem{definition}{Definition}[section]
\newtheorem{example}[theorem]{Example}

\theoremstyle{remark}

\newcommand{\eps}{\varepsilon}
\newcommand{\bQ}{\mathbf{Q}}

\newcommand{\R}{\mathbb{R}}

\newcommand{\X}{\mathcal{X}}

\newcommand{\ngb}{N_\gamma(P_0)}
\newcommand{\Pt}{\tilde{\mathcal{P}}}
\newcommand{\Pth}
{\tilde{\mathcal{P}}_{c_1, c_2, h}}
\newcommand{\Ptm}
{\tilde{\mathcal{P}}_{c_1, c_2, \mu}}
\title{Locally Optimal Private Sampling: Beyond the Global Minimax}

\author{%
  Hrad Ghoukasian \\
  Department of Computing and Software \\
  McMaster University \\
  \texttt{ghoukash@mcmaster.ca}
  \And
  Bonwoo Lee \\
  Department of Mathematical Sciences \\
  Korea Advanced Institute of Science \& Technology \\
  \texttt{righthim@kaist.ac.kr}
  \And
  Shahab Asoodeh \\
  Department of Computing and Software \\
  McMaster University \\
  \texttt{asoodeh@mcmaster.ca}
}

\begin{document}

\maketitle

\begin{abstract}
We study the problem of sampling from a distribution under local differential privacy (LDP). Given a private distribution $P \in \mathcal{P}$, the goal is to generate a single sample from a distribution that remains close to $P$ in $f$-divergence while satisfying the constraints of LDP. This task captures the fundamental challenge of producing realistic-looking data under strong privacy guarantees.
While prior work by Park et al. (NeurIPS'24) focuses on global minimax-optimality across a class of distributions, we take a local perspective. Specifically, we examine the minimax risk in a neighborhood around a fixed distribution~$P_0$, and characterize its exact value, which depends on both~$P_0$ and the privacy level.
Our main result shows that the local minimax risk is determined by the global minimax risk when the distribution class $\mathcal{P}$ is restricted to a neighborhood around $P_0$.
To establish this, we (1) extend previous work from pure LDP to the more general functional LDP framework, and (2) prove that the globally optimal functional LDP sampler yields the optimal local sampler when constrained to distributions near $P_0$. Building on this, we also derive a simple closed-form expression for the locally minimax-optimal samplers which does not depend on the choice of $f$-divergence.
We further argue that this local framework naturally models private sampling with public data, where the public data distribution is represented by $P_0$. In this setting, we empirically compare our locally optimal sampler to existing global methods, and demonstrate that it consistently outperforms global minimax samplers.

\end{abstract}
\vspace{0.05cm}
\section{Introduction}\label{sec: introduction}
\vspace{0.1cm}
Differential privacy (DP) \citep{dwork2006calibrating} is the de facto standard for providing formal privacy guarantees in machine learning. A widely used variant, local differential privacy (LDP) \citep{kasiviswanathan2011can}, enables individuals to randomize their data on their own devices before sharing it with an untrusted curator. LDP ensures that the output of a randomized algorithm is statistically indistinguishable under arbitrary changes to the input, thereby limiting the privacy leakage of an individual’s data. By eliminating the need for a trusted curator, LDP has become particularly appealing to users and companies alike, and has seen wide deployment in practice—including systems developed by Google, Apple, and Microsoft\citep{erlingsson2014rappor, bittau2017prochlo, Apple_Privacy, ding2017collecting}.

Despite this progress, much of the LDP literature assumes that each user holds only a single data point. This assumption is often unrealistic, as modern devices routinely collect large volumes of data—such as images, messages, or time-series records~\citep{husain2020local}. To address this, a line of work known as \emph{user-level DP} assumes that each user holds a dataset of fixed size, with samples drawn i.i.d.\ from an underlying distribution~\citep{kent2024rate, mao2024privshape, bassily2023user, acharya2023discrete, huang2023federated, girgis2022distributed, levy2021learning, ghazi2021user, zhao2024huber}. However, the requirement of equal dataset size across users limits its practicality.

To overcome this limitation, \textit{locally private sampling} was introduced by \citet{husain2020local} and later extended by \citet{park2024exactly}. This framework considers a more general setting where clients possess large datasets of varying sizes and aim to privately release another dataset that closely resembles their original data. Each client's local dataset is modeled as an empirical distribution $P$, from which a single private sample is to be generated. Privacy is enforced by projecting $P$ onto a set of distributions— called a \textit{mollifier}—within a certain radius. While \citet{husain2020local} construct the mollifier in an ad-hoc manner using a reference distribution, \citet{park2024exactly} develop a principled minimax framework to identify the \textit{optimal} mollifier, which corresponds to the worst-case data distribution (namely, degenerate distributions in the discrete case).

In this work, we extend the framework of \citet{park2024exactly} to a public-private setting, where each individual also has access to a public dataset, represented by a distribution $P_0$. From a theoretical perspective, our goal is to move beyond the pessimistic worst-case formulation and instead characterize optimal samplers through a local minimax lens. From a practical standpoint, this framework is motivated by the increasing relevance of personalized collaborative learning~\citep{ozkara2023a,Beirami_personalization,Personalization3}, in which individual models are trained collaboratively, and the public data may reflect (privatized) information shared by other users. Since real-world data distributions are rarely worst-case, our local formulation better captures achievable performance in practice.

Our local minimax framework aims to identify locally private samplers that minimize the $f$-divergence between the true distribution $P$ and the sampling distribution, for all $P$ within a neighborhood $N_\gamma(P_0)$ around $P_0$, parameterized by a constant $\gamma \geq  1$. This neighborhood is defined via one of the strictest notions of distance between distributions: we say $Q \in N_\gamma(P_0)$ if $Q(x) \leq \gamma P_0(x)$ and $P_0(x) \leq \gamma Q(x)$ for all $x$. We will discuss this notion of neighborhood in details in Section \ref{sec: local functional}. 

Interestingly, we show that the local minimax-optimal private sampler under this notion of neighborhood is fully characterized by the global minimax-optimal sampler for a restricted class of data distributions. This relationship is most transparent in the context of functional LDP~\citep{dong2022gaussian}. We therefore first extend the global minimax framework of \citet{park2024exactly} from pure LDP to functional LDP, and then use it to derive the local minimax-optimal samplers.

\textbf{Contributions.} Our main contributions are as follows: 
\begin{itemize}
\item We characterize the exact global minimax risk of private sampling under the functional LDP framework and derive optimal samplers for both continuous and discrete domains. Leveraging properties of functional LDP~\citep{dong2022gaussian}, our results generalize the pure-LDP framework of~\citet{park2024exactly} to approximate and Gaussian LDP (GLDP) notions  (Section~\ref{sec: global functional}). As in their case, our optimal samplers are independent of the choice of $f$-divergence.

\item  We introduce a local minimax formulation to identify samplers that minimize $f$-divergence over all $P \in N_\gamma(P_0)$ for a given $P_0$ and $\gamma \geq 1$. We then show that this minimax risk is fully determined by the global minimax risk when the distribution class is restricted to a neighborhood around \( P_0 \) (Section~\ref{sec: local functional}). Building on our global minimax results, we derive closed-form expressions for the local minimax-optimal functional samplers. We further specialize this to express a pointwise optimal private sampler under pure LDP that achieves the local minimax risk (Section~\ref{sec: local pure}), enabling direct comparison with the global samplers of~\citet{park2024exactly}.

\item We numerically validate our local minimax samplers, showing they consistently—and often substantially—outperform the global samplers of~\citet{park2024exactly} across privacy regimes (Section~\ref{sec: numeric}). Figure~\ref{fig: Laplace pure and functional local vs global} illustrates that our samplers yield distributions significantly closer to the original under both pure LDP and GLDP.\footnote{Experiment code is publicly available at 
\url{https://github.com/hradghoukasian/private_sampling}, 
with reproduction instructions provided in Appendix~\ref{appendix: reproduce}.}

\end{itemize}

\textbf{Related work.}
Locally private sampling was initially introduced by~\citet{husain2020local} and later formalized through a global minimax framework by~\citet{park2024exactly}. In their work, two fundamentally different families of minimax-optimal samplers were identified: one derived from the randomized response mechanism, referred to as the \emph{linear sampler}, and another based on clipping the data distribution, known as the \emph{non-linear sampler}. They rigorously showed that the non-linear sampler is pointwise better than the linear one. In a similar vein, we develop two families of locally minimax-optimal pure LDP samplers and demonstrate that one is pointwise better than the other.

Sampling with public data was recently explored by~\citet{zamanlooy2024locally}, who study minimax-optimal samplers in the presence of a public prior. While related, their work differs in key ways: they focus on a global minimax formulation restricted to discrete domains and linear samplers (i.e., perturbing the data distribution linearly), whereas we study a local minimax problem, support both discrete and continuous settings, and consider arbitrary samplers. Moreover, the role of public data differs between the two approaches. \citet{zamanlooy2024locally} impose it as a hard constraint: their samplers are required to preserve the public prior exactly, ensuring that the public data remains unchanged while privacy is enforced only on the private data. By contrast, we use the public distribution to define a local neighborhood, which in turn yields a local minimax formulation applicable to general sample spaces. A key implication of this distinction is that in their setting, access to public data does not necessarily reduce the minimax risk, whereas in ours it does.

Private sampling under central DP was introduced by~\citet{NEURIPS2021_f2b5e92f}, who proposed algorithms for generating private samples from $k$-ary and product Bernoulli distributions. They showed that sampling can incur significantly lower privacy costs than private learning for some range of parameters. This was extended to broader distribution families~\citep{NEURIPS2023_f4eaa4b8,kamath_sampling,tasnim2025reveal} and to multi-sample settings~\citep{MultiSampling}.

Private sampling has also gained traction in fine-tuning large language models~\citep{flemings2024differentially,flemings2024adaptively}, and is closely related to private generative modeling. However, most prior work in this area focuses on central DP~\citep{raskhodnikova2021differentially,ghazi2023differentially,ebadi2016sampling,abay2018privacy,xie2018differentially,xin2020private,torkzadehmahani2019dp,liu2019ppgan}, assumes single-record LDP~\citep{cunningham2022geopointgan,shibata2023local,zhang2023publishing,gwon2024ldp}, or targets specific estimation tasks~\citep{imola2024metric}.

Finally, we remark that local minimax formulations have also been explored in the context of estimation and distribution learning under DP \cite{chen2024lq, mcmillan2022instance, Duchi_Complexity,asi2020instance,feldman2024instance,rohde2020geometrizing}. In particular, \citet{chen2024lq} analyze discrete distribution estimation under LDP by considering neighborhoods around a reference distribution. Similarly, \citet{mcmillan2022instance} and \citet{Duchi_Complexity} develop locally minimax estimators for one-dimensional parameters in the central and local models, respectively, while \citet{asi2020instance} extend this framework to function estimation. More recently, \citet{feldman2024instance} utilize the notion of neighborhoods around a distribution in the context of nonparametric density estimation.
Although these works focus primarily on estimating statistical functionals or learning distributions, they are conceptually distinct from our setting, which centers on private sampling. Nonetheless, our notion of neighborhood in the local minimax formulation is closely aligned with that of~\citet{feldman2024instance}.



\setlength{\textfloatsep}{0.4cm}
\begin{figure}[t]
  \centering
  \includegraphics[width=0.72\linewidth]{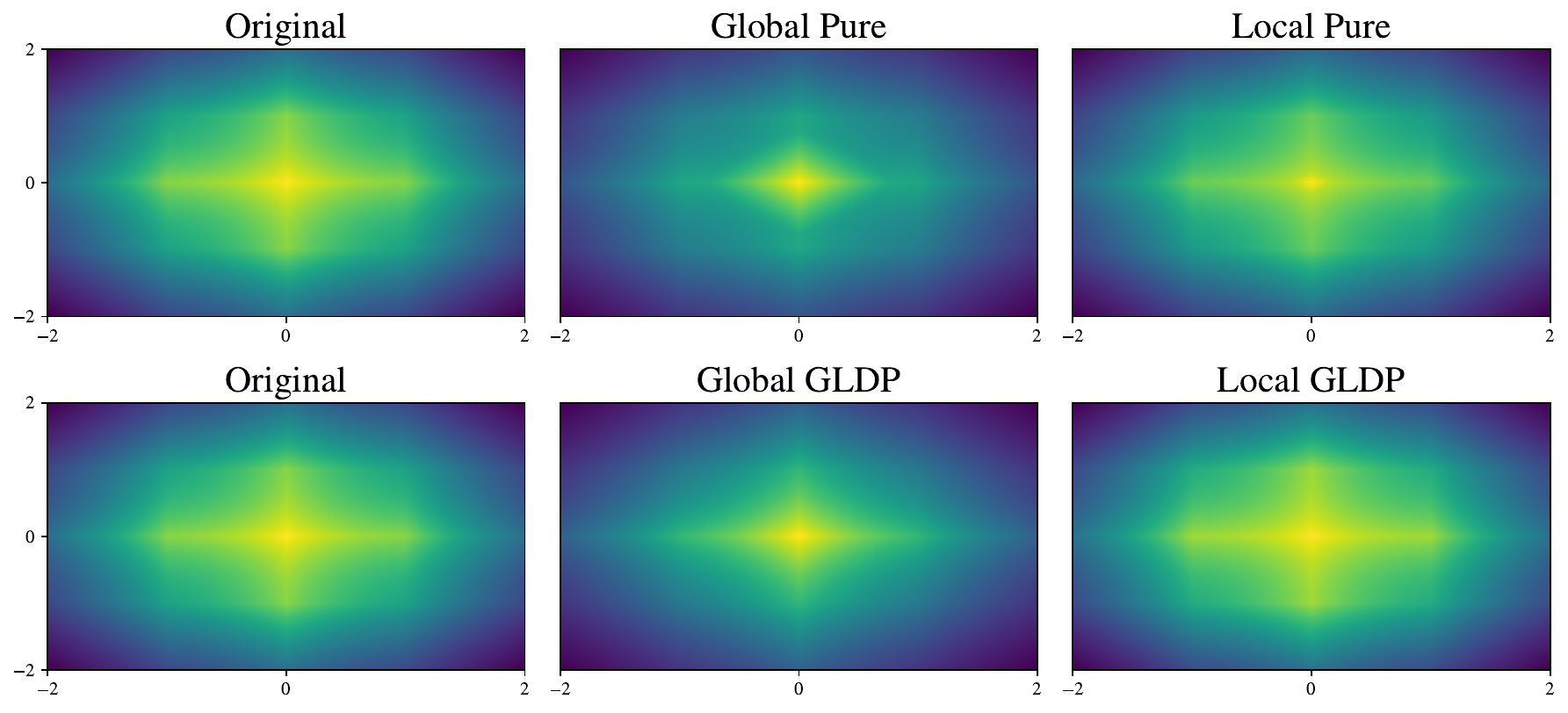}
  \caption{Comparison of global and local minimax-optimal sampler under pure LDP ($\varepsilon = 1$) and $\nu$-GLDP ($\nu = 1.5$). Full details of this experiment are provided in Appendix~\ref{appendix: laplace mixture}. 
}
  \label{fig: Laplace pure and functional local vs global}
\end{figure}

\textbf{Notation.} Let $\mathcal{P}(\mathcal{X})$ denote the set of all probability distributions over the sample space $\mathcal{X}$, and let $\mathcal{C}(\mathbb{R}^n)$ denote the set of all continuous probability distributions on $\mathbb{R}^n$. For $P\in \mathcal{C}(\mathbb{R}^n)$, we denote its probability density function (PDF) by the lowercase letter $p$. For any integer $k \in \mathbb{N}$, we define $[k] \coloneqq \{1, 2, \dots, k\}$. The notation $\mathrm{clip}(x; s_1, s_2) \coloneqq \max\{s_1, \min\{s_2, x\}\}$ refers to the clipping function that bounds $x$ between $s_1$ and $s_2$.
 The convex conjugate of a function $g:D\subseteq\mathbb{R}\to\mathbb{R}$ is the function $g^*:\mathbb{R}\to(-\infty,+\infty]$ defined by $g^*(y) \coloneqq \sup_{x\in D} \{\, xy - g(x) \,\}$ for $y\in\mathbb{R}$.
We denote the \( n \)-dimensional Laplace distribution with mean \( m \in \mathbb{R}^n \) and scale \( b > 0 \) by \( \mathcal{L}(m, b) \), and its density at \( x \in \mathbb{R}^n \) by \( \mathcal{L}(x \mid m, b) \). For the Gaussian case, \( \mathcal{N}(m, \Sigma) \) denotes the \( n \)-dimensional normal distribution with mean \( m \in \R^n\) and covariance matrix \( \Sigma \).

\vspace{0.1cm}
\section{Preliminaries}\label{sec: prilim}
\vspace{0.1cm}
In this section, we formally introduce definitions and concepts necessary for the subsequent sections.

Suppose a client holds a distribution \( P \in \tilde{\mathcal{P}} \subset \mathcal{P}(\mathcal{X}) \) over a sample space \( \mathcal{X} \), and wishes to release a sample that resembles being drawn from \( P \) while preserving privacy. To ensure privacy, the client applies a private sampler \( \mathbf{Q}: \tilde{\mathcal{P}} \to \mathcal{P}(\mathcal{X}) \), which maps the input distribution \( P \) to a privatized distribution \( \mathbf{Q}(P) \in \mathcal{P}(\mathcal{X}) \). The released sample is then drawn from \( \mathbf{Q}(P) \), thereby preserving privacy while approximating the original distribution. We equivalently view \( \mathbf{Q}(P) \) as the conditional distribution \( \mathbf{Q}(\cdot \mid P) \).


\begin{definition}[Approximate LDP]
\label{def:approx-LDP}
Let $\varepsilon \geq 0$ and $\delta \in [0,1]$.  
A sampler $\mathbf{Q}: \tilde{\mathcal{P}} \to \mathcal{P}(\mathcal{X})$ is said to satisfy \emph{$(\varepsilon,\delta)$-local differential privacy ($( \varepsilon,\delta)$-LDP)} if, for every pair of input distributions $P, P' \in \tilde{\mathcal{P}}$ and every measurable set $A \subseteq \mathcal{X}$, we have
\[
\mathbf{Q}(A \mid P) \leq e^{\varepsilon} \, \mathbf{Q}(A \mid P') + \delta.
\]
When $\delta = 0$, $\mathbf{Q}$ is $\eps$-LDP sampler \cite{husain2020local}, also referred to as \emph{pure LDP}. We let $\mathcal{Q}_{\mathcal{X}, \tilde{\mathcal{P}}, \varepsilon}$ and $\mathcal{Q}_{\mathcal{X}, \tilde{\mathcal{P}}, \varepsilon, \delta}$ denote the sets of all $\varepsilon$-LDP and $(\varepsilon, \delta)$-LDP samplers, respectively.

\end{definition}


Definition~\ref{def:approx-LDP} ensures privacy by requiring that the output of the sampler remains indistinguishable under any arbitrary changes in the input. In contrast, functional LDP \cite{dong2022gaussian} interprets privacy through the lens of hypothesis testing: given two inputs, let \(P\) and \(Q\) be the corresponding output distributions. The privacy guarantee is quantified by how hard it is for an adversary to distinguish \(P\) from \(Q\), which can be framed as a binary hypothesis testing problem \citep{wasserman2010statistical,kairouz2015composition}:
\[
H_0:\ \text{output} \sim P 
\quad\text{vs}\quad 
H_1:\ \text{output} \sim Q.
\]
Given a rejection rule \(\phi:\mathcal{X}\rightarrow[0,1]\), the Type I and Type II error rates  are defined as
\[
    a_\phi := \mathbb{E}_{P}[\phi],
    \qquad
    b_\phi := 1 - \mathbb{E}_{Q}[\phi],
\]
respectively.

\begin{definition}[Trade--off function]\label{def:trade-off}
Let \(P\) and \(Q\) be probability measures in $\mathcal{P}(\X)$. 
The \emph{trade--off function}
\(T(P,Q):[0,1]\to[0,1]\) is defined as
\[
   T(P,Q)(u)
   \;=\;
   \inf_{\phi}\,\bigl\{\,b_\phi : a_\phi \le u \bigr\},
\]
where the infimum is taken over all measurable rejection rules~\(\phi\).
\end{definition}

Notice that \(T(P,Q)(u)\) represents the smallest achievable Type II error
subject to Type I error being at most \(u\). A function $g$ is called a trade-off function if $g = T(P,Q)$ for some distributions $P$ and $Q$.

\begin{definition}[Functional LDP \citep{dong2022gaussian,lee2023minimax}]\label{def:gFLDP}
Given a trade-off function $g:[0,1]\to[0,1]$,
a sampler 
\(\mathbf{Q}~:~\tilde{\mathcal{P}}~\to~\mathcal{P}(\mathcal{X})\)
is said to satisfy \emph{\(g\)-functional local differential privacy} (\(g\)-FLDP) if for every pair of input distributions \(P, P' \in \tilde{\mathcal{P}}\) and every
\(u \in [0,1]\),
\[
   T\bigl(\mathbf{Q}(\cdot\mid P),\,
            \mathbf{Q}(\cdot\mid P')\bigr)(u)
   \;\ge\; g(u).
\]
We let $\mathcal{Q}_{\mathcal{X}, \tilde{\mathcal{P}}, g}$ denote the sets of all samplers $\mathbf{Q} : \tilde{\mathcal{P}} \to \mathcal{P}(\mathcal{X})$ satisfying $g$-FLDP.
\end{definition}
Note that functional LDP generalizes both approximate and pure LDP. Specifically, for any privacy parameters \( (\varepsilon, \delta) \), there exists a trade-off function \( g_{\varepsilon,\delta} \) such that \( g_{\varepsilon,\delta} \)-FLDP is equivalent to \( (\varepsilon,\delta) \)-LDP \citep[Proposition 3]{dong2022gaussian}. This function is given by $g_{\varepsilon, \delta}(\theta) = \max \left\{ 0, 1 - \delta - e^{\varepsilon} \theta,\ e^{-\varepsilon}(1 - \delta - \theta) \right\}$. The case of pure LDP corresponds to the special case $g_{\eps} \coloneqq g_{\eps,0}$.
Another notable instance of \( g \)-FLDP is \( \nu \)-GLDP, where the trade-off function is given by
$G_\nu(\theta) = \Phi\left(\Phi^{-1}(1 - \theta) - \nu\right)$,
with \( \Phi \) denoting the standard normal CDF \citep{dong2022gaussian,lee2023minimax}.


In order to quantify utility loss incurred by sampler $\bQ$, we use the \( f \)-divergence between the original distribution \( P \) and the resulting sampling distribution \( \bQ(P) \), defined as follows.
\begin{definition}[$f$-divergence \citep{ali1966general,csiszar1967information}]
\label{def:f-divergence}
Let \( f: (0,\infty) \to \mathbb{R} \) be a convex function with \( f(1) = 0 \). 
The $f$-divergence between $P, Q \in \mathcal{P}(\mathcal{X})$ with $P \ll Q$ is defined as
$
D_f(P \| Q) := \mathbb{E}_Q\Big[ f\left( \frac{dP}{dQ} \right) \Big]
$.

\end{definition}
The above definition can be extended to cases where 
$\displaystyle P\ll Q$ is no longer satisfied as in Appendix~\ref{subsec: f-div}.
 Common examples of \( f \)-divergences include KL divergence, total variation distance, and squared Hellinger distance. We denote total variation distance between $P,Q\in\mathcal{P}(\mathcal{X})$ by \( D_{\mathrm{TV}}(P, Q) \). A key instance in this work is the \( E_\gamma \)-divergence (or hockey-stick divergence), defined as the $f$-divergence with \( f(t) = \max\{t - \gamma, 0\} \) for \( \gamma \geq 1 \)~\citep{sharma2013fundamental}. This divergence is particularly relevant as it defines the local neighborhood in our minimax formulation. We denote $E_\gamma$-divergence of two distributions $P$ and $Q$ as $E_\gamma(P \,\|\, Q)$.

Our most general result, covering both continuous and discrete spaces, requires a technical assumption on distributions, stated below and discussed further in Section~\ref{sec: global functional}.

\begin{definition}[Decomposability \citep{park2024exactly}]\label{def: decomposability}
    Let $\alpha \in (0,1)$ and $t,u \in \mathbb{N}$ with $t > u$. A probability measure $\mu \in \mathcal{P}(\mathcal{X})$ is \emph{$(\alpha, t, u)$-decomposable} if there exist sets $A_1, \dots, A_t \subseteq \mathcal{X}$ such that $\mu(A_i) = \alpha$ for all $i \in [t]$, and for every $x \in \mathcal{X}$, the number of sets $A_i$ containing $x$ is at most $u$.
\end{definition}
Establishing the exact optimal minimax risk in our proofs involves deriving matching upper bounds (the achievability part) and lower bounds (the converse part). For a unified proof over general sample space $\mathcal{X}$, the converse part relies on $(\alpha, t, 1)$-decomposability; namely,  \( \mathcal{X} \) contains \( t \) disjoint measurable subsets, each of measure \( \alpha \).  We note that this mild condition holds naturally for absolutely continuous distributions and for the uniform reference measure in the finite case.


\vspace{0.1cm}
\section{Global minimax-optimal samplers under functional LDP}\label{sec: global functional}
\vspace{0.1cm}


In this section, we study the global minimax samplers under the general framework of functional LDP and derive compact characterizations of the optimal samplers. These global samplers will serve as proxies for constructing local minimax-optimal samplers in subsequent sections.

Using $f$-divergence as the utility measure, we define the global minimax risk as
\begin{equation}\label{eqn: global minimax functional problem}
\mathcal{R}\big(\mathcal{Q}_{\mathcal{X},\tilde{\mathcal{P}},g},\tilde{\mathcal{P}},f\big) \coloneqq \inf_{\bQ \in \mathcal{Q}_{\mathcal{X},\tilde{\mathcal{P}},g}} \hspace{0.1cm} \sup_{P \in \tilde{\mathcal{P}}}\hspace{0.1cm} D_f\big(P \, \| \, \bQ(P)\big).
\end{equation}
As noted by~\citet{park2024exactly}, this risk becomes vacuous in the continuous setting if the distribution class $\tilde{\mathcal{P}}$ is unrestricted. To address this, they consider a restricted class of distributions defined as
\begin{equation}\label{eqn: set of distributions}
   \tilde{\mathcal{P}} = \tilde{\mathcal{P}}_{c_1, c_2, h} := \left\{ P \in \mathcal{C}(\mathbb{R}^n) : c_1 h(x) \leq p(x) \leq c_2 h(x), \quad \forall x \in \mathbb{R}^n \right\}, 
\end{equation}

for a reference function $h : \mathcal{X} \to [0, \infty)$ satisfying $\int_{\mathbb{R}^n} h(x)\,dx < \infty$, and constants $c_2 > c_1 \geq 0$. 
This class captures a broad range of practical distributions, including mixtures of Gaussians and Laplace distributions. See the example below and further discussion in Appendices~\ref{appendix: laplace mixture} and~\ref{appendix: gaussian ring}.

\begin{example}\label{example: Gaussian Laplace universe}
    For $n,k \in \mathbb{N}$,
define $\tilde{\mathcal{P}}_{\mathcal{L}}  = \Big\{ \sum_{i=1}^k \lambda_i \mathcal{L}( m_i, b) :  \lambda_i \geq 0, \, \sum_{i=1}^k \lambda_i = 1, \, \|m_i\|_1 \leq 1 \Big\}$ and $\tilde{\mathcal{P}}_{\mathcal{N}}  = \Big\{ \sum_{i=1}^k \lambda_i \mathcal{N}(m_i, \sigma^2 I_n)  :  \lambda_i \geq 0, \, \sum_{i=1}^k \lambda_i = 1, \, \|m_i\|_2 \leq 1 \Big\}$. Here $m_i, x \in \mathbb{R}^n$ and $I_n$ denotes the $n \times n$ identity matrix.  It can be verified that $\tilde{\mathcal{P}}_{\mathcal{N}} \subseteq \tilde{\mathcal{P}}_{0,\, 1,\, h_{\mathcal{N}}}$ and  $\tilde{\mathcal{P}}_{\mathcal{L}} \subseteq \tilde{\mathcal{P}}_{e^{-1/b},\, e^{1/b},\, h_{\mathcal{L}}}$, where $h_{\mathcal{N}}(x) = \frac{1}{(2\pi \sigma^2)^\frac{n}{2}} \exp\big(-\frac{[\max(0, \|x\|_2 - 1)]^2}{2\sigma^2} \big)$ and  $h_{\mathcal{L}}(x) = \mathcal{L}(x \mid \mathbf{0}, b)$. 
\end{example}
Without loss of generality, we assume that $\int_{\mathbb{R}^n} h(x),dx = 1$ and $0 \leq c_1 < 1 < c_2$. Moreover, for technical reasons---explained in Appendix~\ref{proof: theorem global functional}--- we also assume that \( \frac{c_2 - c_1}{1 - c_1} \in \mathbb{N} \). This assumption is not restrictive and holds in many practical settings, e.g., Gaussian mixtures (see Example~\ref{example: Gaussian Laplace universe}). When it is not satisfied, one can slightly increase \( c_2 \) or decrease \( c_1 \) to enforce the condition, resulting in a superset of the original distribution class, i.e., \( \tilde{\mathcal{P}} \subseteq \tilde{\mathcal{P}}_{c_1, c_2, h} \). 

While many of our examples focus on continuous domains, our analysis applies more generally to \emph{any} sample space~\( \mathcal{X} \), provided the distributions are absolutely continuous with respect to a reference measure~\( \mu \). In this case, we define  $
\tilde{\mathcal{P}}_{c_1, c_2, \mu}
:= \{P\in \mathcal{P}(\mathcal{X}):P \ll \mu, ~c_1 \leq \frac{dP}{d\mu} \leq c_2, ~~ \mu\text{-a.e.}\}$, which we take as the generalized distribution class.
Thus, we make the following assumption. 
\begin{assumption}\label{assumption: norm}
We assume $\int_{\mathbb{R}^n} h(x)\,dx = 1$ and $\mu(\mathcal{X}) = 1$ for the classes $\tilde{\mathcal{P}}_{c_1, c_2, h}$ and $\tilde{\mathcal{P}}_{c_1, c_2, \mu}$, respectively. Additionally, we assume that $0 \leq c_1 < 1 < c_2$, and $\frac{c_2 - c_1}{1 - c_1} \in \mathbb{N}$. 
\end{assumption}
The following proposition shows that an additional condition is required to exclude trivial samplers.
\begin{proposition}\label{prop: triviality functional}
Let $g$ be a trade-off function. Under Assumption~\ref{assumption: norm},  if \( 1 + g^*(-e^\beta) > \frac{(c_2 - c_1 e^\beta)(1 - c_1)}{c_2 - c_1} \) for all \( \beta \geq 0 \), then \( \mathbf{Q}(P) = P \) satisfies \( g \)-FLDP. 
\end{proposition}
In view of this result, we assume throughout that for any given trade-off function $g$ the following condition holds:  \begin{equation}\label{eqn: non-trivial condition condition functional}
    \exists\, \beta \geq 0 \quad \text{such that} \quad 1 + g^*(-e^\beta) \leq \tfrac{(c_2 - c_1 e^\beta)(1 - c_1)}{c_2 - c_1}.
\end{equation}
Now, we are in order to state our first technical result. 
\begin{theorem}\label{thm: continuous global functional} 
Let $\tilde{\mathcal{P}}=\tilde{\mathcal{P}}_{c_1,c_2,h}$. Under Assumption \ref{assumption: norm}, the sampler $\mathbf{Q}_{c_1,c_2,h,g}^\star$ defined as a continuous distribution whose density is given by
\begin{equation}\label{eqn: continuous global functional opt mech}
   q^\star_g(P) (x)\coloneqq \lambda^\star_{c_1,c_2,g} p(x) + \left(1 - \lambda^\star_{c_1,c_2,g}\right) h(x), \quad \lambda^\star_{c_1,c_2,g} = \inf_{\beta \geq 0} 
\tfrac{e^\beta + \frac{c_2 - c_1}{1 - c_1} ( 1 + g^*(-e^\beta)) - 1}
     {(1 - c_1)e^\beta + c_2 - 1},
\end{equation}
satisfies $g$-FLDP, and is minimax-optimal under any \( f \)-divergence, that is
\begin{equation}\label{eqn: continuous global functional optimal value}
    \sup_{P \in \tilde{\mathcal{P}}} D_f\big(P \,\|\, \bQ^\star_{c_1,c_2,h,g}(P)\big)
    = \mathcal{R}\big(\mathcal{Q}_{\mathbb{R}^n, \tilde{\mathcal{P}}, g}, \tilde{\mathcal{P}}, f\big)
    = \frac{1 - r_1}{r_2 - r_1} f(r_2) + \frac{r_2 - 1}{r_2 - r_1} f(r_1),
\end{equation}
for $r_1=\frac{c_1}{1-(1-c_1)\lambda_{c_1,c_2,g}^\star}$ and $r_2=\frac{c_2}{(c_2-1)\lambda_{c_1,c_2,g}^\star+1}$.
\end{theorem}

We note that the non-triviality condition in Proposition~\ref{prop: triviality functional} is equivalent to \( \lambda_{c_1,c_2,g}^\star \leq 1 \), which allows a natural interpretation of the sampler $\mathbf{Q}_{c_1,c_2,h,g}^\star$: It samples from \( P \) with probability \( \lambda_{c_1,c_2,g}^\star \) and from the reference distribution with probability \(1- \lambda_{c_1,c_2,g}^\star \). This probabilistic mixture reflects the privacy–utility trade-off: a larger $ \lambda^\star_{c_1, c_2, g} $ yields outputs closer to $ P $, while sampling from $ h $ introduces the randomness required for privacy.
The same mixture structure yields a minimax-optimal sampler for the discrete case where \( \mathcal{X} = [k] \) and \( \tilde{\mathcal{P}} = \mathcal{P}([k]) \).
\begin{theorem}\label{thm: discrete global functional}
Let \( \tilde{\mathcal{P}} = \mathcal{P}([k]) \) for some \( k \in \mathbb{N} \), and let \( \mu_k \) denote the uniform distribution on \([k]\). Then the sampler defined as
\[
\bQ^\star_{k,g}(P) \coloneqq \lambda^\star_{k,g} P + (1 - \lambda^\star_{k,g}) \mu_k, \quad 
\lambda^\star_{k,g} = \inf_{\beta \geq 0} 
\tfrac{e^\beta + k ( 1 + g^*(-e^\beta)) - 1}
     {e^\beta + k - 1},
\]
satisfies $g$-FLDP, and is minimax-optimal under any \( f \)-divergence, that is
\[
\sup_{P \in \tilde{\mathcal{P}}} D_f\big(P \,\|\, \bQ^\star_{k,g}(P)\big)
= \mathcal{R}\big(\mathcal{Q}_{[k], \tilde{\mathcal{P}}, g}, \tilde{\mathcal{P}}, f\big).
\]
\end{theorem}


Theorems~\ref{thm: continuous global functional} and~\ref{thm: discrete global functional} provide optimal samplers for continuous and finite sample spaces, respectively. The next theorem addresses the same task for general sample spaces under the assumption of decomposability.

\begin{theorem}\label{thm: global functional}
Let \( \tilde{\mathcal{P}} = \tilde{\mathcal{P}}_{c_1, c_2, \mu} \) under Assumption~\ref{assumption: norm}, and suppose \( \mu \) is \( (\alpha, \frac{1}{\alpha}, 1) \)-decomposable with \( \alpha = \frac{1 - c_1}{c_2 - c_1} \). Then the sampler defined as \( \bQ^\star_{c_1,c_2,\mu,g}(P) = \lambda^\star_{c_1,c_2,g} P + (1 - \lambda^\star_{c_1,c_2,g}) \mu \) satisfies $g$-FLDP, and is minimax-optimal with respect to any $f$-divergence, where $\lambda_{c_1,c_2,g}^\star$ is defined in Theorem~\ref{thm: continuous global functional}.

\end{theorem}

Similar to Theorem~\ref{thm: global functional}, all subsequent results can be extended to a general sample space. For clarity of presentation, however, we state them in the continuous setting, while all proofs are given in full generality. 


Having established the optimal sampler for general \( g \)-FLDP, we present specific instantiations of the result for two widely studied settings: pure LDP and $\nu$-GLDP. First, we derive the optimal $\eps$-LDP sampling algorithm by setting $g=g_\eps$ in Theorem~\ref{thm: continuous global functional}, as stated in the following corollary.


\begin{corollary}\label{cor: special case of f-LDP}
Let $\tilde{\mathcal{P}}=\tilde{\mathcal{P}}_{c_1,c_2,h}$. Under Assumption \ref{assumption: norm}, the sampler $\mathbf{Q}_{c_1,c_2,h,g_\eps}^\star$ defined as a continuous distribution whose density is given by
   $q^\star_{g_\eps}(P) (x)\coloneqq \lambda^\star_{c_1,c_2,g_\eps} p(x) + \left(1 - \lambda^\star_{c_1,c_2,g_\eps}\right) h(x)$, with $ \lambda^\star_{c_1,c_2,g_\eps} = \frac{e^\varepsilon - 1}{(1 - c_1)e^\varepsilon + c_2 - 1},
$
satisfies $g_\eps$-FLDP, and is minimax-optimal with respect to any \( f \)-divergence, that is
\begin{equation}\label{eqn: continuous global functional optimal value g_eps}
    \sup_{P \in \tilde{\mathcal{P}}} D_f\big(P \,\|\, \bQ^\star_{c_1,c_2,h,g_\eps}(P)\big)
    = \mathcal{R}\big(\mathcal{Q}_{\mathbb{R}^n, \tilde{\mathcal{P}}, g_\eps}, \tilde{\mathcal{P}}, f\big)
    = \frac{1 - r_1}{r_2 - r_1} f(r_2) + \frac{r_2 - 1}{r_2 - r_1} f(r_1),
\end{equation}
for $
r_1 = c_1 \cdot \frac{(1 - c_1)e^{\eps} + c_2 - 1}{c_2 - c_1 }$, and $
r_2 = \frac{c_2}{c_2 - c_1} \cdot \frac{(1 - c_1)e^{\eps} + c_2 - 1}{e^{\eps}}$.
\end{corollary} 

Indeed, by framing $\eps$-LDP as a special case of functional LDP, Corollary \ref{cor: special case of f-LDP} reproduces the minimax risk of the optimal 
$\eps$-LDP sampler in \citep[Theorem 3.3]{park2024exactly}. 
Next, we focus on the $\nu$-GLDP as a special case of functional LDP. 
\begin{corollary}[$\nu$-GLDP as a special case]\label{cor: special case GDP}
Let \( \tilde{\mathcal{P}} = \tilde{\mathcal{P}}_{c_1, c_2, h} \).
Under Assumption \ref{assumption: norm}, the sampler  \eqref{eqn: continuous global functional opt mech} for $g=G_\nu$ belongs to $\mathcal{Q}_{\R^n, \tilde{\mathcal{P}}, G_\nu}$ and is minimax-optimal with respect to any $f$-divergence, that is,
\begin{equation}
    \sup_{P \in \tilde{\mathcal{P}}} D_f\big(P \, \| \, \bQ^\star_{G_\nu}(P)\big)
    = \mathcal{R}\big(\mathcal{Q}_{\R^n, \tilde{\mathcal{P}},G_\nu}, \tilde{\mathcal{P}}, f\big).
\end{equation}
\end{corollary}
The explicit description of the optimal sampler under $\nu$-GLDP is given in Appendix~\ref{proof: GDP}.

As outlined in the introduction, the global minimax formulation under functional LDP serves as the basis for our local analysis. We now turn to the local minimax setting, first under general  $g$-FLDP (Section~\ref{sec: local functional}) and then under pure LDP (Section~\ref{sec: local pure}), leveraging Theorem \ref{thm: continuous global functional} and Corollary~\ref{cor: special case of f-LDP}, respectively.




\vspace{0.1cm}
\section{Local minimax-optimal sampling under functional LDP}\label{sec: local functional}
\vspace{0.1cm}
In this section, we begin by formalizing the local minimax setting using a neighborhood structure induced by the $E_\gamma$-divergence. We then develop a general framework for identifying local minimax-optimal samplers over arbitrary sample spaces.

The local minimax-optimal sampling problem adopts the same structural form as the global minimax problem, with a key distinction: instead of selecting the worst-case distribution from the entire space of the general class \( \tilde{\mathcal{P}} \), the search is restricted to a neighborhood around a fixed distribution $P_0$. This reference distribution $P_0$  may represent, for example, a publicly available dataset held by an individual or shared collaboratively in a distributed setting. Formally, the local minimax problem under \( g \)-FLDP is defined as:
\begin{equation}\label{eqn: local minimax functional problem} \mathcal{R}\big(\mathcal{Q}_{\mathcal{X},\tilde{\mathcal{P}},g},N_\gamma(P_0),f\big) \coloneqq \inf_{\bQ \in \mathcal{Q}_{\mathcal{X},\tilde{\mathcal{P}},g}} \hspace{0.1cm} \sup_{P \in N_\gamma(P_0)}\hspace{0.1cm} D_f\big(P \, \| \, \bQ(P)\big),
\end{equation}
where $N_\gamma(P_0)$ denotes a neighborhood around $P_0$, recently adopted by  \citet{feldman2024instance}, and is defined for any $\gamma\geq 1$ as
\begin{equation}\label{eqn: ngb}N_\gamma(P_0) \coloneqq  \left\{ P \in \mathcal{P}(\mathcal{X}) :\quad E_\gamma(P \, \| \, P_0) =  E_\gamma(P_0 \, \| \, P) = 0  \right\}. \end{equation}
The zero-\( E_\gamma \) neighborhood generalizes the classical notion of total variation distance. Specifically, for \( \gamma \geq 1 \), \( E_\gamma(P \| Q) = 0 \) implies \( D_{\mathrm{TV}}(P, Q) \leq 1 - \tfrac{1}{\gamma} \) \citep[Proposition~4]{liu2016e_}. Moreover, we can straightforwardly extend this neighborhood to the more general case by replacing \( E_\gamma(P \| P_0) = 0 \) with \( E_\gamma(P \| P_0) \leq \zeta \), meaning the likelihood ratio lies in \( [1/\gamma, \gamma] \) with probability at least \( 1 - \zeta \) under \( P_0 \). This, in turn, enables further generalization to any \( f \)-divergence with a twice differentiable function \( f \), which is a common trick in information theory \citep{polyanskiy2015dissipation,raginsky2016strong,zamanlooy2023strong,asoodeh2024contraction}.
As in the global case (Section~\ref{sec: global functional}), we assume \( \gamma \in \mathbb{N} \) for the same technical reasons. Theorem~\ref{thm: local functional} presents the corresponding locally minimax-optimal sampler, connecting the global minimax analysis of Section~\ref{sec: global functional} with the local setting under functional LDP.


\begin{theorem}\label{thm: local functional}
Let \( P_0 \) be a continuous distribution on \( \mathbb{R}^n \). Define \( N_\gamma(P_0) \)  as in~\eqref{eqn: ngb}, and assume \( N_\gamma(P_0) \subseteq \tilde{\mathcal{P}} \).
Let \( \bQ^\star_g \) denote the global sampler from~\eqref{eqn: continuous global functional opt mech}, instantiated with parameters \( c_1 = \frac{1}{\gamma} \), \( c_2 = \gamma \), and \( h = p_0 \).
We define the sampler \( \bQ^\star_{g, N_\gamma(P_0)} \) as
\[
\bQ^\star_{g, N_\gamma(P_0)}(P) \coloneqq
\begin{cases}
\bQ^\star_g(P), & \text{if } P \in N_\gamma(P_0), \\[1pt]
\bQ^\star_g(\hat{P}), & \text{otherwise},
\end{cases}
\]
where \( \hat{P} \in N_\gamma(P_0) \) is a distribution that minimizes \(  D_f(P \,\|\, P') \) over all $P' \in N_\gamma(P_0)$.
Then \( \bQ^\star_{g, N_\gamma(P_0)} \in \mathcal{Q}_{\mathbb{R}^n, \tilde{\mathcal{P}}, g} \) and is locally minimax-optimal under any \( f \)-divergence, that is
\[
\sup_{P \in N_\gamma(P_0)} D_f\big(P \,\|\, \bQ^\star_{g, N_\gamma(P_0)}(P)\big)
= \mathcal{R}\big(\mathcal{Q}_{\mathbb{R}^n, \tilde{\mathcal{P}}, g}, N_\gamma(P_0), f\big).
\]
\end{theorem}
Theorem~\ref{thm: local functional} establishes that the local minimax-optimal sampler for the problem in~\eqref{eqn: local minimax functional problem} coincides with the global solution in~\eqref{eqn: global minimax functional problem} when the universe is restricted to the $N_\gamma$ neighborhood around \( P_0 \). In this sense, the global solution presented in Theorem~\ref{thm: continuous global functional} directly yields the local solution in Theorem~\ref{thm: local functional}. 

By instantiating Theorem~\ref{thm: local functional} with $g = g_\eps$, we obtain an $\eps$-LDP sampler that is locally minimax-optimal. However, as we demonstrate in the next section, this sampler can still be improved in a \textit{pointwise} sense.



\vspace{0.1cm}
\section{Local minimax-optimal sampling under pure LDP}\label{sec: local pure}
\vspace{0.1cm}

In this section, we turn our attention to pure-LDP sampling and develop a new framework that is fundamentally different from the approach presented in the previous section. Our goal is to identify a sampler that is \textit{pointwise} better than the one obtained from Theorem~\ref{thm: local functional} with $g = g_\eps$.  The construction of such sampler is presented in the next theorem.
\setlength{\textfloatsep}{0.6cm}
\begin{figure}[t]
  \centering
  \includegraphics[width=0.8\linewidth]{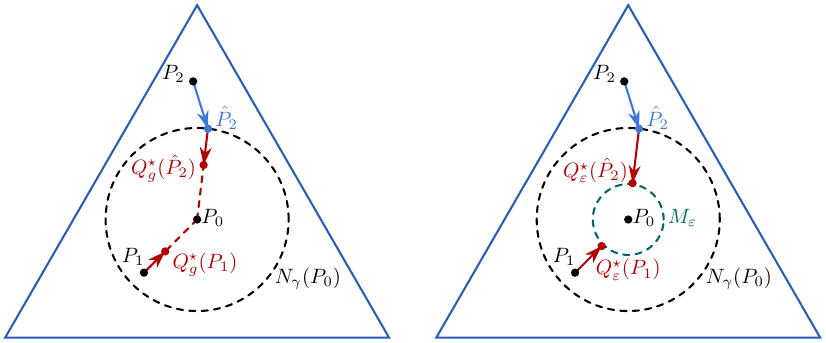}
  \caption{
    Illustrations of the optimal samplers described in Theorems~\ref{thm: local functional} and~\ref{thm: local pure}. 
    \textbf{Left:} $\bQ^\star_{g, N_\gamma(P_0)}$ for functional LDP. 
    \textbf{Right:} $\bQ^\star_{\varepsilon, N_\gamma(P_0)}$ for pure LDP. Following \citep[Proposition 3.4]{park2024exactly}, $M_\eps$  is defined as $M_\eps \coloneqq \{ Q \in \mathcal{C}(\mathbb{R}^n): \hspace{0.1cm} \frac{\gamma + 1}{\gamma + e^\varepsilon} \, p_0(x)  \leq q(x) \leq \frac{(\gamma + 1)e^\varepsilon}{\gamma + e^\varepsilon} \, p_0(x), \hspace{0.1cm}\forall x \in \mathbb{R}^n \}$.} 
  \label{fig:local-mechanisms}
\end{figure}
\begin{theorem}\label{thm: local pure}
Let \( P_0 \) be a continuous distribution on \( \mathbb{R}^n \) with density \( p_0\) and $N_\gamma(P_0)$ be the neighborhood around $P_0$ defined in \eqref{eqn: ngb}. Assuming \( N_\gamma(P_0) \subseteq \tilde{\mathcal{P}} \), let 
\( \bQ^\star_\varepsilon \in \mathcal{Q}_{\mathbb{R}^n, N_\gamma(P_0), \varepsilon} \) be  a sampler with \( \bQ^\star_\varepsilon(P) \)  having density
\[
q(x) 
= \mathrm{clip}\left(
\frac{1}{r_P} \, p(x); \,
\frac{\gamma + 1}{\gamma + e^\varepsilon} \, p_0(x),\ 
\frac{(\gamma + 1)e^\varepsilon}{\gamma + e^\varepsilon} \, p_0(x)
\right),
\]
where \( r_P \) being the normalizing constant ensuring \( \int_{\mathbb{R}^n} q(x) \, dx = 1 \).
Furthermore, define the extended sampler 
\( \bQ^\star_{\varepsilon, N_\gamma(P_0)} \) by
\[
\bQ^\star_{\varepsilon, N_\gamma(P_0)}(P) \coloneqq
\begin{cases}
\bQ^\star_\varepsilon(P), & \text{if } P \in N_\gamma(P_0), \\[1pt]
\bQ^\star_\varepsilon(\hat{P}), & \text{otherwise},
\end{cases}
\]
where \( \hat{P} \in N_\gamma(P_0) \) is a distribution that minimizes \(  D_f(P \,\|\, P') \) over all $P' \in N_\gamma(P_0)$.
Then, we have \( \bQ^\star_{\varepsilon, N_\gamma(P_0)} \in \mathcal{Q}_{\mathbb{R}^n, \tilde{\mathcal{P}}, \varepsilon} \) and 
$
\sup_{P \in N_\gamma(P_0)} D_f\big(P \,\|\, \bQ^\star_{\varepsilon, N_\gamma(P_0)}(P)\big)
= \mathcal{R}\big(\mathcal{Q}_{\mathbb{R}^n, \tilde{\mathcal{P}}, \varepsilon}, N_\gamma(P_0), f\big)$.
\end{theorem}


Similar to Theorem~\ref{thm: local functional}, the above theorem defines a sampler that distinguishes between the cases \( P \in N_\gamma(P_0) \) and \( P \notin N_\gamma(P_0) \). In the latter case, both theorems are aligned in that the sampling distribution depends on \( \hat{P} \), the closest distribution to \( P \) within \( N_\gamma(P_0) \). However, for the case \( P \in N_\gamma(P_0) \), the two constructions differs fundamentally. Theorem~\ref{thm: local functional} assigns a \textit{linear} sampler, whereas Theorem~\ref{thm: local pure} carefully constructs a non-linear sampler tailored to this regime. In fact, this non-linear sampler is obtained by projecting $P$ onto a mollifier $M_\eps$; see proof of the theorem in Appendix~\ref{proof: local pure} and Figure~\ref{fig:local-mechanisms} for a visual representation. This projection is carried out through a clipping operation, which guarantees that the resulting distribution $ \mathbf{Q}^\star_\varepsilon(P) $ satisfies the LDP constraint while maintaining a close match to the original distribution $ P $ in terms of $ f $-divergence. This approach yields a nonlinear transformation of $ P $ that carefully balances privacy constraints and utility preservation.
It can be verified---using a mutatis mutandis adaptation of \cite[Proposition C.5]{park2024exactly}--- that the non-linear sampler pointwise outperforms the linear one under the same privacy guarantees; that is, for any $\eps\geq 0$ and $P\in N_\gamma(P_0)$
\[
D\big(P \,\|\, \bQ^*_{\eps}(P)\big) \leq D\big(P \,\|\, \bQ^*_{g_\eps}(P)\big).
\]
In fact, the non-linear sampler outperforms the linear one because it is not only worst-case optimal but also instance-optimal: for each input distribution $P$ in our distribution class and any $f$-divergence $D_f$, it minimizes $D_f(P || \mathbf{Q})$ over all admissible $\mathbf{Q}$ satisfying $\varepsilon$-LDP (see Proposition \ref{prop: pointwise optimal Park et al.}). In contrast, the linear sampler uses a fixed transformation for all input distributions. This instance-specific optimization explains the empirical advantage of the non-linear sampler. As a result, the overall sampler in Theorem~\ref{thm: local pure} is pointwise better than the one in Theorem~\ref{thm: local functional}. This intuition is formalized in the result below.


\begin{proposition}\label{prop: pointwise}
Let \( \tilde{\mathcal{P}} \) be the global universe and \( P_0 \) a probability measure on \( \mathbb{R}^n \), with \( N_\gamma(P_0) \subseteq \tilde{\mathcal{P}} \) for $\ngb$ defined as in~\eqref{eqn: ngb}. Let \( \bQ^\star_{\varepsilon,\ngb} \) be the optimal $\eps$-LDP sampler from Theorem~\ref{thm: local pure}, and  \( \bQ^\star_{g_\varepsilon,\ngb} \) the instantiation of Theorem~\ref{thm: local functional} with \( g = g_\varepsilon \). Then, for all \( P \in N_\gamma(P_0) \),
\[
D_f(P \,\|\, \bQ^\star_{\varepsilon,\ngb}) \leq D_f(P \,\|\, \bQ^\star_{g_\varepsilon,\ngb}).
\]
\end{proposition}
\vspace{0.1cm}
\section{Numerical results}\label{sec: numeric}
\vspace{0.1cm}
In this section, we numerically compare the worst-case \( f \)-divergence of our locally minimax sampler against the globally optimal sampler of~\citep{park2024exactly} under the $\varepsilon$-LDP setting. Experiments span both finite and continuous domains, evaluating KL divergence, total variation distance, and squared Hellinger distance across \( \varepsilon \in \{0.1, 0.5, 1, 2\} \). Additional results under \( \nu \)-GLDP appear in Appendix~\ref{appendix: experiment GDP local global discrete} and~\ref{appendix: experiment GDP local global continuous}. In addition, we report complementary results for the continuous case under pure LDP in Appendix~\ref{appendix: high dimension laplace}.

\subsection{Finite sample space}\label{experiment: finite}

We compare local and global minimax-optimal samplers in the finite setting \( \mathcal{X} = [k] \), where the global class is \( \tilde{\mathcal{P}}_{\textsf{global}} = \mathcal{P}([k]) = \tilde{\mathcal{P}}_{0, k, \mu_k} \) for the uniform measure \( \mu_k \). The local neighborhood is defined as \( \tilde{\mathcal{P}}_{\textsf{local}} = \mathcal{N}_\gamma(\mu_k) = \tilde{\mathcal{P}}_{\frac{1}{\gamma},\, \gamma,\, \mu_k} \) with \( \gamma = \frac{k}{2} - 1 \), ensuring that all local neighborhood assumptions are satisfied and \( \mathcal{N}_\gamma(\mu_k) \subseteq \mathcal{P}([k]) \).
Indeed, both global and local worst-case \( f \)-divergences can be computed in closed form: the global risk is given by~\citep[Theorem 3.1]{park2024exactly}, while the local risk is derived from a finite-space version of Theorem~\ref{thm: local pure} (see Appendix~\ref{appendix: finite space}). Figure~\ref{fig: finite worst-case k = 20} compares the two for \( k = 20 \); additional results for other \( k \) values are provided in Appendix~\ref{appendix: finite space}. The local minimax-optimal sampler consistently outperforms the global minimax-optimal sampler across all \( \varepsilon \) and \( f \)-divergences.

\setlength{\textfloatsep}{0.4cm}
\begin{figure}[ht]
  \centering
  \includegraphics[width=0.8\linewidth]{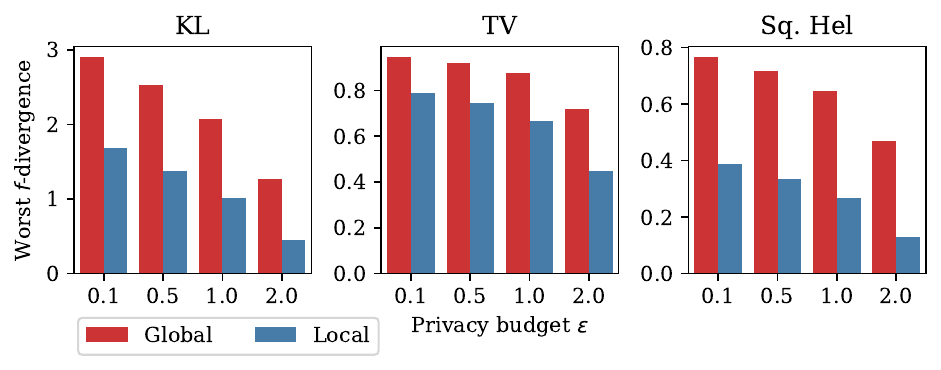}
  \caption{Theoretical worst-case 
$f$-divergences of global and local minimax samplers under the pure LDP setting with uniform reference distribution $\mu_k$ over finite space ($k = 20$).
}
  \label{fig: finite worst-case k = 20}
\end{figure}

\subsection{Continuous sample space}\label{subsec: experiment 1d continuous}

In the continuous setting with \( \mathcal{X} = \mathbb{R} \), we fix the universe \( \tilde{\mathcal{P}}_{\textsf{local}} \) and evaluate the empirical worst‐case \( f \)-divergence over 100 randomly generated client distributions \( P_1, \ldots, P_{100} \in \tilde{\mathcal{P}}_{\textsf{local}} \).  Each \( P_j \) represents a client and is constructed as a mixture of a random number of one–dimensional Laplace components with scale parameter \( b = 1 \); the complete procedure for generating these distributions is described in detail in Appendix~\ref{appendix: details of continuous}.

We define the local and global universes as  
\( \tilde{\mathcal{P}}_{\textsf{local}} = \tilde{\mathcal{P}}_{1/3,\, 3,\, h_{\mathcal{L}}} \) and  
\( \tilde{\mathcal{P}}_{\textsf{global}} = \tilde{\mathcal{P}}_{1/9,\, 9,\, h_{\mathcal{L}}} \), 
where \( h_{\mathcal{L}} \) is the density of a Laplace distribution with mean zero and scale \( b = 1 \).  
Since \( \tilde{\mathcal{P}}_{\textsf{local}} \subseteq \tilde{\mathcal{P}}_{\textsf{global}} \), every client distribution \( P_j \) also belongs to the global universe; thus the input distributions are identical for both samplers.
We evaluate the empirical worst‐case divergence of each sampler as  
\( \max_{j \in [100]} D_f\!\bigl(P_j \,\|\, \mathbf{Q}(P_j)\bigr) \).  
The local minimax sampler follows Theorem~\ref{thm: local pure}, while the global sampler is the optimal sampler from~\citep[Theorem 3.3]{park2024exactly}.  
As illustrated in Figure~\ref{fig: continuous worst-case}, the local sampler consistently achieves lower worst‐case \( f \)-divergence than its global counterpart across all \( f \)-divergences and privacy levels \( \varepsilon \).

\setlength{\textfloatsep}{0.4cm}
\begin{figure}[ht]
  \centering
  \includegraphics[width=0.8\linewidth]{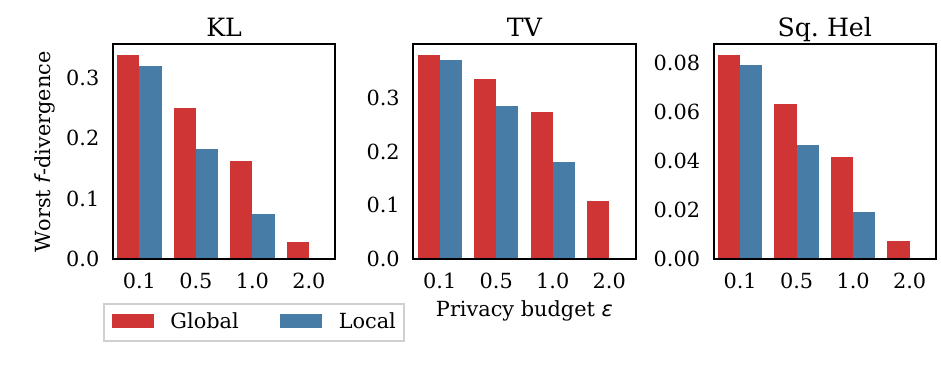}
  \caption{Empirical worst-case 
$f$-divergences of global and local minimax samplers under the pure LDP setting, over 100 experiments on a 1-D Laplace mixture.
}
  \label{fig: continuous worst-case}
\end{figure}

\section{Limitations and future direction}\label{sec: conclusion}
\vspace{0.1cm}

Our main contribution is the development of a minimax-optimal private sampler, validated through a series of experiments on synthetic data. However, the absence of evaluations on high-dimensional real-world datasets limits our understanding of its practical utility. A key reason for this is the computational complexity of our sampler (in particular the clipping function), which poses challenges for scalability in high-dimensional settings. Addressing this limitation---by developing more efficient implementations of our samplers or approximate variants, potentially leveraging techniques such as MCMC--- is an important direction for future work.

Beyond computational considerations, another important direction for future work is to generalize the formulation of the local neighborhood. The neighborhood in our local minimax formulation is defined using the $E_\gamma$-divergence: $P\in N_\gamma(P_0)$ if  \( E_\gamma(P \| P_0) = E_\gamma(P_0 \| P) = 0 \), as in \citep{feldman2024instance}. A natural extension is 
$N_{\gamma,\zeta}(P_0) \coloneqq \left\{ P \in \mathcal{P}(\mathcal{X}) :\; E_\gamma(P \| P_0) \leq \zeta \;\text{and}\; E_\gamma(P_0 \| P) \leq \zeta \right\}$. 
The use of $E_\gamma$-divergence in our formulation is particularly useful, as it provides a foundation for broader neighborhood definitions. In particular, it can be extended to neighborhoods based on general \( f \)-divergences with twice differentiable \( f \), as  
$
N_{f,\zeta}(P_0) \coloneqq \left\{ P \in \mathcal{P}(\mathcal{X}) :\; D_f(P \| P_0) \leq \zeta \;\text{and}\; D_f(P_0 \| P) \leq \zeta \right\}
$ \citep{polyanskiy2015dissipation,raginsky2016strong,zamanlooy2023strong,asoodeh2024contraction}.  
This approach provides a principled strategy for characterizing the local minimax solution over the broad and flexible class of distributional neighborhoods $N_{f,\zeta}(P_0)$. We leave this generalization for future work. 
Furthermore, this work focuses on the setting where each client releases a single sample. Extending to the case of multiple samples per client is a natural direction for future work.


 \newpage

 \section*{Acknowledgments} 
 This work was supported in part by the Natural Sciences and Engineering Research Council of Canada (NSERC).
\bibliography{reference}






\newpage

\newpage
\ifshowappendix

\appendix 

\textbf{Roadmap of Appendix:}
The appendix is organized as follows. A more complete notation table—complementing the abbreviated summary in the main body—is provided in Section~\ref{appendix: notation}. Supplementary definitions and theorems used in the proofs of our results appear in Section~\ref{appendix: Supplementary Definitions and Theorems}. Section~\ref{sec: proofs prilim} presents preliminary results that support the proofs of our main theorems, which are then detailed in Section~\ref{sec: proofs main}. Section~\ref{sec: appendix experimental setup and edditional experiments} describes our experimental setup and provides additional numerical results. 
Finally, instructions for reproducing our experiments is given in Sections~\ref{appendix: reproduce}.

\section{Notation table}\label{appendix: notation}
Table~\ref{tab: notation-table} provides a more comprehensive summary of the notation used throughout the paper, complementing the abbreviated overview in the main body for the reader’s convenience.
\begin{table}[ht]
  \caption{Important notations used in the paper}
  \label{tab: notation-table}
  \centering
  \begin{tabular}{ll}
    \toprule
    Notation & Description \\
    \midrule
    $\mathcal{P}(\mathcal{X})$ & Set of all probability distributions over sample space $\mathcal{X}$ \\
    $\mathcal{C}(\mathbb{R}^n)$ & Set of all continuous probability distributions on $\mathbb{R}^n$ \\
    $p$ & Probability density function (PDF) of $P \in \mathcal{C}(\mathbb{R}^n)$ \\
    $[k]$ & Shorthand for $\{1, 2, \dots, k\}$ for $k \in \mathbb{N}$ \\
    $\mu_k$ & uniform distribution over finite sample space $[k]$ \\
    $\mathrm{clip}(x; s_1, s_2)$ & Clipping function: $\max\{s_1, \min\{s_2, x\}\}$ \\
    $g^*$ & Convex conjugate of a function $g$ \\
    $\mathcal{L}(m, b)$ & $n$-dimensional Laplace distribution with mean $m \in \mathbb{R}^n$ and scale $b > 0$ \\
    $\mathcal{L}(x \mid m, b)$ & Density of the Laplace distribution at $x \in \mathbb{R}^n$ \\
    $\mathcal{N}(m, \Sigma)$ & $n$-dimensional Gaussian with mean $m \in \mathbb{R}^n$ and covariance matrix $\Sigma$ \\

    $\mathcal{N}_{m, \Sigma}[x]$ & Density of the Gaussian distribution at $x \in \mathbb{R}^n$ \\
    $\mathcal{Q}_{\mathcal{X},\tilde{\mathcal{P}},\varepsilon}$ & Set of $\varepsilon$-LDP samplers: $\bQ: \tilde{\mathcal{P}} \to \mathcal{P}(\mathcal{X})$ \\
    $\mathcal{Q}_{\mathcal{X},\tilde{\mathcal{P}},\varepsilon,\delta}$ & Set of $(\varepsilon,\delta)$-LDP samplers: $\bQ: \tilde{\mathcal{P}} \to \mathcal{P}(\mathcal{X})$ \\
    $\mathcal{Q}_{\mathcal{X},\tilde{\mathcal{P}},g}$ & Set of $g$-FLDP samplers: $\bQ: \tilde{\mathcal{P}} \to \mathcal{P}(\mathcal{X})$ \\

    $D_f^B(\lambda_1 \,\|\, \lambda_2)$ & $f$-divergence between Bernoulli distributions with $\Pr(1) = \lambda_1$ and $\lambda_2$ \\
    \bottomrule
  \end{tabular}
\end{table}

\section{Supplementary definitions and theorems
}\label{appendix: Supplementary Definitions and Theorems}

\subsection{\texorpdfstring{General definition and special cases of $f$-divergence}{f-divergence}}\label{subsec: f-div} 

In this appendix, we review the general definition and special cases of $f$-divergences which are important in this work. 

\begin{definition}[General case of $f$-divergence \citep{sason2016f}]\label{def: general f-div}
    Let \( P \) and \( Q \) be probability measures, and let \( \mu \) be a dominating measure of \( P \) and \( Q \) (i.e., \( P, Q \ll \mu \); e.g., \( \mu = P + Q \)), and let \( p := \frac{dP}{d\mu} \) and \( q := \frac{dQ}{d\mu} \). The \textit{f}-divergence from \( P \) to \( Q \) is given, independently of \( \mu \), by
\begin{equation}
    D_f(P\|Q) := \int q f\left( \frac{p}{q} \right) d\mu,
\end{equation}
\noindent where
\begin{align}
    f(0) &:= \lim_{t \to 0^+} f(t), \\
    0 f\left( \frac{0}{0} \right) &:= 0, \\
    0 f\left( \frac{a}{0} \right) &:= \lim_{t \to 0^+} t f\left( \frac{a}{t} \right) = a \lim_{u \to \infty} \frac{f(u)}{u}, \quad a > 0.
\end{align}
\end{definition}

Popular examples of $f$-divergences include KL divergence ($f(t) = t\log t$), total variation distance ($f(t) =\frac{1}{2} |t-1|$), squared Hellinger distance ($f(t) = (\sqrt{t} - 1)^2$), and $\chi^2$-divergence ($f(t) = (t-1)^2$). We also formally define $E_\gamma$-divergence for  \( f(t) = \max\{t - \gamma, 0\} \) for \( \gamma \geq 1 \).

\begin{definition}[$E_\gamma$-divergence \citep{sharma2013fundamental}]
\label{def:Egamma-divergence}
Given $\gamma \geq 1$, for two probability measures $P,Q \in \mathcal{P}(\mathcal{X})$ with $P \ll Q$, the
\emph{$E_\gamma$-divergence} is
\[
    E_\gamma(P \,\|\, Q)
    \;:=\;
    \mathbb{E}_Q\!\left[
        \max\!\Bigl\{\tfrac{dP}{dQ}-\gamma,\,0\Bigr\}
    \right]= \frac{1}{2} \int \left| \mathrm{d}P - \gamma\,\mathrm{d}Q \right| - \frac{1}{2} (\gamma - 1).
\]
\end{definition}

For a more extensive list of 
$f$-divergences and their properties, we refer readers to~\citet{sason2016f}.  

\subsection{Measure-theoretic assumptions}

The appendices assume familiarity with basic measure theory and real analysis. 
Throughout the main paper and the appendices, we adopt the following conventions. For every sample space \( \mathcal{X} \), a \(\sigma\)-algebra on \( \mathcal{X} \) is assumed to be given implicitly. Unless stated otherwise, we use the discrete \(\sigma\)-algebra when \( \mathcal{X} \) is finite, and the Borel \(\sigma\)-algebra when \( \mathcal{X} = \mathbb{R}^n \). 
Whenever we refer to a “subset” of \( \mathcal{X} \), we mean a measurable subset. Similarly, the notation \( A \subseteq \mathcal{X} \) always indicates that \( A \) is measurable.
Finally, the term “continuous distribution” refers specifically to a distribution that is absolutely continuous with respect to the Lebesgue measure.

\subsection{Supplementary theorems}

This subsection presents two well-known results that we will use in proving our main results.

\begin{theorem}[Data processing inequality]\label{thm: dpi}
   Let \( M \) be a conditional distribution (Markov kernel) mapping from \( \mathcal{X} \) to \( \mathcal{Y} \). Given distributions \( P_1, P_2 \in \mathcal{P}(\mathcal{X}) \), and their push-forward measures \( Q_1, Q_2 \in \mathcal{P}(\mathcal{Y}) \) through \( M \), the following inequality holds for any \( f \)-divergence:
\[
D_f(P_1 \parallel P_2) \geq D_f(Q_1 \parallel Q_2).
\]
\end{theorem}

\begin{proposition}[Proposition 1 of \citet{dong2022gaussian}]\label{prop: trade-off}
    A function $g : [0,1] \to [0,1]$ is a trade-off function if and only if $g$ is convex, continuous, non-increasing, and $g(x) \leq 1 - x$ for $x \in [0,1]$.
\end{proposition}

\section{Preliminary results
}\label{sec: proofs prilim}

\begin{lemma}\label{lem:decrease-lambda}
Let $\mu$ be a (fixed) distribution over some output space $\mathcal{X}$. For each $0 \leq \lambda \leq 1$, define a sampler $\bQ_\lambda$ by
\[
  \bQ_\lambda(A \mid P) \;=\; \lambda\,P(A) \;+\; (1-\lambda)\,\mu(A),
  \quad
  \text{for any distribution }P \text{ on } \mathcal{X} \text{ and event } A \subseteq \mathcal{X}.
\]
Suppose there exists $0 \leq \lambda_1 \leq 1$ such that $\bQ_{\lambda_1}$ is $(\varepsilon,\delta)$-LDP. Then for any $\lambda_2 \in [0,\lambda_1]$, the sampler $\bQ_{\lambda_2}$ is also $(\varepsilon,\delta)$-LDP.
\end{lemma}

\begin{proof}
Fix $\lambda_2 \in [0,\lambda_1]$, and let $P,P'$ be any two distributions on $\mathcal{X}$. We need to show that for every event $A$,
\[
  \bQ_{\lambda_2}(A \mid P)
  \;\le\;
  e^\varepsilon\;\bQ_{\lambda_2}(A \mid P')
  \;+\;
  \delta.
\]
Because $\lambda_2 \le \lambda_1$, there is a $\beta \in [0,1]$ such that
\[
  \lambda_2 \;=\;\beta \,\lambda_1.
\]
Then, for all events $A$,
\[
  \bQ_{\lambda_2}(A \mid P)
  \;=\;
  \lambda_2\,P(A) \;+\; (1-\lambda_2)\,\mu(A)
  \;=\;
  \beta\,\lambda_1\,P(A) + \bigl[1 - \beta\,\lambda_1\bigr]\,\mu(A).
\]
Recognizing that $\lambda_1\,P(A) + (1-\lambda_1)\,\mu(A)$ is $\bQ_{\lambda_1}(A \mid P)$, we rewrite:
\[
  \bQ_{\lambda_2}(A \mid P)
  \;=\;
  \beta\,\bQ_{\lambda_1}(A \mid P)
  \;+\;
  \bigl[\,1-\beta\,\bigr]\mu(A).
\]
Similarly,
\[
  \bQ_{\lambda_2}(A \mid P')
  \;=\;
  \beta\,\bQ_{\lambda_1}(A \mid P') \;+\; \bigl[1-\beta\bigr]\mu(A).
\]

Since $\bQ_{\lambda_1}$ is $(\varepsilon,\delta)$-LDP, we know
\[
  \bQ_{\lambda_1}(A \mid P)
  \;\le\;
  e^\varepsilon\, \bQ_{\lambda_1}(A \mid P') \;+\;\delta.
\]
Multiplying both sides by $\beta \ge 0$ and then adding $(1-\beta)\,\mu(A)$, we obtain
\[
  \beta\,\bQ_{\lambda_1}(A \mid P) \;+\; (1-\beta)\,\mu(A)
  \;\le\;
  \beta\,e^\varepsilon\,\bQ_{\lambda_1}(A \mid P')
  \;+\;
  \beta\,\delta
  \;+\;
  (1-\beta)\,\mu(A).
\]
But the left-hand side above is exactly $\bQ_{\lambda_2}(A \mid P)$, so
\[
  \bQ_{\lambda_2}(A \mid P)
  \;\le\;
  \beta\,e^\varepsilon\,\bQ_{\lambda_1}(A \mid P')
  \;+\;
  \beta\,\delta
  \;+\;
  (1-\beta)\,\mu(A).
\]
On the other hand,
\[
  e^\varepsilon\,\bQ_{\lambda_2}(A \mid P')
  \;=\;
  e^\varepsilon\,\Bigl[\beta\,\bQ_{\lambda_1}(A \mid P') + (1-\beta)\,\mu(A)\Bigr]
  \;=\;
  \beta\,e^\varepsilon\,\bQ_{\lambda_1}(A \mid P') + e^\varepsilon\,(1-\beta)\,\mu(A).
\]
Hence,
\[
  e^\varepsilon\,\bQ_{\lambda_2}(A \mid P') \;+\;\delta
  \;=\;
  \beta\,e^\varepsilon\,\bQ_{\lambda_1}(A \mid P')
  \;+\;
  (1-\beta)\,e^\varepsilon\,\mu(A)
  \;+\;\delta.
\]
It is enough to check
\[
  \beta\,e^\varepsilon\,\bQ_{\lambda_1}(A \mid P')
  \;+\;
  \beta\,\delta
  \;+\;
  (1-\beta)\,\mu(A)
  \;\;\le\;\;
  \beta\,e^\varepsilon\,\bQ_{\lambda_1}(A \mid P')
  \;+\;
  (1-\beta)\,e^\varepsilon\,\mu(A)
  \;+\;
  \delta,
\]
which simplifies to
\[
  \beta\,\delta + (1-\beta)\,\mu(A)
  \;\le\;
  (1-\beta)\,e^\varepsilon\,\mu(A) + \delta.
\]
Rearranging,
\[
  (1-\beta)\,\mu(A)\bigl[1 - e^\varepsilon\bigr]
  \;\;\le\;\;
  \delta \bigl[\,1 - \beta\bigr].
\]
Since $1 - e^\varepsilon \le 0$ for $\varepsilon \ge 0$, and $\delta \ge 0$, this inequality holds trivially (the left-hand side is at most $0$, while the right-hand side is nonnegative). Thus the overall chain of inequalities is valid, and we conclude
\[
  \bQ_{\lambda_2}(A \mid P)
  \;\le\;
  e^\varepsilon\,\bQ_{\lambda_2}(A \mid P') \;+\;\delta,
  \quad
  \forall P,P',\quad \text{event } A.
\]
Hence $\bQ_{\lambda_2}$ is also $(\varepsilon,\delta)$-LDP.
\end{proof}

Now, we first restate the normalization condition and introduce the non-triviality assumption for $(\eps,\delta)$-LDP before stating the next proposition.

\textbf{Normalization condition:}
\begin{equation}\label{eqn: normalization condition}
    \mu(\mathcal{X}) = 1, \quad 0 \le c_1 < 1 < c_2.
\end{equation}
\textbf{Non-triviality condition:}
\begin{equation}\label{eqn: non-trivial condition condition approximate}
    \delta \leq \frac{(c_2-c_1e^\eps)(1-c_1)}{c_2-c_1}.
\end{equation}
\medskip
\begin{proposition}\label{prop: global appx linear}
    Consider the following optimization problem:
    \begin{equation}\label{eqn: global appx}
        \mathcal{R}\big(\mathcal{Q}_{\mathcal{X}, \tilde{\mathcal{P}}, \varepsilon, \delta}, \tilde{\mathcal{P}}, f\big) 
        \coloneqq 
        \inf_{\bQ \in \mathcal{Q}_{\mathcal{X}, \tilde{\mathcal{P}}, \varepsilon, \delta}} 
        \;\sup_{P \in \tilde{\mathcal{P}}} 
        \;D_f\big(P \, \| \, \bQ(P)\big).
    \end{equation}

    Let $\tilde{\mathcal{P}} = \tilde{\mathcal{P}}_{c_1, c_2, \mu}$ such that the normalization condition~\eqref{eqn: normalization condition} and the non-triviality condition (\ref{eqn: non-trivial condition condition approximate}) hold and $\frac{c_2 - c_1}{1 - c_1} \in \mathbb{N}$. Suppose that $\mu$ is $(\alpha,\frac{1}{\alpha},1)$-decomposable with $\alpha = \frac{1 - c_1}{c_2 - c_1}$. Define
    \[
    r_1 := \frac{c_1}{c_2 - c_1} \cdot \frac{(1 - c_1)e^{\varepsilon} + c_2 - 1}{1 - \delta},
    \qquad
    r_2 := \frac{c_2}{c_2 - c_1} \cdot \frac{(1 - c_1)e^{\varepsilon} + c_2 - 1}{e^{\varepsilon} + \frac{c_2 - 1}{1 - c_1} \delta}.
    \]

    Then, the optimal value of the problem~\eqref{eqn: global appx} is given by
    \begin{equation}\label{eqn: global appx optimal value}
        \mathcal{R}\big(\mathcal{Q}_{\mathcal{X}, \tilde{\mathcal{P}}, \varepsilon, \delta}, \tilde{\mathcal{P}}, f\big)
        = \frac{1 - r_1}{r_2 - r_1} \, f(r_2) + \frac{r_2 - 1}{r_2 - r_1} \, f(r_1).
    \end{equation}

    Moreover, the sampler
    $\bQ^\star_{\varepsilon, \delta}(P) \coloneqq \lambda^\star_{\varepsilon, \delta} P + (1 - \lambda^\star_{\varepsilon, \delta}) \mu,
    $
    with
    $
    \lambda^\star_{\varepsilon, \delta} 
    = \frac{e^{\varepsilon} + \frac{c_2 - c_1}{1 - c_1} \delta - 1}
           {(1 - c_1) e^{\varepsilon} + c_2 - 1}$,
    belongs to the class $\mathcal{Q}_{\mathcal{X}, \tilde{\mathcal{P}}, \varepsilon, \delta}$ and is minimax-optimal under any $f$-divergence $D_f$. That is,
    \begin{equation}\label{eqn: global appx linear is optimal}
        \sup_{P \in \tilde{\mathcal{P}}} D_f\big(P \, \| \, \bQ^\star_{\varepsilon, \delta}(P)\big)
        = \mathcal{R}\big(\mathcal{Q}_{\mathcal{X}, \tilde{\mathcal{P}}, \varepsilon, \delta}, \tilde{\mathcal{P}}, f\big).
    \end{equation}
\end{proposition}
\begin{proof}
We first show that if the non-triviality constraint (\ref{eqn: non-trivial condition condition approximate}) does not hold, then the identity sampler $\bQ^I(P) = P$ satisfies $(\eps,\delta)$-LDP and has the trivial minimax risk zero.  

Suppose we have $\delta>\frac{(c_2-c_1e^\eps)(1-c_1)}{c_2-c_1}$. Let $A\subseteq\mathcal{X}$ be a measurable subset and $P,P^\prime\in\mathcal{P}_{c_1,c_2,\mu}$. If $0\leq\mu(A)\leq\frac{1-c_1}{c_2-c_1}$,
\begin{eqnarray*}
    e^\eps P(A)+\delta-P^\prime(A)&\geq&c_1e^\eps\mu(A)+\delta-c_2\mu(A)\\
    &\geq&\min\left\{\frac{(c_1e^\eps-c_2)(1-c_1)}{c_2-c_1}+\delta,\delta\right\}\\
    &\geq&0.
\end{eqnarray*}
If $\mu(A)\geq\frac{1-c_1}{c_2-c_1}$,
\begin{eqnarray*}
    e^\eps P(A)+\delta-P^\prime(A)&\geq&c_1e^\eps\mu(A)+\delta-(1-P^\prime(A^c))\\
    &\geq&c_1e^\eps\mu(A)+\delta-(1-c_1(1-\mu(A)))\\
    &=&(c_1e^\eps-c_1)\mu(A)+\delta-1+c_1\\
    &\geq&\frac{(c_1e^\eps-c_1)(1-c_1)}{c_2-c_1}-1+c_1+\delta\\
    &=&\frac{(c_1e^\eps-c_2)(1-c_1)}{c_2-c_1}+\delta\\
    &>&0.
\end{eqnarray*}

Therefore, for $\delta>\frac{(c_2-c_1e^\eps)(1-c_1)}{c_2-c_1}$, $\bQ^I(P) \in \mathcal{Q}_{\mathcal{X}, \tilde{\mathcal{P}}, \varepsilon, \delta}$ and the minimax risk is trivially zero. In the remainder of the proof, we consider the non-trivial case of $\delta \leq \frac{(c_2-c_1e^\eps)(1-c_1)}{c_2-c_1}$.

   \medskip
   
    \textbf{Step 1: Proof of $\bQ^\star_{\varepsilon, \delta}$ is a probability measure} 

    In this proof, we show that $\bQ^\star_{\varepsilon, \delta}$  defined by $
\bQ^\star_{\varepsilon, \delta}(P) 
:= \lambda^\star_{\varepsilon, \delta} P + \bigl(1 - \lambda^\star_{\varepsilon, \delta}\bigr) \mu
$ is a probability measure.
Since $P$ and $\mu$ are probability measures, we have
\[
\bQ^\star_{\varepsilon, \delta}(\mathcal{X})
= \lambda^\star_{\varepsilon, \delta} P(\mathcal{X}) 
+ \bigl(1 - \lambda^\star_{\varepsilon, \delta}\bigr)\mu(\mathcal{X})
= \lambda^\star_{\varepsilon, \delta} + \bigl(1 - \lambda^\star_{\varepsilon, \delta}\bigr)
= 1.
\]
Assumptions $0\leq c_1 < 1 < c_2$ and $\delta \leq \frac{(c_2-c_1e^\eps)(1-c_1)}{c_2-c_1}$ imply
\[\lambda^\star_{\varepsilon, \delta} 
    = \frac{e^{\varepsilon} + \frac{c_2 - c_1}{1 - c_1} \delta - 1}
           {(1 - c_1) e^{\varepsilon} + c_2 - 1} \leq \frac{e^{\varepsilon} + \frac{c_2 - c_1}{1 - c_1} \frac{(c_2-c_1e^\eps)(1-c_1)}{c_2-c_1} - 1}
           {(1 - c_1) e^{\varepsilon} + c_2 - 1} = 1.\]
Therefore, for any $A \subseteq \X$, we have $\bQ^\star_{\varepsilon, \delta}(A) \geq 0$. It suffices to show that for every measurable subset $A \subseteq \mathcal{X}$, we also have
\[
\bQ^\star_{\varepsilon, \delta}(A) \le 1,
\]
which implies $\bQ^\star_{\varepsilon, \delta}$ is a probability measure. We proceed by bounding:
\begin{align*}
\bQ^\star_{\varepsilon, \delta}(A) 
& = \lambda^\star_{\varepsilon, \delta} P(A) 
  + \bigl(1 - \lambda^\star_{\varepsilon, \delta}\bigr)\mu(A) \\
& \le \lambda^\star_{\varepsilon, \delta} 
  \Bigl(1 - c_1\bigl(1 - \mu(A)\bigr)\Bigr)
  + \bigl(1 - \lambda^\star_{\varepsilon, \delta}\bigr)\mu(A) \\
& = \Bigl(1 - \bigl(1 - c_1\bigr)\lambda^\star_{\varepsilon, \delta}\Bigr)\mu(A)
  + \lambda^\star_{\varepsilon, \delta}(1 - c_1) \\
& \le \frac{(c_2 - c_1)\,(1 - \delta)}{\bigl(1 - c_1\bigr)e^{\varepsilon} + c_2 - 1}
  \;+\; \frac{\bigl(1 - c_1\bigr)e^{\varepsilon} + (c_2 - c_1)\delta 
  - \bigl(1 - c_1\bigr)}{\bigl(1 - c_1\bigr)e^{\varepsilon} + c_2 - 1} \\
& = \frac{\bigl(1 - c_1\bigr)e^{\varepsilon} + c_2 - 1}
  {\bigl(1 - c_1\bigr)e^{\varepsilon} + c_2 - 1} \\
& = 1.
\end{align*}
The inequality in the second line uses the fact that
\[
P(A) 
\;\le\; 1 - c_1\bigl(1 - \mu(A)\bigr).
\]
This follows from the definition
\[
\tilde{\mathcal{P}}_{c_1, c_2, \mu}
:= \Bigl\{\,P \in \mathcal{P}(\mathcal{X}) :\; P \ll \mu,\quad 
  c_1 \le \tfrac{dP}{d\mu} \le c_2 \;\;\mu\text{-a.e.}\Bigr\},
\]
and the fact that for $P \in \tilde{\mathcal{P}}_{c_1, c_2, \mu}$:
\[
c_1 \,\mu(A^c) \;\le\; P(A^c) 
\;\;\implies\;\;
1 - c_1\,\mu(A^c) 
\;\ge\; 1 - P(A^c) 
\;\;\implies\;\;
1 - c_1\bigl(1 - \mu(A)\bigr) 
\;\ge\; P(A).
\]
And the inequality in the fourth line comes by plugging $\lambda^\star_{\varepsilon, \delta} 
    = \frac{e^{\varepsilon} + \frac{c_2 - c_1}{1 - c_1} \delta - 1}
           {(1 - c_1) e^{\varepsilon} + c_2 - 1}$ into the equation and the fact that $\mu(A) \leq 1$. Hence, $\bQ^\star_{\varepsilon, \delta}$ is indeed a probability measure.

    \medskip

    \textbf{Step 2: Proof of $\bQ^\star_{\varepsilon, \delta} \in \mathcal{Q}_{\mathcal{X}, \tilde{\mathcal{P}}, \varepsilon, \delta}$ } 

Now, we want to show that $\bQ^\star_{\varepsilon, \delta}$ is an $(\eps,\delta)$-LDP sampler. For this purpose, we have to show that for each measurable set $A \subseteq \mathcal{X}$ and any pair of probability measures $P$ and $P'$, we have:
\[e^\eps\big(\bQ^\star_{\varepsilon, \delta}(P)(A)\big) + \delta - \bQ^\star_{\varepsilon, \delta}(P')(A) \geq 0.\]

First, suppose \(\mu(A) \le \frac{1 - c_1}{c_2 - c_1}\). We then have:
\begin{align*}
 e^\varepsilon \Bigl(\lambda^\star_{\varepsilon,\delta}& P(A) 
   +   \bigl(1 - \lambda^\star_{\varepsilon,\delta}\bigr)\mu(A)\Bigr)
 \;+\;\delta
 \;-\;\Bigl(\lambda^\star_{\varepsilon,\delta} P'(A) 
   + \bigl(1 - \lambda^\star_{\varepsilon,\delta}\bigr)\mu(A)\Bigr)\\
&\ge 
  e^\varepsilon \Bigl(\lambda^\star_{\varepsilon,\delta}\,c_1\,\mu(A) 
  + \bigl(1 - \lambda^\star_{\varepsilon,\delta}\bigr)\mu(A)\Bigr)
 \;+\;\delta
 \;-\;
 \Bigl(\lambda^\star_{\varepsilon,\delta}\,c_2\,\mu(A)
   + \bigl(1 - \lambda^\star_{\varepsilon,\delta}\bigr)\mu(A)\Bigr)\\
&=\Big(\lambda^\star_{\varepsilon,\delta}\big( e^\eps c_1 - e^\eps - c_2 + 1\big) + \big(e^\eps - 1\big)\Big)
   \;\mu(A)
 \;+\;\delta\\
&= -\,\frac{c_2 - c_1}{1 - c_1}\,\delta\,\mu(A) 
   \;+\;\delta\\
&\ge 0.
\end{align*}

Next, suppose \(\mu(A) \ge \frac{1 - c_1}{c_2 - c_1}\). Then:
\begin{align*}
e^{\varepsilon} \Big(&\lambda^\star_{\varepsilon,\delta}  P(A)  +  \big(1 - \lambda^\star_{\varepsilon,\delta}\big) \mu(A)\Big) + \delta - \Big(\lambda^\star_{\varepsilon,\delta} P'(A) + \big(1 - \lambda^\star_{\varepsilon,\delta}\big) \mu(A)\Big) \\ 
& \geq e^{\varepsilon} \Big(\lambda^\star_{\varepsilon,\delta} c_1 \mu(A) + \big(1 - \lambda^\star_{\varepsilon,\delta}\big) \mu(A)\Big) + \delta - \Big(\lambda^\star_{\varepsilon,\delta} \big(1 - c_1 + c_1 \mu(A)\big) + \big(1 - \lambda^\star_{\varepsilon,\delta}\big) \mu(A)\Big) \\
& = \Big(e^{\varepsilon} \big(1 - (1 - c_1) \lambda^\star_{\varepsilon,\delta}\big) - \big(1 - (1 - c_1) \lambda^\star_{\varepsilon,\delta}\big)\Big) \mu(A) - \lambda^\star_{\varepsilon,\delta}(1 - c_1) + \delta \\
& = (e^{\varepsilon} - 1)\big(1 - (1 - c_1) \lambda^\star_{\varepsilon,\delta}\big) \mu(A) - \lambda^\star_{\varepsilon,\delta}(1 - c_1) + \delta \\
& \geq (e^{\varepsilon} - 1)\big(1 - (1 - c_1) \lambda^\star_{\varepsilon,\delta}\big) \Big(\frac{1 - c_1}{c_2 - c_1}\Big) - \lambda^\star_{\varepsilon,\delta}(1 - c_1) + \delta \\
& = \frac{(e^{\varepsilon} - 1)(1 - \delta)(1 - c_1)}{(1 - c_1)e^{\varepsilon} + c_2 - 1} - \frac{(1 - c_1)e^{\varepsilon} + (c_2 - c_1)\delta - (1 - c_1)}{(1 - c_1)e^{\varepsilon} + c_2 - 1} + \delta \\
& = \frac{-\delta(1 - c_1)(e^{\varepsilon} - 1) - (c_2 - c_1)\delta}{(1 - c_1)e^{\varepsilon} + c_2 - 1} + \delta \\
& = 0.
\end{align*}

Therefore, in both cases:
\begin{align*}
e^\varepsilon \Bigl(\lambda^\star_{\varepsilon,\delta} P(A) 
   + \bigl(1 - \lambda^\star_{\varepsilon,\delta}\bigr)\mu(A)\Bigr)
\;+\;\delta
\;-\;
\Bigl(\lambda^\star_{\varepsilon,\delta} P'(A) 
   + \bigl(1 - \lambda^\star_{\varepsilon,\delta}\bigr)\mu(A)\Bigr)
\;\;\ge\;0,
\end{align*}

which establishes the \((\varepsilon,\delta)\)-LDP criterion for every measurable set \(A\). Hence the sampler is \((\varepsilon,\delta)\)-LDP.

    \textbf{Step 3: Proof of optimality, lower bound}

    In this part, we show that:
    \[\frac{1 - r_1}{r_2 - r_1} \, f(r_2) + \frac{r_2 - 1}{r_2 - r_1} \, f(r_1) \leq \mathcal{R}\big(\mathcal{Q}_{\mathcal{X}, \tilde{\mathcal{P}}, \varepsilon, \delta}, \tilde{\mathcal{P}}, f\big).\]
     To this end, we need to show that for each $\bQ \in \mathcal{Q}_{\mathcal{X}, \tilde{\mathcal{P}}, \varepsilon, \delta}$, we have
     \[\frac{1 - r_1}{r_2 - r_1} \, f(r_2) + \frac{r_2 - 1}{r_2 - r_1} \, f(r_1) \le \sup_{P \in \tilde{\mathcal{P}}} 
        \;D_f\big(P \, \| \, \bQ(P)\big)\]
       
When $\frac{c_2 - c_1}{1 - c_1} \in \mathbb{N}$, let $\bQ$ be an $(\varepsilon, \delta)$-LDP sampler, i.e, $\bQ \in \mathcal{Q}_{\mathcal{X}, \tilde{\mathcal{P}}, \varepsilon, \delta}$ Define
\[
\alpha := \frac{1 - c_1}{c_2 - c_1}
\quad ,
\quad t := \alpha^{-1} \in \mathbb{N}.
\]
Since $\mu$ is $(\alpha,\frac{1}{\alpha},1)$-decomposable with $\frac{1}{\alpha} \in \mathbb{N}$, we can find disjoint sets $A_1, \ldots, A_t \subseteq \mathcal{X}$ with $\mu(A_i) = \alpha$. Define $P_i$ as a probability measure by $\frac{dP_i}{d\mu} = p_i$ defined as:
\[
p_i(x) = c_2 \mathbf{1}_{A_i}(x) + c_1 \mathbf{1}_{A_i^c}(x),
\]
where \[
\mathbf{1}_{A_i}(x) =
\begin{cases}
1 & \text{if } x \in A_i, \\
0 & \text{otherwise}.
\end{cases}
\]
Let $\bQ_i := \bQ(P_i)$. We aim to show:
\[
\min_i \bQ_i(A_i) \leq \frac{e^\varepsilon + (t-1)\delta}{e^\varepsilon + t - 1}.
\]

By the $(\varepsilon, \delta)$-LDP property, for every $i > 1$, we have
\[
e^\varepsilon \bQ_1(A_i) + \delta \geq \bQ_i(A_i).
\]
Thus,
\begin{equation*}
    \sum_{i=2}^t \Big( e^\varepsilon \bQ_1(A_i) + \delta \Big) \geq \sum_{i=2}^t \bQ_i(A_i).
\end{equation*}

Note that
\begin{align*}
    \sum_{i=2}^t \Big( e^\varepsilon \bQ_1(A_i) + \delta\Big)
& = e^\varepsilon \bQ_1\left(\bigcup_{i=2}^t A_i\right) + (t-1)\delta \\ &
\leq e^\varepsilon \left(1 - \bQ_1(A_1)\right) + (t-1)\delta
\\ & \leq e^\varepsilon \left(1 - \min_i \bQ_i(A_i)\right) + (t-1)\delta,
\end{align*}
and
\begin{equation*}
    \sum_{i=2}^t \bQ_i(A_i) \geq (t-1) \min_i \bQ_i(A_i).
\end{equation*}

Therefore,
\[
\min_i \bQ_i(A_i) \leq \frac{e^\varepsilon + (t-1)\delta}{e^\varepsilon + t-1}.
\]

Rewriting in terms of $\alpha$, we have
\begin{equation}\label{eqn: min Q_i}
\min_i \bQ_i(A_i) \leq \frac{\alpha e^\varepsilon + (1-\alpha)\delta}{\alpha e^\varepsilon + 1-\alpha}.    
\end{equation}

\medskip

Next, we lower bound the worst-case $f$-divergence:
\[
\sup_i D_f(P_i \,\|\, \bQ(P_i))
\geq \sup_i D_f^B\left(c_2 \alpha \,\|\, \bQ_i(A_i)\right),
\]
where $D_f^B(\lambda_1 \,\|\, \lambda_2)$ denotes the $f$-divergence between Bernoulli distributions with $\Pr(1) = \lambda_1$ and $\lambda_2$:
\[
D^B_f(\lambda_1 \parallel \lambda_2) = \lambda_2 f \left( \frac{\lambda_1}{\lambda_2} \right) + (1 - \lambda_2) f \left( \frac{1 - \lambda_1}{1 - \lambda_2} \right).
\]

For each $A_i$, the push-forward measures of $P_i$ and $\bQ_i$ by the indicator function $\mathbf{1}_{A_i}$ are Bernoulli distributions with $\Pr(1) = c_2 \alpha$ and $\bQ_i(A_i)$, respectively. By the data processing inequality (Theorem \ref{thm: dpi}), we have:
\[
D_f ( P_i\| \bQ_i) ) \geq D_f^B \left( c_2 \alpha \| \bQ_i(A_i) \right).
\]

Therefore, we have:
\begin{align*}
    \sup_{P \in \tilde{\mathcal{P}}} 
        \;D_f\big(P \, \| \, \bQ(P)\big) & \geq \sup_{i \in [t]} 
        \; D_f ( P_i \, \| \, \bQ_i) \\ & \geq  \sup_{i \in [t]} 
        \; D_f^B \left( c_2 \alpha \, \| \, \bQ_i(A_i) \right) \\ & \geq 
        D_f^B \left( c_2 \alpha \, \| \, \frac{\alpha e^\varepsilon + (1-\alpha)\delta}{\alpha e^\varepsilon + 1-\alpha} \right) 
\end{align*}

where the inequality in the last line follows from (\ref{eqn: min Q_i}) and the fact that if $\delta \leq \frac{(c_2-c_1e^\eps)(1-c_1)}{c_2-c_1}$, then $ \frac{\alpha e^\varepsilon + (1-\alpha)\delta}{\alpha e^\varepsilon + 1-\alpha}  \leq c_2 \alpha$ and $D_f^{\mathrm{B}}(\lambda_1 \,\|\, \lambda_2)$ is decreasaing in $\lambda_2 \in [0 , \lambda_1]$.

Here is the proof that if $\delta \leq \frac{(c_2-c_1e^\eps)(1-c_1)}{c_2-c_1}$, then $ \frac{\alpha e^\varepsilon + (1-\alpha)\delta}{\alpha e^\varepsilon + 1-\alpha}  \leq c_2 \alpha$.

\begin{eqnarray*}
    \frac{\alpha e^\eps+(1-\alpha)\delta}{\alpha e^\eps+1-\alpha}&=&\frac{(1-c_1) e^\eps+(c_2-1)\delta}{(1-c_1) e^\eps+c_2-1}\\
    &\leq&\frac{(1-c_1) e^\eps+(c_2-1)\left(\frac{(c_2-c_1e^\eps)(1-c_1)}{c_2-c_1}\right)}{(1-c_1) e^\eps+c_2-1}\\
    &=&\frac{1-c_1}{c_2-c_1}\left(\frac{(c_2-c_1)e^\eps+(c_2-1)(c_2-c_1e^\eps)}{(1-c_1)e^\eps+c_2-1}\right)\\
    &=&\frac{c_2(1-c_1)}{c_2-c_1}\\
    &=&c_2\alpha.
\end{eqnarray*}

Therefore, it suffices to compute $D_f^B \left( c_2 \alpha \, \| \, \frac{\alpha e^\varepsilon + (1-\alpha)\delta}{\alpha e^\varepsilon + 1-\alpha} \right)$. It can be shown that
\begin{align*}
  D_f^B\left(c_2 \alpha \,\Big\|\, \frac{\alpha e^\varepsilon + (1-\alpha)\delta}{\alpha e^\varepsilon + 1-\alpha}\right) &
  = \frac{\alpha e^\varepsilon + (1-\alpha)\delta}{\alpha e^\varepsilon + 1-\alpha}
f\left( c_2 \alpha \frac{\alpha e^\varepsilon + 1-\alpha}{\alpha e^\varepsilon + (1-\alpha)\delta} \right) \\
& + \frac{(1-\alpha)(1-\delta)}{\alpha e^\varepsilon + 1-\alpha}
f\left( c_1 (1-\alpha) \frac{\alpha e^\varepsilon + 1-\alpha}{(1-\alpha)(1-\delta)} \right) \\
&= \frac{(1-c_1)e^\varepsilon + (c_2-1)\delta}{(1-c_1)e^\varepsilon + c_2-1}
f\left( \frac{c_2}{c_2-c_1} \frac{(1-c_1)e^\varepsilon + c_2-1}{e^\varepsilon + \frac{c_2-1}{1-c_1}\delta} \right) \\
&+ \frac{(c_2-1)(1-\delta)}{(1-c_1)e^\varepsilon + c_2-1}
f\left( \frac{c_1}{c_2-c_1} \frac{(1-c_1)e^\varepsilon + c_2-1}{1-\delta} \right).
\end{align*}


Defining
\[
    r_1 := \frac{c_1}{c_2 - c_1} \cdot \frac{(1 - c_1)e^{\varepsilon} + c_2 - 1}{1 - \delta},
    \qquad
    r_2 := \frac{c_2}{c_2 - c_1} \cdot \frac{(1 - c_1)e^{\varepsilon} + c_2 - 1}{e^{\varepsilon} + \frac{c_2 - 1}{1 - c_1} \delta},
    \]
 we have   
\begin{align*}
1 - r_1 &= \frac{(1-c_1)\left(c_2 - c_1 e^\varepsilon - \frac{c_2-c_1}{1-c_1}\delta\right)}{(c_2-c_1)(1-\delta)}, \\
r_2 - 1 &= \frac{(c_2-1)\left(c_2 - c_1 e^\varepsilon - \frac{c_2-c_1}{1-c_1}\delta\right)}{(c_2-c_1)\left(e^\varepsilon + \frac{c_2-1}{1-c_1}\delta\right)}.
\end{align*}
Thus
\begin{align*}
\frac{1 - r_1}{r_2 - r_1} &= \frac{(1-c_1)e^\varepsilon + (c_2-1)\delta}{(1-c_1)e^\varepsilon + c_2-1}, \\
\frac{r_2 - 1}{r_2 - r_1} &= \frac{(c_2-1)(1-\delta)}{(1-c_1)e^\varepsilon + c_2-1}.
\end{align*}
Hence, for each $\bQ \in \mathcal{Q}_{\mathcal{X}, \tilde{\mathcal{P}}, \varepsilon, \delta}$, we have
     \[\frac{1 - r_1}{r_2 - r_1} \, f(r_2) + \frac{r_2 - 1}{r_2 - r_1} \, f(r_1) \le \sup_{P \in \tilde{\mathcal{P}}} 
        \;D_f\big(P \, \| \, \bQ(P)\big).\]
        
    \textbf{Step 4: Proof of optimality, upper bound (achievability part)}

    In this part, we show that:
    \[\mathcal{R}\big(\mathcal{Q}_{\mathcal{X}, \tilde{\mathcal{P}}, \varepsilon, \delta}, \tilde{\mathcal{P}}, f\big)
       \le \frac{1 - r_1}{r_2 - r_1} \, f(r_2) + \frac{r_2 - 1}{r_2 - r_1} \, f(r_1).\]
        To this end, we need to show that for each $P \in  \tilde{\mathcal{P}}$, we have:
        \[D_f\big(P \, \| \, \bQ^\star_{\varepsilon, \delta}(P)\big) \le \frac{1 - r_1}{r_2 - r_1} \, f(r_2) + \frac{r_2 - 1}{r_2 - r_1} \, f(r_1).\]

    Fix \( P \in \tilde{\mathcal{P}} \). Let \( p = \frac{dP}{d\mu} \) and \( q = \frac{d\bQ^\star_{\eps,\delta}(P)}{d\mu} \). We first claim that
\[
r_1 \leq \frac{p(x)}{q(x)} \leq r_2
\]
for \(\mu\)-almost every \( x \in \mathcal{X} \). Indeed, we have
\[
\frac{p(x)}{q(x)}
= \frac{p(x)}{\lambda^\star_{\varepsilon,\delta} p(x) + (1 - \lambda^\star_{\varepsilon,\delta})}
,\]
which is an increasing function of \( p(x) \geq 0 \) as long as $\lambda^\star_{\eps,\delta} \leq 1$. Now we show that if $\delta \leq \frac{(c_2-c_1e^\eps)(1-c_1)}{c_2-c_1}$, then $\lambda^\star_{\varepsilon, \delta} 
    = \frac{e^{\varepsilon} + \frac{c_2 - c_1}{1 - c_1} \delta - 1}
           {(1 - c_1) e^{\varepsilon} + c_2 - 1} \leq 1$. 
           \begin{align*}
               \lambda^\star_{\varepsilon, \delta} 
   & = \frac{e^{\varepsilon} + \frac{c_2 - c_1}{1 - c_1} \delta - 1}
           {(1 - c_1) e^{\varepsilon} + c_2 - 1} \\
           &\leq  \frac{e^{\varepsilon} + \frac{c_2 - c_1}{1 - c_1} \Big( \frac{(c_2-c_1e^\eps)(1-c_1)}{c_2-c_1}\Big) - 1}
           {(1 - c_1) e^{\varepsilon} + c_2 - 1}.\\
           & = \frac{e^{\varepsilon} +  c_2-c_1e^\eps - 1}
           {(1 - c_1) e^{\varepsilon} + c_2 - 1} \\
           & = 1.
           \end{align*}
           Since \( p(x) \) is bounded between \( c_1 \) and \( c_2 \), it follows that
\begin{equation}\label{eqn: p/q ratio}
    \frac{c_1}{\lambda^\star_{\varepsilon,\delta} c_1 + (1 - \lambda^\star_{\varepsilon,\delta})}
\leq \frac{p(x)}{q(x)}
\leq \frac{c_2}{\lambda^\star_{\varepsilon,\delta} c_2 + (1 - \lambda^\star_{\varepsilon,\delta})}
\end{equation}
for \(\mu\)-almost every \( x \in \mathcal{X} \).
From the premise of the proposition, we have
\begin{equation*}
    \lambda^\star_{\varepsilon, \delta} 
    = \frac{e^{\varepsilon} + \frac{c_2 - c_1}{1 - c_1} \delta - 1}
           {(1 - c_1) e^{\varepsilon} + c_2 - 1}
\end{equation*}
It can be shown that:
\[
\frac{c_1}{c_2 -c_1}\frac{  (1-c_1) e^{\varepsilon} + c_2 - 1 }{(1-\delta)}
\leq \frac{p(x)}{q(x)}
\leq \frac{c_2}{c_2 - c_1} \cdot \frac{(1 - c_1)e^{\varepsilon} + c_2 - 1}{e^{\varepsilon} + \frac{c_2 - 1}{1 - c_1} \delta}
\]
which concludes
\[
r_1 \leq \frac{p(x)}{q(x)} \leq r_2
\]
for \(\mu\)-almost every \( x \in \mathcal{X} \).
To finalize the proof of the achievability part, we need the following lemma.

\begin{lemma}[{[Theorem 2.1]{\citep{rukhin1997information}}}]\label{lem: bounded f-div}
Let \( P, Q \in \mathcal{P}(\mathcal{X}) \). Suppose that \( P \) and \( Q \) are both absolutely continuous with respect to a reference measure \( \mu \) on \( \mathcal{X} \). Assume that there exist \( r_1, r_2 \in \mathbb{R} \) with \( 0 \leq r_1 < 1 < r_2 \) such that the corresponding densities \( p = \frac{dP}{d\mu} \) and \( q = \frac{dQ}{d\mu} \) satisfy \( q(x) > 0 \) and
\[
r_1 \leq \frac{p(x)}{q(x)} \leq r_2
\quad \text{for } \mu\text{-almost every } x \in \mathcal{X}.
\]
Then, for any \( f \)-divergence \( D_f \), it holds that
\[
D_f(P\|Q) \leq \frac{1 - r_1}{r_2 - r_1} f(r_2) + \frac{r_2 - 1}{r_2 - r_1} f(r_1).
\]
\end{lemma}

Lemma \ref{lem: bounded f-div}, directly shows that for each $P \in  \tilde{\mathcal{P}}$, we have:
        \[D_f\big(P \, \| \, \bQ^\star_{\varepsilon, \delta}(P)\big) \le \frac{1 - r_1}{r_2 - r_1} \, f(r_2) + \frac{r_2 - 1}{r_2 - r_1} \, f(r_1)\]
for
\[
    r_1 = \frac{c_1}{c_2 - c_1} \cdot \frac{(1 - c_1)e^{\varepsilon} + c_2 - 1}{1 - \delta},
    \qquad
    r_2 = \frac{c_2}{c_2 - c_1} \cdot \frac{(1 - c_1)e^{\varepsilon} + c_2 - 1}{e^{\varepsilon} + \frac{c_2 - 1}{1 - c_1} \delta}.
    \]
\end{proof}

\begin{corollary}[$g_{\eps,\delta}$-FLDP as a special case]\label{cor: special case of f-LDP approximate}

Let $\tilde{\mathcal{P}} = \tilde{\mathcal{P}}_{c_1, c_2, \mu}$. Under Assumption \ref{assumption: norm}, suppose the reference measure $\mu$ is $(\alpha,\frac{1}{\alpha},1)$-decomposable with $\alpha = \frac{1 - c_1}{c_2 - c_1}$. Then the sampler
\begin{equation}
    \bQ^\star_{c_1,c_2,\mu,g_{\varepsilon,\delta}}(P) \coloneqq \lambda^\star_{c_1,c_2,g_{\varepsilon,\delta}} P + (1 - \lambda^\star_{c_1,c_2,g_{\varepsilon,\delta}})\mu
\end{equation}
belongs to the class $\mathcal{Q}_{\mathcal{X}, \tilde{\mathcal{P}}, g_{\varepsilon,\delta}}$ and is minimax-optimal with respect to any $f$-divergence $D_f$, that is,
\begin{equation}
    \sup_{P \in \tilde{\mathcal{P}}} D_f\big(P \, \| \, \bQ^\star_{c_1,c_2,\mu,g_{\varepsilon,\delta}}(P)\big)
    = \mathcal{R}\big(\mathcal{Q}_{\mathcal{X}, \tilde{\mathcal{P}}, g_{\varepsilon,\delta}}, \tilde{\mathcal{P}}, f\big),
\end{equation}
where 
\[
\lambda^\star_{c_1,c_2,g_{\varepsilon,\delta}} = \frac{e^{\varepsilon} + \frac{c_2 - c_1}{1 - c_1} \, \delta - 1}{(1 - c_1) e^{\varepsilon} + c_2 - 1}.
\]

\end{corollary}

\begin{proof}
From Theorem \ref{thm: global functional}, we know that the sampler \[\bQ^\star_{c_1,c_2,\mu,g_{\varepsilon,\delta}}(P) \coloneqq \lambda^\star_{c_1,c_2,g_{\varepsilon,\delta}} P + (1 - \lambda^\star_{c_1,c_2,g_{\varepsilon,\delta}})\mu\]  for $\lambda^\star_{c_1,c_2,g_{\varepsilon,\delta}}$ defined as \[
    \lambda^\star_{c_1,c_2,g_{\varepsilon,\delta}} 
    \coloneq \inf_{\beta \geq 0} 
    \frac{e^\beta + \frac{c_2 - c_1}{1 - c_1} \big( 1 + g_{\varepsilon,\delta}^*(-e^\beta) \big) - 1}
         {(1 - c_1)e^\beta + c_2 - 1}.
    \] 
    belongs to the class $\mathcal{Q}_{\mathcal{X}, \tilde{\mathcal{P}}, g_{\varepsilon,\delta}}$ and is minimax-optimal with respect to any $f$-divergence $D_f$. Therefore, it suffices to compute $g_{\varepsilon,\delta}^*(-e^\varepsilon)$ and then solve the minimization problem.

A direct comparison of the three affine pieces shows

\begin{equation}
    g_{\varepsilon,\delta}(\theta)=
\begin{cases}
1-\delta-e^{\varepsilon}\theta, & 0\le\theta\le \frac{1-\delta}{e^{\varepsilon}+1},\\[4pt]
e^{-\varepsilon}(1-\delta-\theta), & \frac{1-\delta}{e^{\varepsilon}+1}\le\theta\le 1-\delta,\\[4pt]
0, & 1-\delta\le\theta\le 1,\\[4pt]
+\infty, & \text{otherwise.}
\end{cases}\label{eq: piecewise-f}
\end{equation}

For any \(y\in\mathbb{R}\),
\(
g_{\varepsilon,\delta}^{*}(y)
   =\sup\limits_{0\le\theta\le1}(\theta y-g_{\varepsilon,\delta}(\theta))
\)
splits naturally into the supremum over the three intervals of~\eqref{eq: piecewise-f}.
Denote the corresponding optimized values by
\[
\begin{aligned}
S_1(y)&:=\sup_{0\le\theta\le  \frac{1-\delta}{e^{\varepsilon}+1}}\bigl\{\theta y-(1-\delta-e^{\varepsilon}\theta)\bigr\},\\
S_2(y)&:=\sup_{ \frac{1-\delta}{e^{\varepsilon}+1}\le\theta\le 1 - \delta}\bigl\{\theta y-e^{-\varepsilon}(1-\delta-\theta)\bigr\},\\
S_3(y)&:=\sup_{1 - \delta\le\theta\le 1}\bigl\{\theta y\bigr\}.
\end{aligned}
\]

\vspace{4pt}
\noindent
We now optimize over each sub-interval.

\emph{(i) Interval \(0\le\theta\le  \frac{1-\delta}{e^{\varepsilon}+1}\).}  
Writing the objective as \((y+e^{\varepsilon})\theta-(1-\delta)\), it is linear in \(\theta\).
Hence  
\[
S_1(y)=
\begin{cases}
-(1-\delta), & y\le -e^{\varepsilon},\\[4pt]
\bigl(y+e^{\varepsilon}\bigr) \frac{1-\delta}{e^{\varepsilon}+1}-(1-\delta)
      =\dfrac{1-\delta}{e^{\varepsilon}+1}(y-1),
      & y\ge -e^{\varepsilon}.
\end{cases}
\]

\emph{(ii) Interval \(\frac{1-\delta}{e^{\varepsilon}+1}\le\theta\le 1-\delta\).}  
Here the objective equals \((y+e^{-\varepsilon})\theta-e^{-\varepsilon}(1-\delta)\).
Thus  
\[
S_2(y)=
\begin{cases}
\dfrac{1-\delta}{e^{\varepsilon}+1}(y-1), & y\le -e^{-\varepsilon},\\[6pt]
(1-\delta)y, & y\ge -e^{-\varepsilon}.
\end{cases}
\]

\emph{(iii) Interval \(1 - \delta\le\theta\le 1\).}  
The objective is simply \(y\theta\), so  
\[
S_3(y)=
\begin{cases}
(1-\delta)\,y, & y< 0,\\[4pt]
y, & y\ge 0.
\end{cases}
\]

We now take the overall supremum.

Comparing \(S_1,S_2,S_3\) on the four regimes  
\(\bigl(-\infty,-e^{\varepsilon}\bigr)\),
\(\bigl[-e^{\varepsilon},-e^{-\varepsilon}\bigr]\),
\(\bigl[-e^{-\varepsilon},0\bigr]\),
\([0,\infty)\)
gives
\[
g_{\varepsilon,\delta}^{*}(y)
   =\max\{S_1(y),S_2(y),S_3(y)\}
   =
\begin{cases}
-(1-\delta),
   & y< -e^{\varepsilon},\\[6pt]
\dfrac{1-\delta}{e^{\varepsilon}+1}(y-1),
   & -e^{\varepsilon}\le y\le -e^{-\varepsilon},\\[10pt]
(1-\delta)\,y,
   & -e^{-\varepsilon}\le y\le 0,\\[6pt]
y,
   & y\ge 0.
\end{cases}
\]


\textbf{Step 2: Solve the infimum to obtain $\lambda^\star_{c_1,c_2,g_{\varepsilon,\delta}}$}

 From   \[
    \lambda^\star_{c_1,c_2,g_{\varepsilon,\delta}} 
    = \inf_{\beta \geq 0} 
    \frac{e^\beta + \frac{c_2 - c_1}{1 - c_1} \big( 1 + g_{\eps,\delta}^*(-e^\beta) \big) - 1}
         {(1 - c_1)e^\beta + c_2 - 1}
    \] 
    Hence, we only need to know the value of $g_{\eps,\delta}^*(-e^\beta )$. We know $-e^{\beta} \le -1$ for $\beta \geq 0$. Therefore, $-e^{\beta} \leq - e^{-\eps}$ for given $\eps \geq 0$. i.e.,
    \[g_{\varepsilon,\delta}^{*}(-e^\beta) = \begin{cases}
        -(1-\delta), &  -e^\beta \le -e^\eps,\\[6pt]
        \dfrac{1-\delta}{e^{\varepsilon}+1}(-e^\beta-1), & -e^\eps \le -e^\beta \leq -1.
        
    \end{cases}\]

It follows that:
\begin{align*}\lambda^\star_{c_1,c_2,g_{\varepsilon,\delta}} & =  
         \min\Bigg\{ \inf_{\beta \geq \eps} 
    \frac{e^\beta + \frac{c_2 - c_1}{1 - c_1}  \delta  - 1}
         {(1 - c_1)e^\beta + c_2 - 1} \, , \, \inf_{\beta \in [0,\eps]} 
    \frac{e^\beta + \frac{c_2 - c_1}{1 - c_1} \bigg( 1 - \frac{e^\beta+1}{e^{\varepsilon}+1}(1 - \delta) \bigg) - 1}
         {(1 - c_1)e^\beta + c_2 - 1}\Bigg\} \\
         & = \min\bigg\{\frac{e^\eps + \frac{c_2 - c_1}{1 - c_1}\delta - 1}{(1 - c_1)e^\eps + c_2 - 1} \, , \, \frac{1 - (1-\delta)\frac{2}{e^\eps + 1}}{1 - c_1}\bigg\}.
\end{align*}
Since $\frac{c_2 - c_1}{1 - c_1} \in \mathbb{N}$ and $c_2 > 1$, it follows that $\frac{c_2 - c_1}{1 - c_1} \geq 2$, which in turn implies  $c_1 + c_2 \geq 2$. Now we want to show that if $c_1 + c_2 \geq 2$, then:
\[\frac{e^\eps + \frac{c_2 - c_1}{1 - c_1}\delta - 1}{(1 - c_1)e^\eps + c_2 - 1} \leq \frac{1 - (1-\delta)\frac{2}{e^\eps + 1}}{1 - c_1}.\]

Because $c_1 < 1 < c_2$ and $e^\varepsilon \ge 1$, the denominator
\(
(1-c_1)e^\varepsilon + c_2 - 1
\)
is strictly positive, so we can multiply both sides of the
inequality by it without changing the sign.

Set
\[
A \coloneqq (1 - c_1)e^\varepsilon + c_2 - 1 \;>\; 0.
\]
The claim is equivalent to
\[
(1 - c_1)(e^\varepsilon - 1) + (c_2 - c_1)\,\delta
\;\le\;
\Bigl(1 - (1-\delta)\,\tfrac{2}{e^\varepsilon + 1}\Bigr)A .
\] 
Bring the left–hand side to the right and factor out $(1-\delta)$:  
\begin{align*}
0
&\;\le\;
A - (1-\delta)\,\frac{2A}{e^\varepsilon + 1}
      -\bigl[(1 - c_1)(e^\varepsilon - 1) + (c_2 - c_1)\delta\bigr]\\[2pt]
&\;=\;
(c_2 - c_1)(1-\delta) - (1-\delta)\,\frac{2A}{e^\varepsilon + 1}\\[2pt]
&\;=\;
(1-\delta)\Bigl[(c_2 - c_1) - \frac{2A}{e^\varepsilon + 1}\Bigr].
\end{align*}

Since $1-\delta \ge 0$, we only need to prove  
\[
(c_2 - c_1)(e^\varepsilon + 1) \;\ge\; 2A .
\]

Substituting $A$ gives
\[
(c_2 - c_1)e^\varepsilon + (c_2 - c_1)
\;\ge\;
2(1 - c_1)e^\varepsilon + 2(c_2 - 1).
\]
Rearranging, we have:
\[
\bigl(c_2 + c_1 - 2\bigr)(e^\varepsilon - 1)\;\ge\;0.
\]

Because $e^\varepsilon - 1 \ge 0$ for every $\varepsilon \ge 0$,
the last inequality holds precisely when $c_1 + c_2 - 2 \ge 0$,
i.e.\ when $c_1 + c_2 \ge 2$.  This is exactly the hypothesis,
so the desired inequality is proved and we have:
\[\lambda^\star_{c_1,c_2,g_{\varepsilon,\delta}} = \frac{e^{\varepsilon} + \frac{c_2 - c_1}{1 - c_1} \, \delta - 1}{(1 - c_1) e^{\varepsilon} + c_2 - 1}.\]

\textbf{Step 3: Optimality proof}

Following Theorem \ref{thm: global functional}, for the obtained value of $\lambda_{c_1,c_2,g_{\eps,\delta}}$, 
the sampler \begin{equation*}
    \bQ^\star_{c_1,c_2,\mu,g_{\varepsilon,\delta}}(P) \coloneqq \lambda^\star_{c_1,c_2,g_{\varepsilon,\delta}} P + (1 - \lambda^\star_{c_1,c_2,g_{\varepsilon,\delta}})\mu
\end{equation*}
belongs to the class $\mathcal{Q}_{\mathcal{X}, \tilde{\mathcal{P}}, g_{\varepsilon,\delta}}$ and is minimax-optimal with respect to any $f$-divergence $D_f$, that is,
\begin{equation*}
    \sup_{P \in \tilde{\mathcal{P}}} D_f\big(P \, \| \, \bQ^\star_{c_1,c_2,\mu,g_{\varepsilon,\delta}}(P)\big)
    = \mathcal{R}\big(\mathcal{Q}_{\mathcal{X}, \tilde{\mathcal{P}}, g_{\varepsilon,\delta}}, \tilde{\mathcal{P}}, f\big).
\end{equation*}
\end{proof}

\section{Proofs of the main results
}\label{sec: proofs main}

\subsection{Proof of Proposition \ref{prop: triviality functional}}\label{appendix: proof of triviality functional}
\begin{proof}
    Based on Proposition~6 in~\citet{dong2022gaussian}, for a  trade-off function $g$, a sampler is \emph{$g$-FLDP} if and only if it is $\bigl(\varepsilon,\,1 + g^*(-\,e^\varepsilon)\bigr)$-LDP for all $\varepsilon \ge 0$. In the proof of Proposition \ref{prop: global appx linear}, we showed that if $\delta > \frac{(c_2-c_1e^\eps)(1-c_1)}{c_2-c_1}$, then the identity sampler $\bQ^I(P) = P$ satisfies $(\eps,\delta)$-LDP and has the trivial minimax risk zero.

Suppose the non-triviality condition (\ref{eqn: non-trivial condition condition functional}) does not hold, i.e., for any $\eps \geq 0$, we have:
\[1 + g^*(-e^\eps) > \frac{(c_2-c_1e^\eps)(1-c_1)}{c_2-c_1}.\]

In this case, the trivial sampler $\bQ^I(P) = P$ satisfies $(\eps,1 + g^*(-e^\eps))$-LDP for any $\eps \geq 0$ and therefore satisfies $g$-FLDP and has the trivial minimax risk zero. 
\end{proof}

\subsection{Proof of Theorem \ref{thm: global functional}}\label{proof: theorem global functional}

\begin{proof}
In this proof, we assume that the non-triviality condition (\ref{eqn: non-trivial condition condition functional}) is satisfied.

The proof of this theorem is directly based on Proposition~\ref{prop: global appx linear}. Accordingly, the global minimax-optimal sampler in Theorem~\ref{thm: global functional} builds upon the minimax-optimal sampler proposed in Proposition~\ref{prop: global appx linear}. The lower bound in the optimality proof of Proposition~\ref{prop: global appx linear} relies on the existence of disjoint sets \( A_1, \ldots, A_t \), each with measure \( \mu(A_i) = \alpha \), where \( t = \frac{1}{\alpha} \). In fact, the converse part of the optimality proof hinges on the \( (\alpha, \frac{1}{\alpha}, 1) \)-decomposability of \( \mu \), where \( \frac{1}{\alpha} = \frac{c_2 - c_1}{1 - c_1} \in \mathbb{N} \). This technical requirement motivates the assumption \( \frac{1}{\alpha} = \frac{c_2 - c_1}{1 - c_1} \in \mathbb{N} \) in Assumption~\ref{assumption: norm}.

\textbf{Step 1: To propose a $g$-FLDP sampler $\bQ^\star_{c_1,c_2,\mu,g}$} 

    From Proposition~\ref{prop: global appx linear}, we know that for each pair $(\varepsilon, \delta)$, the sampler
\[
\bQ^\star_{\varepsilon, \delta}(P) 
\;=\; 
\lambda^\star_{\varepsilon, \delta}\,P \;+\; \bigl(1 \;-\; \lambda^\star_{\varepsilon, \delta}\bigr)\mu,
\]
with 
\[
\lambda^\star_{\varepsilon, \delta} 
\;=\;
\frac{e^{\varepsilon} \;+\; \tfrac{c_2 - c_1}{1 - c_1}\,\delta \;-\; 1}
     {\,(1 - c_1)\,e^{\varepsilon} \;+\; c_2 \;-\; 1},
\]
is minimax-optimal in the class $\mathcal{Q}_{\mathcal{X}, \tilde{\mathcal{P}}, \varepsilon, \delta}$.

Next, based on Proposition~6 in~\citet{dong2022gaussian}, for a trade-off function $g$, a sampler is \emph{$g$-FLDP} if and only if it is $\bigl(\varepsilon,\,1 + g^*(-\,e^\varepsilon)\bigr)$-LDP for all $\varepsilon \ge 0$. Consequently, if we set $\delta(\varepsilon) = 1 + g^*\!\bigl(-\,e^\varepsilon\bigr)$ and define 
\[
\bQ^\star_{\varepsilon,\,\delta(\varepsilon)}(P) 
\;\coloneqq\; 
\lambda^\star_{\varepsilon,\,\delta(\varepsilon)}\,P \;+\; \Bigl(1 - \lambda^\star_{\varepsilon,\,\delta(\varepsilon)}\Bigr)\mu,
\]
then by Proposition~\ref{prop: global appx linear}, we have $\bQ^\star_{\varepsilon,\,\delta(\varepsilon)} \in \mathcal{Q}_{\mathcal{X}, \tilde{\mathcal{P}}, \varepsilon, \,\delta(\varepsilon)}$. 

We then use Lemma~\ref{lem:decrease-lambda}, which states that if $\bQ_{\lambda_1}(P) = \lambda_1\,P + (1 - \lambda_1)\mu$ is $\bigl(\varepsilon,\delta\bigr)$-LDP for $0 \leq \lambda_1 \leq 1$, then for \emph{any} $\lambda_2 \in [0,\,\lambda_1]$, the sampler $\bQ_{\lambda_2}(P) = \lambda_2\,P + (1 - \lambda_2)\mu$ is also $\bigl(\varepsilon,\delta\bigr)$-LDP.  

Therefore, if we define 
\[
\lambda^\star_{{c_1,c_2,g}} 
\;=\;
\inf_{\beta \,\ge\, 0}\;
\frac{\;e^{\beta } \;+\; \tfrac{c_2 - c_1}{\,1 - c_1\,}\,\Bigl[\,1 + g^*(-\,e^\beta )\Bigr] \;-\; 1}
     {\,(1 - c_1)\,e^\beta  + \bigl(c_2 - 1\bigr)},
\]
this guarantees that the sampler
\[
\bQ^\star_{{c_1,c_2,\mu,g}}(P) 
\;\coloneqq\; 
\lambda^\star_{c_1,c_2,g}\,P \;+\; \bigl(1 - \lambda^\star_{c_1,c_2,g}\bigr)\,\mu
\]
is $\bigl(\varepsilon,\,1 + g^*(-e^\varepsilon)\bigr)$-LDP for \emph{all} $\varepsilon \ge 0$. Note that the non-triviality constraint (\ref{eqn: non-trivial condition condition functional}) guarantees that there exists an $\eps \geq 0$ for which $1 + g^*(-e^\eps) \leq \frac{(c_2-c_1e^\eps)(1-c_1)}{c_2-c_1}$ and therefore $\lambda^\star_{c_1,c_2,g} \leq 1$. By Proposition~6 of~\citet{dong2022gaussian}, being $\bigl(\varepsilon,\,1 + g^*(-e^\varepsilon)\bigr)$-LDP for all $\varepsilon \ge 0$ is exactly the definition of being $g$-FLDP. 
Hence, $\bQ^\star_{c_1,c_2,\mu,g}$ is indeed a $g$-FLDP sampler and therefore belongs to $\mathcal{Q}_{\mathcal{X},\,\tilde{\mathcal{P}},\,g}$.  This completes the proof of the first step.

           \textbf{Step 2: To prove the optimality of $\bQ^\star_{c_1,c_2,\mu,g}$ } 

           Proposition~6 of~\citet{dong2022gaussian} tells us that if we set
\[
\delta(\varepsilon) \;=\; 1 \;+\; g^*\!\bigl(-\,e^\varepsilon\bigr),
\]
then for every $\varepsilon \geq 0$ we have
\begin{equation}\label{eqn: functional lower bounded by appx}
    \inf_{\bQ \,\in\, \mathcal{Q}_{\mathcal{X}, \tilde{\mathcal{P}}, g}} 
    \;\sup_{P \,\in\, \tilde{\mathcal{P}}}
    \,D_f\bigl(P \,\|\, \bQ(P)\bigr) 
    \;\;\ge\;\; 
    \inf_{\bQ \,\in\, \mathcal{Q}_{\mathcal{X}, \tilde{\mathcal{P}}, \varepsilon, \delta(\varepsilon)}} 
    \;\sup_{P \,\in\, \tilde{\mathcal{P}}}
    \,D_f\bigl(P \,\|\, \bQ(P)\bigr).
\end{equation}
On the other hand, define the sampler 
\[
\bQ^\star_{c_1,c_2,\mu,g}(P) 
\;\coloneqq\; 
\lambda^\star_{c_1,c_2,g}\,P \;+\; \bigl(1 - \lambda^\star_{c_1,c_2,g}\bigr)\mu,
\]
where
\[
\lambda^\star_{c_1,c_2,g} 
\;=\;
\inf_{\eps  \,\ge\, 0}
\;\;
\frac{\,e^\eps  \;+\; \tfrac{c_2 - c_1}{1 - c_1}\,\Bigl(1 + g^*\!\bigl(-\,e^\eps \bigr)\Bigr) \;-\; 1}
     {\,(1 - c_1)\,e^\eps  + \bigl(c_2 - 1\bigr)}.
\]
Let $\varepsilon^\star_g$ be the value of $\varepsilon$ that achieves this infimum. The infimum is attained at a finite value of $\varepsilon$ because, if we define
\[
h_g(\varepsilon) \coloneqq \frac{\,e^\varepsilon \;+\; \tfrac{c_2 - c_1}{1 - c_1}\,\Bigl(1 + g^*\!\bigl(-\,e^\varepsilon\bigr)\Bigr) \;-\; 1}
{(1 - c_1)\,e^\varepsilon + (c_2 - 1)},
\]
then we have
\[
h_g(0) \leq \lim_{\varepsilon \to \infty} h_g(\varepsilon).
\]
Together with the continuity of $h_g(\varepsilon)$, this implies that the infimum is achieved at some finite $\varepsilon_g^\star$.

To establish the continuity of $h_g(\varepsilon)$, we note that both the numerator and denominator are continuous functions of $\varepsilon$, and the denominator is never zero. To verify that the numerator is continuous, it suffices to show that $g^*$ is continuous. This follows directly from the definition of the convex conjugate and the fact that the original function $g$ is continuous on the compact interval $[0,1]$ and takes values in $[0,1]$ (see Proposition \ref{prop: trade-off}). Hence,
\[
\lambda^\star_{c_1,c_2,g}
\;=\;
\frac{\,e^{\varepsilon^\star_g} \;+\; \tfrac{c_2 - c_1}{1 - c_1}\,\Bigl(1 + g^*\!\bigl(-\,e^{\varepsilon^\star_g}\bigr)\Bigr) \;-\; 1}
     {\,(1 - c_1)\,e^{\varepsilon^\star_g} + \bigl(c_2 - 1\bigr)}.
\]
By construction, $\bQ^\star_{c_1,c_2,\mu,g}$ belongs to the set 
\(\mathcal{Q}_{\mathcal{X}, \tilde{\mathcal{P}}, \varepsilon^\star_g, \delta(\varepsilon^\star_g)}\).  
From Proposition~\ref{prop: global appx linear}, we know $\bQ^\star_{c_1,c_2,\mu,g}$ is minimax-optimal in 
\(\mathcal{Q}_{\mathcal{X}, \tilde{\mathcal{P}}, \varepsilon^\star_g, \delta(\varepsilon^\star_g)}\).  
Hence,
\begin{equation}\label{eqn: functional coincides appx optimal}
    \sup_{P \,\in\, \tilde{\mathcal{P}}}
    \,D_f\Bigl(P \,\bigl\|\, \bQ^\star_{c_1,c_2,\mu,g}(P)\Bigr)
    \;=\;
    \inf_{\bQ \,\in\, \mathcal{Q}_{\mathcal{X}, \tilde{\mathcal{P}}, \varepsilon^\star_g, \delta(\varepsilon^\star_g)}}
    \;\sup_{P \,\in\, \tilde{\mathcal{P}}}
    \,D_f\bigl(P \,\|\, \bQ(P)\bigr).
\end{equation}
Finally, combining \eqref{eqn: functional lower bounded by appx}, \eqref{eqn: functional coincides appx optimal}, and the trivial inequality 
\[
\inf_{\bQ \,\in\, \mathcal{Q}_{\mathcal{X}, \tilde{\mathcal{P}}, g}}
\;\sup_{P \,\in\, \tilde{\mathcal{P}}}
\,D_f\bigl(P \,\|\, \bQ(P)\bigr)
\;\;\le\;\;
\sup_{P \,\in\, \tilde{\mathcal{P}}}
\,D_f\Bigl(P \,\bigl\|\, \bQ^\star_{c_1,c_2,\mu,g}(P)\Bigr),
\]
we conclude:
\begin{align*}
    \sup_{P \,\in\, \tilde{\mathcal{P}}}
\,D_f\Bigl(P \,\bigl\|\, \bQ^\star_{c_1,c_2,\mu,g}(P)\Bigr) &\;=\;
    \inf_{\bQ \,\in\, \mathcal{Q}_{\mathcal{X}, \tilde{\mathcal{P}}, \varepsilon^\star_g, \delta(\varepsilon^\star_g)}}
    \;\sup_{P \,\in\, \tilde{\mathcal{P}}}
    \,D_f\bigl(P \,\|\, \bQ(P)\bigr) 
\\ & \;=\;
    \mathcal{R}\Bigl(\mathcal{Q}_{\mathcal{X}, \tilde{\mathcal{P}}, \varepsilon^\star_g, \delta(\varepsilon^\star_g)}, 
                    \tilde{\mathcal{P}}, f\Bigr)
                      \\
    & \;=\; \inf_{\bQ \,\in\, \mathcal{Q}_{\mathcal{X}, \tilde{\mathcal{P}}, g}}
    \;\sup_{P \,\in\, \tilde{\mathcal{P}}}
    \,D_f\bigl(P \,\|\, \bQ(P)\bigr) \\
    & \;=\; \mathcal{R}\big(\mathcal{Q}_{\mathcal{X}, \tilde{\mathcal{P}}, g}, \tilde{\mathcal{P}}, f\big),
\end{align*}
which completes the proof of the second step.

\textbf{Step 3: Computing the optimal value of $\mathcal{R}\big(\mathcal{Q}_{\mathcal{X}, \tilde{\mathcal{P}}, g}, \tilde{\mathcal{P}}, f\big)$}

It follows from the previous step that in order to compute $\mathcal{R}\big(\mathcal{Q}_{\mathcal{X}, \tilde{\mathcal{P}}, g}, \tilde{\mathcal{P}}, f\big)$, it suffices to compute $\mathcal{R}\Bigl(\mathcal{Q}_{\mathcal{X}, \tilde{\mathcal{P}}, \varepsilon^\star_g, \delta(\varepsilon^\star_g)}, 
                    \tilde{\mathcal{P}}, f\Bigr)$.

                    From Proposition \ref{prop: global appx linear}, we know that
\[ \mathcal{R}\Bigl(\mathcal{Q}_{\mathcal{X}, \tilde{\mathcal{P}}, \varepsilon^\star_g, \delta(\varepsilon^\star_g)}, 
                    \tilde{\mathcal{P}}, f\Bigr) = \frac{1 - r_1}{r_2 - r_1} \, f(r_2) + \frac{r_2 - 1}{r_2 - r_1} \, f(r_1)\]
for
\[
    r_1 = \frac{c_1}{c_2 - c_1} \cdot \frac{(1 - c_1)e^{\varepsilon^\star_g} + c_2 - 1}{1 - \delta(\varepsilon^\star_g)},
    \qquad
    r_2 = \frac{c_2}{c_2 - c_1} \cdot \frac{(1 - c_1)e^{\varepsilon^\star_g} + c_2 - 1}{e^{\varepsilon^\star_g} + \frac{c_2 - 1}{1 - c_1} \delta(\varepsilon^\star_g)}.
    \]

    We now compute $r_1$ and $r_2$ in terms of $\lambda_{c_1,c_2,g}^\star$. For \( \delta(\varepsilon^\star_g) = 1 + g^*(-e^{\varepsilon^\star_g}) \), we have
\[
r_1 = \frac{c_1}{c_2 - c_1} \cdot 
      \frac{(1 - c_1)e^{\varepsilon^\star_g} + c_2 - 1}
           {-g^*(-e^{\varepsilon^\star_g})},
\quad
r_2 = \frac{c_2}{c_2 - c_1} \cdot 
      \frac{(1 - c_1)e^{\varepsilon^\star_g} + c_2 - 1}
           {e^{\varepsilon^\star_g} + \left(\frac{c_2 - 1}{1 - c_1}\right)\left(1 + g^*(-e^{\varepsilon^\star_g})\right)}.
\]
Moreover, we have
\[
\lambda_{c_1,c_2,g}^\star 
\;=\;
\frac{\,e^{\varepsilon^\star_g} \;+\; \tfrac{c_2 - c_1}{1 - c_1}\,\Bigl(1 + g^*\!\bigl(-\,e^{\varepsilon^\star_g}\bigr)\Bigr) \;-\; 1}
     {\,(1 - c_1)\,e^{\varepsilon^\star_g} + \bigl(c_2 - 1\bigr)}.
\]
Let $\theta \coloneqq e^{\varepsilon_g^\star},
\quad
\delta(\varepsilon_g^\star) \coloneqq 1 + g^{\!*}\!\bigl(-\theta\bigr)$, we have
\begin{align}\label{eqn: lambda_g}
&\lambda_{c_1,c_2,g}^\star
      = \frac{\theta
              + \dfrac{c_2-c_1}{1-c_1}\,\delta(\varepsilon_g^\star)
              - 1}
             {(1-c_1)\theta + (c_2-1)}.
\end{align}

Therefore,
\[
r_1
  = \frac{c_1}{c_2-c_1}
    \frac{(1-c_1)\theta + c_2 - 1}{1-\delta(\varepsilon_g^\star)}.
\]
From \eqref{eqn: lambda_g},
\begin{align*}
1-(1-c_1)\lambda_{c_1,c_2,g}^\star
  &=\frac{(c_2-c_1)\bigl[1-\delta(\varepsilon_g^\star)\bigr]}
          {(1-c_1)\theta + (c_2-1)} .
\end{align*}
Hence
\[
r_1
 = \frac{c_1}{1-(1-c_1)\lambda_{c_1,c_2,g}^\star} .
\]

Similarly,
\[
r_2
  = \frac{c_2}{c_2-c_1}\,
    \frac{(1-c_1)\theta + c_2 - 1}
         {\theta + \dfrac{c_2-1}{1-c_1}\,\delta(\varepsilon_g^\star)} .
\]
Using \eqref{eqn: lambda_g} again,
\begin{align*}
(c_2-1)\lambda_{c_1,c_2,g}^\star + 1
  &= \frac{c_2-c_1}{(1-c_1)\theta + (c_2-1)}
     \Bigl[\theta + \tfrac{c_2-1}{1-c_1}\,\delta(\varepsilon_g^\star)\Bigr].
\end{align*}
Thus
\[
r_2
 = \frac{c_2}{(c_2-1)\lambda_{c_1,c_2,g}^\star + 1 }.
\]
In conclusion,
\[\mathcal{R}\Bigl(\mathcal{Q}_{\mathcal{X}, \tilde{\mathcal{P}}, \varepsilon^\star_g, \delta(\varepsilon^\star_g)}, 
                    \tilde{\mathcal{P}}, f\Bigr) = \mathcal{R}\big(\mathcal{Q}_{\mathcal{X}, \tilde{\mathcal{P}}, g}, \tilde{\mathcal{P}}, f\big) = \frac{1 - r_1}{r_2 - r_1} f(r_2) + \frac{r_2 - 1}{r_2 - r_1} f(r_1),
\]

for\[r_1=\frac{c_1}{1-(1-c_1)\lambda_{c_1,c_2,g}^\star} \quad \text{and } \quad r_2=\frac{c_2}{(c_2-1)\lambda_{c_1,c_2,g}^\star+1}.\]
\end{proof}

\subsection{Proof of Theorem \ref{thm: continuous global functional}}

\begin{proof}
  
Based on Theorem~\ref{thm: global functional}, we define \( \tilde{\mathcal{P}} = \tilde{\mathcal{P}}_{c_1, c_2, \mu} \) under Assumption~\ref{assumption: norm}. Moreover, suppose \( \mu \) is \( \left(\alpha, \frac{1}{\alpha}, 1\right) \)-decomposable with \( \alpha = \frac{1 - c_1}{c_2 - c_1} \). Then, the sampler defined as
\[
\bQ^\star_{c_1,c_2,\mu,g}(P) = \lambda^\star_{c_1,c_2,g} P + \left(1 - \lambda^\star_{c_1,c_2,g}\right) \mu
\]
satisfies $g$-FLDP and is minimax-optimal with respect to any $f$-divergence, where \( \lambda_{c_1,c_2,g}^\star \) is defined as
\[
\lambda^\star_{c_1,c_2,g} = \inf_{\beta \geq 0} 
\tfrac{e^\beta + \frac{c_2 - c_1}{1 - c_1} ( 1 + g^*(-e^\beta)) - 1}
     {(1 - c_1)e^\beta + c_2 - 1}.
\]

Theorem~\ref{thm: continuous global functional} is a special case of Theorem~\ref{thm: global functional}. Below, we demonstrate this reduction explicitly.

Recall that in the setup of Theorem~\ref{thm: global functional}, for a general sample space \( \mathcal{X} \), the universe \( \tilde{\mathcal{P}} \) is defined as
\[
\tilde{\mathcal{P}}_{c_1, c_2, \mu}
:= \left\{P\in \mathcal{P}(\mathcal{X}) : P \ll \mu,\ c_1 \leq \frac{dP}{d\mu} \leq c_2,\ \mu\text{-a.e.} \right\}.
\]

In the setup of Theorem~\ref{thm: continuous global functional}, \( \mathcal{X} = \mathbb{R}^n \), and \( \tilde{\mathcal{P}} \) is defined as
\[
\tilde{\mathcal{P}}_{c_1, c_2, h} := \left\{ P \in \mathcal{C}(\mathbb{R}^n) : c_1 h(x) \leq p(x) \leq c_2 h(x),\quad \forall x \in \mathbb{R}^n \right\}.
\]


Let \( \lambda \) denote the Lebesgue measure on \( \mathbb{R}^n \). We have \( \mathcal{X} = \mathbb{R}^n \), and \( \tilde{\mathcal{P}} = \tilde{\mathcal{P}}_{c_1, c_2, h} = \tilde{\mathcal{P}}_{c_1, c_2, \mu} \), where \( \mu \ll \lambda \) and \( \frac{d\mu}{d\lambda} = h \). Therefore, \( \bQ^\star_{c_1,c_2,\mu,g} = \bQ^\star_{c_1,c_2,h,g} \).

This is because for each \( P \in \tilde{\mathcal{P}} \) with corresponding PDF \( p(x) \), the chain rule of the Radon-Nikodym derivative gives:
\[
p(x) = \frac{dP}{d\lambda} = \frac{dP}{d\mu}(x) \cdot \frac{d\mu}{d\lambda}(x) = \frac{dP}{d\mu}(x) h(x).
\]

It remains to show that \( \mu \) is \( \left(\alpha, \frac{1}{\alpha}, 1\right) \)-decomposable for \( \alpha = \frac{1 - c_1}{c_2 - c_1} \). From Assumption~\ref{assumption: norm}, we know that \( t = \frac{c_2 - c_1}{1 - c_1} \in \mathbb{N} \).  Therefore we need to show that  \( \mu \) is \( \left(\frac{1}{t}, t, 1\right) \)-decomposable.

Since \( \mu \ll \lambda \), the function \( s \mapsto \mu((-\infty, s] \times \mathbb{R}^{n-1}) \) is continuous. As \( s \to -\infty \), this measure tends to \( 0 \), and as \( s \to \infty \), it tends to \( 1 \). By the intermediate value theorem, for each \( i \in [t] \), there exists a threshold \( s_i \in \mathbb{R} \) such that
\[
\mu((-\infty, s_i] \times \mathbb{R}^{n-1}) = \alpha i.
\]
We then define the sets \( A_1 = (-\infty, s_1] \times \mathbb{R}^{n-1} \), and for \( i \geq 2 \), set \( A_i = (s_{i-1}, s_i] \times \mathbb{R}^{n-1} \). These sets satisfy the requirements of \( \left(\frac{1}{t}, t, 1\right) \)-decomposability. Therefore, \( \mu \) is indeed \( \left(\frac{1}{t}, t, 1\right) \)-decomposable, completing the proof.

Therefore, Theorem~\ref{thm: global functional} contains Theorem~\ref{thm: continuous global functional} as a special case. Under Assumption~\ref{assumption: norm}, the sampler \( \mathbf{Q}_{c_1,c_2,h,g}^\star \), defined as a continuous distribution whose density is given by
\begin{equation}
   q^\star_g(P)(x) \coloneqq \lambda^\star_{c_1,c_2,g} p(x) + \left(1 - \lambda^\star_{c_1,c_2,g}\right) h(x), \quad \lambda^\star_{c_1,c_2,g} = \inf_{\beta \geq 0} 
\tfrac{e^\beta + \frac{c_2 - c_1}{1 - c_1} ( 1 + g^*(-e^\beta)) - 1}
     {(1 - c_1)e^\beta + c_2 - 1},
\end{equation}
satisfies \( g \)-FLDP and is minimax-optimal under any \( f \)-divergence. That is,
\begin{equation*}
    \sup_{P \in \tilde{\mathcal{P}}} D_f\big(P \,\|\, \bQ^\star_{c_1,c_2,h,g}(P)\big)
    = \mathcal{R}\big(\mathcal{Q}_{\mathbb{R}^n, \tilde{\mathcal{P}}, g}, \tilde{\mathcal{P}}, f\big)
    = \frac{1 - r_1}{r_2 - r_1} f(r_2) + \frac{r_2 - 1}{r_2 - r_1} f(r_1),
\end{equation*}
for $r_1=\frac{c_1}{1-(1-c_1)\lambda_{c_1,c_2,g}^\star}$ and $r_2=\frac{c_2}{(c_2-1)\lambda_{c_1,c_2,g}^\star+1}$.
\end{proof}
\subsection{Proof of Theorem \ref{thm: discrete global functional}}
\begin{proof}
    From Theorem \ref{thm: global functional}, we know that if $\tilde{\mathcal{P}} = \tilde{\mathcal{P}}_{c_1, c_2, \mu}$ is defined under Assumption \ref{assumption: norm} and  $\mu$ is $(\alpha,\frac{1}{\alpha},1)$-decomposable with $\alpha = \frac{1 - c_1}{c_2 - c_1}$, then the sampler
    $\bQ^\star_{c_1,c_2,\mu,g}(P) \coloneqq \lambda^\star_{c_1,c_2,g} P + (1 - \lambda^\star_{c_1,c_2,g})\mu$
    belongs to the class $\mathcal{Q}_{\mathcal{X}, \tilde{\mathcal{P}}, g}$ and is minimax-optimal under any $f$-divergence $D_f$; that is,
    \begin{equation*}
        \sup_{P \in \tilde{\mathcal{P}}} D_f\big(P \, \| \, \bQ^\star_{c_1,c_2,\mu,g}(P)\big)
        = \mathcal{R}\big(\mathcal{Q}_{\mathcal{X}, \tilde{\mathcal{P}}, g}, \tilde{\mathcal{P}}, f\big),
    \end{equation*}
In this case, 
    \[
    \lambda^\star_{c_1,c_2,g} 
    = \inf_{\beta \geq 0} 
    \frac{e^\beta + \frac{c_2 - c_1}{1 - c_1} \big( 1 + g^*(-e^\beta) \big) - 1}
         {(1 - c_1)e^\beta + c_2 - 1}.
    \]

 It can be shown that:
\[\mathcal{P}([k]) = \tilde{\mathcal{P}}_{0, k, \mu_k}
:= \left\{ P \in \mathcal{P}([k]) : \quad P \ll \mu_k,\quad 0 \leq \frac{dP}{d\mu_k} \leq k \quad \mu_k \text{-a.e.} \right\}
\]
Consider the following setting:  
\[
\tilde{\mathcal{P}} = \mathcal{P}([k]), \quad \mathcal{X} = [k], \quad c_1 = 0, \quad c_2 = k, \quad \mu = \mu_k,
\]
where \( \mu_k \) is the uniform distribution on \([k]\). In this case, we have \( \frac{c_2 - c_1}{1 - c_1} = k \in \mathbb{N} \), and \( \mu_k \) is \( \left(\alpha, \frac{1}{\alpha}, 1 \right) \)-decomposable with \( \alpha = \frac{1}{k} \). 
Moreover, for this specific instantiation, it can be shown that for all \( \varepsilon \geq 0 \),
\[
1 + g^*(-e^\varepsilon) \leq \frac{(c_2 - c_1 e^\varepsilon)(1 - c_1)}{c_2 - c_1} 
= \frac{(k)(1)}{k} = 1.
\]
This comes from the definition of the convex conjugate and the properties of the function \( g \):
\[
\forall \varepsilon \geq 0: \quad -1 \leq g^*(-e^\varepsilon) \leq 0,
\]
and hence,
\[
1 + g^*(-e^\varepsilon) \leq 1.
\]
Therefore, all conditions of Theorem \ref{thm: global functional} are satisfied and  Theorem \ref{thm: discrete global functional} is a special case of the more general case Theorem \ref{thm: global functional}. Hence, the sampler $
    \bQ^\star_{k,g}(P) \coloneqq \lambda^\star_{k,g} P + (1 - \lambda^\star_{k,g})\mu_k$
    belongs to the class $\mathcal{Q}_{[k], \tilde{\mathcal{P}}, g}$ and is minimax-optimal under any $f$-divergence $D_f$; that is,
    \begin{equation*}
        \sup_{P \in \tilde{\mathcal{P}}} D_f\big(P \, \| \, \bQ^\star_{k,g}(P)\big)
        = \mathcal{R}\big(\mathcal{Q}_{[k], \tilde{\mathcal{P}}, g}, \tilde{\mathcal{P}}, f\big),
    \end{equation*}
    where $\lambda^\star_{k,g}$ is the optimal solution to the following optimization problem:
    \[
    \lambda^\star_{k,g} 
    = \inf_{\beta \geq 0} 
    \frac{e^\beta + k \big( 1 + g^*(-e^\beta) \big) - 1}
         {e^\beta + k - 1}.
    \]
\end{proof}

\subsection{Proof of Corollary \ref{cor: special case of f-LDP}}\label{appendix: proof of g_eps}
\begin{proof}

Corollary~\ref{cor: special case of f-LDP} follows immediately from Corollary~\ref{cor: special case of f-LDP approximate} by setting \(\delta = 0\). With this choice, we have \(g_\varepsilon = g_{\varepsilon,0}\), thus the proof is identical.  Moreover, as in the proof of Theorem~\ref{thm: continuous global functional}, the continuous case of Corollary~\ref{cor: special case of f-LDP} can be shown to be a special instance of the more general result stated in Corollary~\ref{cor: special case of f-LDP approximate}. For completeness, a brief proof sketch is provided below.

\textbf{Step 1: Find $g^*_{\eps}$}
\begin{equation}
    g^*_\varepsilon(y) =
    \begin{cases}
        -1, & \text{if } y < -e^{\varepsilon}, \\[6pt]
        \dfrac{y - 1}{e^{\varepsilon} + 1}, & \text{if } -e^{\varepsilon} \leq y \leq -e^{-\varepsilon}, \\[6pt]
        y, & \text{if } y > -e^{-\varepsilon}.
    \end{cases}
\end{equation}

\textbf{Step 2: Solve the infimum to obtain $\lambda^\star_{c_1,c_2,g_\eps}$}

\begin{equation}
    \lambda^\star_{c_1,c_2,g_\eps} =
    \begin{cases}
        \frac{e^\eps - 1}{(e^\eps + 1)(1 - c_1)}, & \text{if } c_1 + c_2 < 2, \\[6pt]
        \frac{e^\eps - 1}{(1 - c_1)e^\eps + c_2 - 1}, & \text{if } c_1 + c_2 \geq 2.
    \end{cases}
\end{equation}

Since $\frac{c_2 - c_1}{1 - c_1} \in \mathbb{N}$ and $c_2 > 1$, it follows that $\frac{c_2 - c_1}{1 - c_1} \geq 2$, which in turn implies  $c_1 + c_2 \geq 2$. Hence:

\[\lambda^\star_{c_1,c_2,g_\eps} = \frac{e^\eps - 1}{(1 - c_1)e^\eps + c_2 - 1}.\]

\textbf{Step 3: Optimality proof}

Following Theorem \ref{thm: global functional}, for the obtained value of $\lambda_{c_1,c_2,g_{\eps}}$, 
the sampler \begin{equation*}
    \bQ^\star_{c_1,c_2,\mu,g_{\varepsilon}}(P) \coloneqq \lambda^\star_{c_1,c_2,g_{\varepsilon}} P + (1 - \lambda^\star_{c_1,c_2,g_{\varepsilon}})\mu
\end{equation*}
belongs to the class $\mathcal{Q}_{\mathcal{X}, \tilde{\mathcal{P}}, g_{\varepsilon,\delta}}$ and is minimax-optimal with respect to any $f$-divergence $D_f$, that is,
\begin{equation*}
    \sup_{P \in \tilde{\mathcal{P}}} D_f\big(P \, \| \, \bQ^\star_{c_1,c_2,\mu,g_{\varepsilon}}(P)\big)
    = \mathcal{R}\big(\mathcal{Q}_{\mathcal{X}, \tilde{\mathcal{P}}, g_{\varepsilon}}, \tilde{\mathcal{P}}, f\big).
\end{equation*}
Therefore, we have:
\begin{equation*}
    \sup_{P \in \tilde{\mathcal{P}}} D_f\big(P \,\|\, \bQ^\star_{c_1,c_2,h,g_\eps}(P)\big)
    = \mathcal{R}\big(\mathcal{Q}_{\mathbb{R}^n, \tilde{\mathcal{P}}, g_\eps}, \tilde{\mathcal{P}}, f\big)
    = \frac{1 - r_1}{r_2 - r_1} f(r_2) + \frac{r_2 - 1}{r_2 - r_1} f(r_1),
\end{equation*}
for $
r_1 = c_1 \cdot \frac{(1 - c_1)e^{\eps} + c_2 - 1}{c_2 - c_1 }$, and $
r_2 = \frac{c_2}{c_2 - c_1} \cdot \frac{(1 - c_1)e^{\eps} + c_2 - 1}{e^{\eps}}$.

The optimal value of 
\(\mathcal{R}\bigl(\mathcal{Q}_{\mathbb{R}^n,\tilde{\mathcal{P}},g_\varepsilon},
                    \tilde{\mathcal{P}},f\bigr)\)
is obtained directly from
Theorem~\ref{thm: continuous global functional}
by substituting
\(
\lambda_{c_1,c_2,g_\varepsilon}^{\star}
      =\dfrac{e^{\varepsilon}-1}{(1-c_1)e^{\varepsilon}+c_2-1}.
\)

\end{proof}

\subsection{Proof of Corollary \ref{cor: special case GDP}}\label{proof: GDP}

\begin{proof}
    From Theorem \ref{thm: global functional} we know that for
    \[
    \lambda^\star_{G_\nu} 
    = \inf_{\beta \geq 0} 
    \frac{e^\beta + \frac{c_2 - c_1}{1 - c_1} \big( 1 + G_\nu^*(-e^\beta) \big) - 1}
         {(1 - c_1)e^\beta + c_2 - 1}.
    \]
   the sampler  \eqref{eqn: continuous global functional opt mech} for $g=G_\nu$ belongs to $\mathcal{Q}_{\R^n, \tilde{\mathcal{P}}, G_\nu}$ and is minimax-optimal with respect to any $f$-divergence, that is,
\begin{equation*}
    \sup_{P \in \tilde{\mathcal{P}}} D_f\big(P \, \| \, \bQ^\star_{G_\nu}(P)\big)
    = \mathcal{R}\big(\mathcal{Q}_{\R^n, \tilde{\mathcal{P}},G_\nu}, \tilde{\mathcal{P}}, f\big).
\end{equation*}
    Now, we compute $G_\nu^*(-e^\beta)$. It is shown in Corollary 1 of \citet{dong2022gaussian} that 
\[
G_\nu^*(y)
  = y\,\Phi\Bigl(-\frac{\nu}{2}-\frac{1}{\nu}\log(-y)\Bigr)\;-\;
  \Phi\Bigl(-\frac{\nu}{2}+\frac{1}{\nu}\log(-y)\Bigr).
\]
When $y = -e^\beta$,
\[G^*_{\nu}\bigl(-e^{\beta}\bigr)
  \;=\;
  -\,e^{\beta}\,
    \Phi\Bigl(-\frac{\nu}{2}-\frac{\beta}{\nu}\Bigr)
  \;-\;
    \Phi\Bigl(-\frac{\nu}{2}+\frac{\beta}{\nu}\Bigr)\]
Therefore, we have:
\begin{align*}
\lambda^\star_{G_\nu} & =\inf_{\beta \geq 0} 
    \frac{e^\beta + \Big(\frac{c_2 - c_1}{1 - c_1}\Big) \bigg( 1 -e^\beta \Phi\left(-\frac{\nu}{2} - \frac{\beta}{\nu}\right) - \Phi\left(-\frac{\nu}{2} + \frac{\beta}{\nu}\right) \bigg) - 1}
         {(1 - c_1)e^\beta + c_2 - 1}\\
         & = \inf_{\beta \geq 0} 
    \frac{e^\beta + \Big(\frac{c_2 - c_1}{1 - c_1}\Big) \bigg( \Phi\left(\frac{\nu}{2} - \frac{\beta}{\nu}\right)  -e^\beta \Phi\left(-\frac{\nu}{2} - \frac{\beta}{\nu}\right)  \bigg) - 1}
         {(1 - c_1)e^\beta + c_2 - 1}.    
\end{align*}

\end{proof}

\subsection{Proof of Theorem \ref{thm: local functional}}
 
\begin{proof}
    \textbf{Step 1: Lower bound}

In the first step of the proof, we establish a lower bound for the local minimax objective function 
\(\mathcal{R}\big(\mathcal{Q}_{\mathcal{X},\tilde{\mathcal{P}},g},N_\gamma(P_0),f\big)\) as follows:
\begin{align} \label{eqn: lower bound local functional}
\mathcal{R}\big(\mathcal{Q}_{\mathcal{X},\tilde{\mathcal{P}},g},N_\gamma(P_0),f\big) 
&= \inf_{\bQ \in \mathcal{Q}_{\mathcal{X},\tilde{\mathcal{P}},g}} \, 
\sup_{P \in N_\gamma(P_0)} D_f\big(P \, \| \, \bQ(P)\big) \nonumber \\
&\geq \inf_{\bQ \in \mathcal{Q}_{\mathcal{X},N_\gamma(P_0),g}} \, 
\sup_{P \in N_\gamma(P_0)} D_f\big(P \, \| \, \bQ(P)\big) \nonumber \\
&= \mathcal{R}\big(\mathcal{Q}_{\mathcal{X},N_\gamma(P_0),g},N_\gamma(P_0),f\big).
\end{align}

The inequality in the second line follows from the fact that any sampler 
\(\bQ : \tilde{\mathcal{P}} \to \mathcal{P}(\mathcal{X})\) in 
\(\mathcal{Q}_{\mathcal{X}, \tilde{\mathcal{P}}, g}\) also belongs to 
\(\mathcal{Q}_{\mathcal{X}, N_\gamma(P_0), g}\). 
The third line is by definition of the global minimax risk in~\eqref{eqn: global minimax functional problem} when the universe $\mathcal{\tilde{P}}$ is restricted to \(N_\gamma(P_0)\).

\textbf{Step 2: Upper bound}

We now construct a sampler that achieves the lower bound. Specifically, we aim to find a sampler 
\(\bQ^\star_{g, N_\gamma(P_0)} \in \mathcal{Q}_{\mathcal{X}, \tilde{\mathcal{P}}, g}\) such that
\[
\sup_{P \in N_\gamma(P_0)} D_f\big(P \, \| \, \bQ^\star_{g, N_\gamma(P_0)}(P)\big) 
= \mathcal{R}\big(\mathcal{Q}_{\mathcal{X}, N_\gamma(P_0), g}, N_\gamma(P_0), f\big).
\]

Observe that the neighborhood \(N_\gamma(P_0)\) aligns with the class 
\(\tilde{\mathcal{P}} = \tilde{\mathcal{P}}_{c_1, c_2, \mu}\) defined in Theorem~\ref{thm: global functional}:
\[
\tilde{\mathcal{P}}_{c_1, c_2, \mu}
:= \left\{ P \in \mathcal{P}(\mathcal{X}) : \quad P \ll \mu, \quad 
c_1 \leq \frac{dP}{d\mu} \leq c_2 \quad \mu\text{-a.e.} \right\},
\]
\[
N_\gamma(P_0) := \left\{ P \in \mathcal{P}(\mathcal{X}) : \quad 
\frac{1}{\gamma} \leq \frac{dP}{dP_0}(x) \leq \gamma \quad \forall x \in \mathcal{X} \right\}.
\]

By substituting \(c_1 = \tfrac{1}{\gamma}\), \(c_2 = \gamma\), and \(\mu = P_0\), we obtain 
\(\tilde{\mathcal{P}}_{c_1, c_2, \mu} = N_\gamma(P_0)\).

Additionally, since \(\gamma \in \mathbb{N}\), it follows that 
\(\tfrac{c_2 - c_1}{1 - c_1} = \tfrac{\gamma - \frac{1}{\gamma}}{1 - \frac{1}{\gamma}} \in \mathbb{N}\), 
satisfying the integer condition in Theorem~\ref{thm: global functional}. The \((\alpha,\frac{1}{\alpha},1)\)-decomposability 
of \(P_0\) with \(\alpha = \tfrac{1}{\gamma + 1}\) also matches the requirement that $\mu$ is $(\alpha,\frac{1}{\alpha},1)$-decomposable with  
\(\alpha = \tfrac{1 - c_1}{c_2 - c_1}\).

Therefore, all the conditions in Theorem~\ref{thm: global functional} are satisfied. Thus, the optimal 
sampler \(\bQ^\star_g\) corresponding to 
\((c_1, c_2, \mu) = \left(\tfrac{1}{\gamma}, \gamma, P_0\right)\) achieves
\[
\sup_{P \in N_\gamma(P_0)} D_f\big(P \, \| \, \bQ^\star_g(P)\big)
= \mathcal{R}\big(\mathcal{Q}_{\mathcal{X}, N_\gamma(P_0), g}, N_\gamma(P_0), f\big).
\]

Since \(\bQ^\star_g\) is defined only on \(N_\gamma(P_0)\), we extend it to the larger domain 
\(\tilde{\mathcal{P}}\) by defining:
\[
\bQ^\star_{g, N_\gamma(P_0)}(P) := 
\begin{cases}
\bQ^\star_g(P), & \text{if } P \in N_\gamma(P_0), \\[6pt]
\bQ^\star_g(\hat{P}), & \text{otherwise},
\end{cases}
\]
where \( \hat{P} \in N_\gamma(P_0) \) is a distribution that minimizes \(  D_f(P \,\|\, P') \) over all $P' \in N_\gamma(P_0)$.

By construction, \(\bQ^\star_{g, N_\gamma(P_0)} \in \mathcal{Q}_{\mathcal{X}, \tilde{\mathcal{P}}, g}\), 
and for all \(P \in N_\gamma(P_0)\), we have:
\begin{align*}
\sup_{P \in N_\gamma(P_0)} D_f\big(P \, \| \, \bQ^\star_{g, N_\gamma(P_0)}(P)\big) 
&= \sup_{P \in N_\gamma(P_0)} D_f\big(P \, \| \, \bQ^\star_g(P)\big) \\
&= \mathcal{R}\big(\mathcal{Q}_{\mathcal{X}, N_\gamma(P_0), g}, N_\gamma(P_0), f\big).
\end{align*}

Hence, we obtain the upper bound:
\begin{align} \label{eqn: upper bound local functional}
\mathcal{R}\big(\mathcal{Q}_{\mathcal{X}, \tilde{\mathcal{P}}, g}, N_\gamma(P_0), f\big)
&= \inf_{\bQ \in \mathcal{Q}_{\mathcal{X}, \tilde{\mathcal{P}}, g}} \,
\sup_{P \in N_\gamma(P_0)} D_f\big(P \, \| \, \bQ(P)\big) \nonumber \\ 
&\leq \sup_{P \in N_\gamma(P_0)} D_f\big(P \, \| \, \bQ^\star_{g, N_\gamma(P_0)}(P)\big)  \nonumber \\ & = \mathcal{R}\big(\mathcal{Q}_{\mathcal{X}, N_\gamma(P_0), g}, N_\gamma(P_0), f\big).
\end{align}

Combining equations~\eqref{eqn: lower bound local functional} and~\eqref{eqn: upper bound local functional}, 
we conclude that
\[
\mathcal{R}\Big(\mathcal{Q}_{\mathcal{X}, \tilde{\mathcal{P}}, g}, N_\gamma(P_0), f\Big) 
= \mathcal{R}\Big(\mathcal{Q}_{\mathcal{X}, N_\gamma(P_0), g}, N_\gamma(P_0), f\Big).
\]
Note that this proof is stated for a general space \(\mathcal{X}\) and a general universe \(\tilde{\mathcal{P}} = \Ptm\). It extends directly to the continuous case with \(\mathcal{X} = \mathbb{R}^n\) and \(\tilde{\mathcal{P}} = \Pth\), following the same reduction argument used in the proof of Theorem~\ref{thm: continuous global functional}.
\end{proof}

\subsection{Proof of Theorem \ref{thm: local pure}} \label{proof: local pure}

\begin{proof}

let the global minimax formulation under $\eps$-LDP be defined as: 
\begin{equation}\label{eqn: global minimax pure problem}\mathcal{R}\big(\mathcal{Q}_{\mathcal{X},\tilde{\mathcal{P}},\eps},\tilde{\mathcal{P}},f\big) \coloneqq \inf_{\bQ \in \mathcal{Q}_{\mathcal{X},\tilde{\mathcal{P}},\eps}} \hspace{0.1cm} \sup_{P \in \tilde{\mathcal{P}}}\hspace{0.1cm} D_f\big(P \, \| \, \bQ(P)\big)
\end{equation}
Under certain assumptions, formally defined in Section~\ref{sec: global functional}, \citet{park2024exactly} derives the optimal sampler and optimal value for the optimization problem~\eqref{eqn: global minimax pure problem}.

\textbf{Step 1: Lower bound} 

In the first step of the proof, we establish a lower bound for the local minimax objective function 
\(\mathcal{R}\big(\mathcal{Q}_{\mathcal{X},\tilde{\mathcal{P}},\eps},N_\gamma(P_0),f\big)\) as follows:

\begin{align} \label{eqn: lower bound local pure}
\mathcal{R}\big(\mathcal{Q}_{\mathcal{X},\tilde{\mathcal{P}},\eps},N_\gamma(P_0),f\big) 
&= \inf_{\bQ \in \mathcal{Q}_{\mathcal{X},\tilde{\mathcal{P}},\eps}} \, 
\sup_{P \in N_\gamma(P_0)} D_f\big(P \, \| \, \bQ(P)\big) \nonumber \\
&\geq \inf_{\bQ \in \mathcal{Q}_{\mathcal{X},N_\gamma(P_0),\eps}} \, 
\sup_{P \in N_\gamma(P_0)} D_f\big(P \, \| \, \bQ(P)\big) \nonumber \\
&= \mathcal{R}\big(\mathcal{Q}_{\mathcal{X},N_\gamma(P_0),\eps},N_\gamma(P_0),f\big).
\end{align}

The inequality in the second line follows from the fact that any sampler 
\(\bQ : \tilde{\mathcal{P}} \to \mathcal{P}(\mathcal{X})\) in 
\(\mathcal{Q}_{\mathcal{X}, \tilde{\mathcal{P}}, \eps}\) also belongs to 
\(\mathcal{Q}_{\mathcal{X}, N_\gamma(P_0), \eps}\). 
The third line is by definition of the global minimax risk in~\eqref{eqn: global minimax pure problem} when the universe $\mathcal{\tilde{P}}$ is restricted to \(N_\gamma(P_0)\).

\textbf{Step 2: Upper bound} 

We now construct a sampler that achieves the lower bound. Specifically, we aim to find a sampler 
\(\bQ^\star_{\eps, N_\gamma(P_0)} \in \mathcal{Q}_{\mathcal{X}, \tilde{\mathcal{P}}, \eps}\) such that
\[
\sup_{P \in N_\gamma(P_0)} D_f\big(P \, \| \, \bQ^\star_{\eps, N_\gamma(P_0)}(P)\big) 
= \mathcal{R}\big(\mathcal{Q}_{\mathcal{X}, N_\gamma(P_0), \eps}, N_\gamma(P_0), f\big).
\]




Let sampler \(\bQ^\star_\eps: \ngb \to \mathcal{P}(\mathcal{X})\) be an optimal sampler for the global minimax problem (\ref{eqn: global minimax pure problem}) where the universe $\mathcal{\tilde{P}}$ is restricted to \(N_\gamma(P_0)\). 
In other words, $\bQ^\star_\eps$ achieves
\[
\sup_{P \in N_\gamma(P_0)} D_f\big(P \, \| \, \bQ^\star_\eps(P)\big)
= \mathcal{R}\Big(\mathcal{Q}_{\mathcal{X}, N_\gamma(P_0), \eps}, N_\gamma(P_0), f\Big).
\]

Since \(\bQ^\star_\eps\) is defined only on \(N_\gamma(P_0)\), we extend it to the larger domain 
\(\tilde{\mathcal{P}}\) by defining:
\begin{equation}\label{eqn: local pure optimal mech}
    \bQ^\star_{\eps, N_\gamma(P_0)}(P) := 
\begin{cases}
\bQ^\star_\eps(P), & \text{if } P \in N_\gamma(P_0), \\[6pt]
\bQ^\star_\eps(\hat{P}), & \text{otherwise},
\end{cases}    
\end{equation}

where \( \hat{P} \in N_\gamma(P_0) \) is a distribution that minimizes \(  D_f(P \,\|\, P') \) over all $P' \in N_\gamma(P_0)$.

By construction, \(\bQ^\star_{\eps, N_\gamma(P_0)} \in \mathcal{Q}_{\mathcal{X}, \tilde{\mathcal{P}}, \eps}\), 
and for all \(P \in N_\gamma(P_0)\), we have:
\begin{align*}
\sup_{P \in N_\gamma(P_0)} D_f\big(P \, \| \, \bQ^\star_{\eps, N_\gamma(P_0)}(P)\big) 
&= \sup_{P \in N_\gamma(P_0)} D_f\big(P \, \| \, \bQ^\star_\eps(P)\big) \\
&= \mathcal{R}\big(\mathcal{Q}_{\mathcal{X}, N_\gamma(P_0), \eps}, N_\gamma(P_0), f\big).
\end{align*}

Hence, we obtain the upper bound:
\begin{align} \label{eqn: upper bound local pure}
\mathcal{R}\big(\mathcal{Q}_{\mathcal{X}, \tilde{\mathcal{P}}, \eps}, N_\gamma(P_0), f\big)
&= \inf_{\bQ \in \mathcal{Q}_{\mathcal{X}, \tilde{\mathcal{P}}, \eps}} \,
\sup_{P \in N_\gamma(P_0)} D_f\big(P \, \| \, \bQ(P)\big) \nonumber \\ 
&\leq \sup_{P \in N_\gamma(P_0)} D_f\big(P \, \| \, \bQ^\star_{\eps, N_\gamma(P_0)}(P)\big)  \nonumber \\ & = \mathcal{R}\big(\mathcal{Q}_{\mathcal{X}, N_\gamma(P_0), \eps}, N_\gamma(P_0), f\big).
\end{align}

Combining equations~\eqref{eqn: lower bound local pure} and~\eqref{eqn: upper bound local pure}, 
we conclude that
\[
\mathcal{R}\Big(\mathcal{Q}_{\mathcal{X}, \tilde{\mathcal{P}}, \eps}, N_\gamma(P_0), f\Big) 
= \mathcal{R}\Big(\mathcal{Q}_{\mathcal{X}, N_\gamma(P_0), \eps}, N_\gamma(P_0), f\Big).
\] 

\textbf{Step 3: Local minimax risk in a closed form}

Under the same non-triviality assumption as \citet{park2024exactly}, for the pure LDP case we assume \(\varepsilon < 2\log(\gamma)\); otherwise, the identity sampler becomes the trivial local minimax-optimal sampler with zero minimax risk.

Consider the case where $2 \log(\gamma) \le \eps$. In this case, for any $P_1, P_2 \in N_{\gamma}(P_0)$, we have 
$p_1(x)/p_2(x) \leq \frac{\gamma}{\frac{1}{\gamma}} < e^{\varepsilon}$, hence we can easily observe that the sampler 
$\mathbf{Q}^I_\eps$ defined as $\mathbf{Q}^I_\eps(P) = P$ for all $P \in N_{\gamma}(P_0)$ satisfies 
$\eps$-LDP and  $\mathcal{R}\big(\mathcal{Q}_{\mathcal{X},N_\gamma(P_0),\eps},N_\gamma(P_0),f\big) = 0$. Since \(\bQ^I_\eps\) is defined only on \(N_\gamma(P_0)\), we extend it to the larger domain 
\(\tilde{\mathcal{P}}\) by defining:
\[
\bQ^I_{\eps, N_\gamma(P_0)}(P) := 
\begin{cases}
\bQ^I_\eps(P), & \text{if } P \in N_\gamma(P_0), \\[6pt]
\bQ^I_\eps(\hat{P}), & \text{otherwise},
\end{cases}
\]
where \( \hat{P} \in N_\gamma(P_0) \) is a distribution that minimizes \(  D_f(P \,\|\, P') \) over all $P' \in N_\gamma(P_0)$.

We can conclude $\mathcal{R}\big(\mathcal{Q}_{\mathcal{X},\tilde{\mathcal{P}},\eps},N_\gamma(P_0),f\big) = 0$. Therefore, to exclude trivial samplers we assume $\eps < 2\log(\gamma)$.

In order to obtain the local minimax risk under $\varepsilon$-LDP, we refer to Theorem C.4 in \citet{park2024exactly}.

\begin{theorem}[Theorem C.4 of \citet{park2024exactly}]\label{thm: global pure}
Let $\tilde{\mathcal{P}} = \tilde{\mathcal{P}}_{c_1, c_2, \mu}$ be a set of distributions such that the normalization condition~\eqref{eqn: normalization condition} holds, and suppose $\mu$ is $(\alpha,\frac{1}{\alpha},1)$-decomposable with $\alpha = \frac{1 - c_1}{c_2 - c_1}$. Define
\[
b \coloneqq \frac{c_2 - c_1}{(e^\varepsilon - 1)(1 - c_1) + c_2 - c_1}, \quad
r_1 \coloneqq \frac{c_1}{b}, \quad
r_2 \coloneqq \frac{c_2}{b e^\varepsilon}.
\]
Then, the optimal value of the problem~\eqref{eqn: global minimax pure problem} is given by
\[
\mathcal{R}\big(\mathcal{Q}_{\mathcal{X}, \tilde{\mathcal{P}}, \varepsilon}, \tilde{\mathcal{P}}, f\big)
= \frac{1 - r_1}{r_2 - r_1} \, f(r_2) + \frac{r_2 - 1}{r_2 - r_1} \, f(r_1).
\]

Furthermore, the sampler $\bQ^\star_{c_1, c_2, \mu, \varepsilon}$ defined below is an optimal solution to problem~\eqref{eqn: global minimax pure problem} under any $f$-divergence $D_f$. For each \(P \in \tilde{\mathcal{P}}\), let \(\bQ^\star_{c_1, c_2, \mu, \varepsilon}(P) = Q\) be a probability measure $Q \ll \mu$, such that
\[
\frac{dQ}{d\mu}(x) 
= 
\mathrm{clip}\left(
\frac{1}{r_P} \frac{dP}{d\mu}(x)
\;;\;
b,
\; b e^\varepsilon
\right),
\]
where \(r_P\) is the normalizing constant ensuring that \(\int \frac{dQ}{d\mu} \, d\mu(x) = 1\).
\end{theorem}

Observe that the neighborhood \(N_\gamma(P_0)\) aligns with the class 
\(\tilde{\mathcal{P}} = \tilde{\mathcal{P}}_{c_1, c_2, \mu}\) defined in Theorem~\ref{thm: global functional}:
\[
\tilde{\mathcal{P}}_{c_1, c_2, \mu}
:= \left\{ P \in \mathcal{P}(\mathcal{X}) : \quad P \ll \mu, \quad 
c_1 \leq \frac{dP}{d\mu} \leq c_2 \quad \mu\text{-a.e.} \right\},
\]
\[
N_\gamma(P_0) := \left\{ P \in \mathcal{P}(\mathcal{X}) : \quad 
\frac{1}{\gamma} \leq \frac{dP}{dP_0}(x) \leq \gamma \quad \forall x \in \mathcal{X} \right\}.
\]

By substituting \(c_1 = \tfrac{1}{\gamma}\), \(c_2 = \gamma\), and \(\mu = P_0\), we obtain 
\(\tilde{\mathcal{P}}_{c_1, c_2, \mu} = N_\gamma(P_0)\).

Additionally, the \((\alpha,\frac{1}{\alpha},1)\)-decomposability 
of \(P_0\) with \(\alpha = \tfrac{1}{\gamma + 1}\) also matches the requirement that $\mu$ is $(\alpha,\frac{1}{\alpha},1)$-decomposable with  
\(\alpha = \tfrac{1 - c_1}{c_2 - c_1}\). We now aim to compute the closed-form expression for
\[
\mathcal{R}\Big(\mathcal{Q}_{\mathcal{X}, N_\gamma(P_0), \varepsilon}, N_\gamma(P_0), f\Big).
\]
By the equivalence \(\tilde{\mathcal{P}}_{c_1, c_2, \mu} = N_\gamma(P_0)\) and Theorem~\ref{thm: global pure} applied with \((c_1, c_2, \mu) = \left(\tfrac{1}{\gamma}, \gamma, P_0\right)\), we have:
\[
\mathcal{R}\Big(\mathcal{Q}_{\mathcal{X}, N_\gamma(P_0), \varepsilon}, N_\gamma(P_0), f\Big)
= \frac{1 - r_1}{r_2 - r_1} \, f(r_2) + \frac{r_2 - 1}{r_2 - r_1} \, f(r_1),
\]
where
\[
b \coloneqq \frac{\gamma - \tfrac{1}{\gamma}}{(e^\varepsilon - 1)\left(1 - \tfrac{1}{\gamma}\right) + \gamma - \tfrac{1}{\gamma}} 
= \frac{\gamma + 1}{\gamma + e^\varepsilon}, \quad
r_1 \coloneqq \frac{1}{\gamma b} 
= \frac{e^\varepsilon + \gamma}{\gamma(\gamma + 1)}, \quad
r_2 \coloneqq \frac{\gamma}{b e^\varepsilon} 
= \frac{\gamma(e^\varepsilon + \gamma)}{e^\varepsilon(\gamma + 1)}.
\]
Since
\[
\mathcal{R}\Big(\mathcal{Q}_{\mathcal{X}, \tilde{\mathcal{P}}, \eps}, N_\gamma(P_0), f\Big) 
= \mathcal{R}\Big(\mathcal{Q}_{\mathcal{X}, N_\gamma(P_0), \eps}, N_\gamma(P_0), f\Big),
\] the optimal minimax risk is obtained.

\medskip

\textbf{Step 4: Optimal sampler} 

From the upper bound proof of the theorem, we know that the sampler 
\(\bQ^\star_{\varepsilon, N_\gamma(P_0)}\), defined in \eqref{eqn: local pure optimal mech}, is an optimal sampler, provided that
\[
\sup_{P \in N_\gamma(P_0)} D_f\big(P \, \| \, \bQ^\star_\varepsilon(P)\big)
= \mathcal{R}\Big(\mathcal{Q}_{\mathcal{X}, N_\gamma(P_0), \varepsilon}, N_\gamma(P_0), f\Big).
\]
Therefore, it suffices to construct a sampler 
\(\bQ^\star_\varepsilon : N_\gamma(P_0) \to \mathcal{P}(\mathcal{X})\) that belongs to $\mathcal{Q}_{\mathcal{X}, \ngb, \varepsilon}$
and satisfies the above equality.

By applying Theorem~\ref{thm: global pure} with 
\((c_1, c_2, \mu) = \left(\tfrac{1}{\gamma}, \gamma, P_0\right)\), 
we obtain the following definition for such a sampler. For each \(P \in N_\gamma(P_0)\), define \(\bQ^\star_\varepsilon(P) = Q\), where \(Q\) is a probability measure satisfying $Q \ll P_0$ and given by
\[
\frac{dQ}{dP_0}(x) 
= 
\mathrm{clip}\left(
\frac{1}{r_P} \frac{dP}{dP_0}(x)
\;;\;
\frac{\gamma + 1}{\gamma + e^\varepsilon},
\; \frac{\gamma + 1}{\gamma + e^\varepsilon} e^\varepsilon
\right),
\]
with \(r_P\) denoting the normalizing constant ensuring that 
\(\int \frac{dQ}{dP_0} \, dP_0(x) = 1\).

We extend this sampler to a function over all of \(\mathcal{P}(\mathcal{X})\) by defining
\[
\bQ^\star_{\varepsilon, N_\gamma(P_0)}(P) \coloneqq 
\begin{cases}
\bQ^\star_\varepsilon(P), & \text{if } P \in N_\gamma(P_0), \\[6pt]
\bQ^\star_\varepsilon(\hat{P}), & \text{otherwise},
\end{cases}
\]
where \( \hat{P} \in N_\gamma(P_0) \) is a distribution that minimizes \(  D_f(P \,\|\, P') \) over all $P' \in N_\gamma(P_0)$.

Then \(\bQ^\star_{\varepsilon, N_\gamma(P_0)} \in \mathcal{Q}_{\mathcal{X}, \tilde{\mathcal{P}}, \varepsilon}\), and the following holds:
\begin{equation*}
\sup_{P \in N_\gamma(P_0)} 
D_f\big(P \,\|\, \bQ^\star_{\varepsilon, N_\gamma(P_0)}(P)\big)
= 
\mathcal{R}\Big(\mathcal{Q}_{\mathcal{X}, \tilde{\mathcal{P}}, \varepsilon}, N_\gamma(P_0), f\Big).
\end{equation*}

Note that this proof is stated for a general space \(\mathcal{X}\) and a general universe \(\tilde{\mathcal{P}} = \Ptm\). It extends directly to the continuous case with \(\mathcal{X} = \mathbb{R}^n\) and \(\tilde{\mathcal{P}} = \Pth\), following the same reduction argument used in the proof of Theorem~\ref{thm: continuous global functional}.
\end{proof}

\subsection{Proof of Proposition \ref{prop: pointwise}}\label{appendix: proof pointwise}
\begin{proof}
    To prove this result, we need to define preliminary samplers and the universe set. Let
\[
\tilde{\mathcal{P}}_{c_1, c_2, \mu}
:= \left\{ P \in \mathcal{P}(\mathcal{X}) : P \ll \mu, \quad c_1 \leq \frac{dP}{d\mu} \leq c_2 \quad \mu\text{-a.e.} \right\},
\]
Define a sampler $\mathbf{Q}^\star_{c_1,c_2,\mu,\eps} \in \mathcal{Q}_{\mathcal{X}, \tilde{\mathcal{P}}, \eps}$ as follows:

For each  $P \in \tilde{\mathcal{P}}$, $\mathbf{Q}^\star_{c_1,c_2,\mu,\eps}(P) \coloneqq Q$ is defined as a probability measure such that $Q \ll \mu$  and
\[
\frac{dQ}{d\mu}(x) = \textnormal{clip} \left( \frac{1}{r_P} \frac{dP}{d\mu}(x) ; b, b e^{\eps} \right),
\]
where  $r_P > 0$ is a constant depending on  $P$ so that  $\int \frac{dQ}{d\mu} d\mu(x) = 1$.

\begin{proposition}[Proposition C.8. of \citet{park2024exactly}]\label{prop: pointwise optimal Park et al.}
For any $P \in \Ptm$ and any $f$-divergences $D_f$, we have
    \[
D_f\left(P \,\|\, \mathbf{Q}^\star_{c_1,c_2,\mu,\eps}(P)\right)
= \inf_{\bQ \in \Pt_{b,be^\eps,\mu}} D_f(P \,\|\, \bQ),
\]
for \[
b = \frac{c_2 - c_1}{(e^{\eps} - 1)(1 - c_1) + c_2 - c_1}.
\]
\end{proposition}
We want to prove that if we define the neighborhood as 
\begin{equation*}
N_\gamma(P_0) \coloneqq  \left\{ P \in \mathcal{P}(\mathcal{X}) :\quad E_\gamma(P \, \| \, P_0) =  E_\gamma(P_0 \, \| \, P) = 0  \right\},
\end{equation*}
and let \( \bQ^\star_\varepsilon \) denote the optimal sampler from Theorem~\ref{thm: local pure}, and let \( \bQ^\star_{g_\varepsilon} \) be the instantiation of Theorem~\ref{thm: local functional} with \( g = g_\varepsilon \). Then, for all \( P \in N_\gamma(P_0) \),
\[
D_f(P \,\|\, \bQ^\star_\varepsilon) \leq D_f(P \,\|\, \bQ^\star_{g_\varepsilon}).
\]
Note that this inequality is equivalent to the condition that for all \( P \in N_\gamma(P_0) \),
\[
D_f(P \,\|\, \bQ^\star_{\varepsilon,\ngb}) \leq D_f(P \,\|\, \bQ^\star_{g_\varepsilon,\ngb}),
\]
which is the ultimate result we want to prove in this proposition. 

Therefore, it suffices to show that for $\Ptm = \ngb$, we have $\bQ^\star_\eps = \mathbf{Q}^\star_{c_1,c_2,\mu,\eps}$ and that $\bQ^\star_{g_\varepsilon} \in \Pt_{b,be^\eps,\mu}$.

It follows that for $c_1 = \frac{1}{\gamma}$, $c_2 = \gamma$, and $\mu = P_0$, we have the equivalence $\Ptm = \ngb$. For this instantiation of $\Ptm$, the optimal sampler $\mathbf{Q}^\star_{c_1,c_2,\mu,\eps}$ for input $P$
is defined as a probability measure such that $Q \ll P_0$ and
\[
\frac{dQ}{dP_0}(x) = \textnormal{clip} \left( \frac{1}{r_P} \frac{dP}{dP_0}(x) ; b, b e^{\eps} \right), 
\]
where
\[
b = \frac{\gamma - \frac{1}{\gamma}}{(e^{\eps} - 1)(1 - \frac{1}{\gamma}) + \gamma - \frac{1}{\gamma}} = \frac{\gamma + 1}{e^\eps + \gamma} .
\]
This is exactly  \( \bQ^\star_\varepsilon \), the optimal sampler from Theorem~\ref{thm: local pure}. Thus, it remains to prove that $\bQ^\star_{g_\varepsilon} \in \Pt_{b,be^\eps,P_0}$ for $b =  \frac{\gamma + 1}{e^\eps + \gamma}$.

Combining Theorem~\ref{thm: local functional} and Corollary~\ref{cor: special case of f-LDP}, we obtain
\[
\bQ^\star_{g_\varepsilon}(P) \coloneqq \lambda^\star_{g_\varepsilon} P + (1 - \lambda^\star_{g_\varepsilon}) P_0, \quad 
\lambda^\star_{g_\varepsilon} = \frac{e^\varepsilon - 1}{(1 - \frac{1}{\gamma})e^\varepsilon + \gamma - 1} = \frac{\gamma(e^\varepsilon - 1)}{(\gamma - 1)(e^\eps + \gamma)}.
\]

Recall that
\[
\tilde{\mathcal{P}}_{b, be^\eps, P_0}
:= \left\{ Q \in \mathcal{P}(\mathcal{X}) : Q \ll P_0, \quad b \leq \frac{dQ}{dP_0} \leq be^\eps \quad P_0\text{-a.e.} \right\},
\]
and
\[
\ngb
:= \left\{ P \in \mathcal{P}(\mathcal{X}) : P \ll P_0, \quad \frac{1}{\gamma} \leq \frac{dP}{dP_0} \leq \gamma \quad P_0\text{-a.e.} \right\}.
\]

Therefore, if $P \in \ngb$, we have
\[
\frac{d\bQ^\star_{g_\varepsilon}}{dP_0} = \lambda^\star_{g_\varepsilon} \frac{dP}{dP_0} + (1 - \lambda^\star_{g_\varepsilon}).
\]
It follows that $\bQ^\star_{g_\varepsilon} \ll P_0$ and
\begin{align*}
    \frac{\lambda^\star_{g_\varepsilon}}{\gamma} + (1 - \lambda^\star_{g_\varepsilon})\leq \frac{d\bQ^\star_{g_\varepsilon}}{dP_0} \le \gamma \lambda^\star_{g_\varepsilon} + (1 - \lambda^\star_{g_\varepsilon}) \quad P_0\text{-a.e.}.
\end{align*}
Equivalently,
\begin{align*}
   \frac{e^\varepsilon - 1}{(\gamma - 1)(e^\eps + \gamma)} + \frac{\gamma^2 -e^\eps}{(\gamma - 1)(e^\eps + \gamma)}\leq \frac{d\bQ^\star_{g_\varepsilon}}{dP_0} \le \frac{\gamma^2(e^\varepsilon - 1)}{(\gamma - 1)(e^\eps + \gamma)} + \frac{\gamma^2 -e^\eps}{(\gamma - 1)(e^\eps + \gamma)} \quad P_0\text{-a.e.}.
\end{align*}
As a result,
\begin{align*}
    \frac{\gamma^2 -1}{(\gamma - 1)(e^\eps + \gamma)}\leq \frac{d\bQ^\star_{g_\varepsilon}}{dP_0} \le   \frac{(\gamma^2 -1)e^\eps}{(\gamma - 1)(e^\eps + \gamma)} \quad P_0\text{-a.e.}.
\end{align*}
Or equivalently,
\begin{align*}
    \frac{\gamma + 1}{e^\eps + \gamma}\leq \frac{d\bQ^\star_{g_\varepsilon}}{dP_0} \le   \frac{(\gamma + 1)e^\eps}{e^\eps + \gamma} \quad P_0\text{-a.e.}.
\end{align*}
This shows that $\bQ^\star_{g_\eps} \in \tilde{\mathcal{P}}_{b, be^\eps, P_0}$ and completes the proof.

\end{proof}

\section{Detailed experimental setup and additional experiments
}\label{sec: appendix experimental setup and edditional experiments}

\subsection{Mixture of Laplace distributions}\label{appendix: laplace mixture}


\textbf{Proof of $\mathbf{\tilde{\mathcal{P}}_{\mathcal{L}} \subseteq \tilde{\mathcal{P}}_{e^{-1/b},\, e^{1/b},\, h_{\mathcal{L}}}}$ (Example \ref{example: Gaussian Laplace universe})}

Let the mixture of Laplace distributions with fixed scale parameter \( b \) be defined as
\[
\tilde{\mathcal{P}}_{\mathcal{L}} = \left\{ \sum_{i=1}^k \lambda_i \, \mathcal{L}(m_i, b) : k \in \mathbb{N},\, \lambda_i \geq 0,\, \sum_{i=1}^k \lambda_i = 1,\, \|m_i\|_1 \leq 1 \right\},
\]
where \( \mathcal{L}(m, b) \) denotes the \(n\)-dimensional Laplace distribution with mean \( m \in \mathbb{R}^n \) and scale parameter \( b > 0 \). We aim to show that
\[
\tilde{\mathcal{P}}_{\mathcal{L}} \subseteq \tilde{\mathcal{P}}_{e^{-1/b},\, e^{1/b},\, h_{\mathcal{L}}},
\]
where \( h_{\mathcal{L}} \) is the density of the zero-mean \(n\)-dimensional Laplace distribution with scale parameter \( b \).

The density of an \(n\)-dimensional Laplace distribution with location \( m \in \mathbb{R}^n \) and scale \( b \) is given by
\[
\mathcal{L}(x \mid m, b) = \frac{1}{(2b)^n} \exp\left( -\frac{\|x - m\|_1}{b} \right),
\quad x \in \mathbb{R}^n.
\]

Now, fix any \( m \in \mathbb{R}^n \) such that \( \|m\|_1 \leq 1 \). Then for any \( x \in \mathbb{R}^n \), we have the following bounds:
\begin{align*}
\frac{\|x\|_1 - \|m\|_1}{b}
&\leq \frac{\|x - m\|_1}{b}
\leq \frac{\|x\|_1 + \|m\|_1}{b} \\
\Rightarrow \quad
\frac{\|x\|_1 - 1}{b}
&\leq \frac{\|x - m\|_1}{b}
\leq \frac{\|x\|_1 + 1}{b}.
\end{align*}
Exponentiating both sides, we get:
\[
\exp\left( -\frac{\|x\|_1 + 1}{b} \right)
\leq \exp\left( -\frac{\|x - m\|_1}{b} \right)
\leq \exp\left( -\frac{\|x\|_1 - 1}{b} \right).
\]
Multiplying by the constant \( \frac{1}{(2b)^n} \), we obtain:
\[
\frac{e^{-1/b}}{(2b)^n} \exp\left( -\frac{\|x\|_1}{b} \right)
\leq \mathcal{L}(x \mid m, b)
\leq \frac{e^{1/b}}{(2b)^n} \exp\left( -\frac{\|x\|_1}{b} \right),
\]
which implies
\[
e^{-1/b} \, \mathcal{L}(x \mid \mathbf{0}, b) \leq \mathcal{L}(x \mid m, b) \leq e^{1/b} \, \mathcal{L}(x \mid \mathbf{0}, b).
\]
For any distribution \( P \in \tilde{\mathcal{P}}_{\mathcal{L}} \), we can write
\[
p(x) = \sum_{i=1}^k \lambda_i \, \mathcal{L}(x \mid m_i, b).
\]
Since each \( m_i \) satisfies \( \|m_i\|_1 \leq 1 \), the above inequality applies to each mixture component. Thus, for all \( x \in \mathbb{R}^n \),
\[
e^{-1/b} \, \mathcal{L}(x \mid \mathbf{0}, b)
\leq \sum_{i=1}^k \lambda_i \mathcal{L}(x \mid m_i, b)
\leq e^{1/b} \, \mathcal{L}(x \mid \mathbf{0}, b).
\]
Defining \( p_0(x) = \mathcal{L}(x \mid \mathbf{0}, b) \), we conclude that for every \( P \in \tilde{\mathcal{P}}_{\mathcal{L}} \),
\[
e^{-1/b} \leq \frac{p(x)}{p_0(x)} \leq e^{1/b}, \qquad \forall x \in \mathbb{R}^n,
\]
which confirms that \( \tilde{\mathcal{P}}_{\mathcal{L}} \subseteq \tilde{\mathcal{P}}_{e^{-1/b},\, e^{1/b},\, h_{\mathcal{L}}} \).

\textbf{Experimental details of Figure \ref{fig: Laplace pure and functional local vs global}}

The original distribution is a mixture of four two-dimensional Laplace distributions, each with scale parameter \( b = 2 \) and means at \( (1,0) \), \( (0,1) \), \( (-1,0) \), and \( (0,-1) \), respectively, with equal weights \( \frac{1}{4} \). That is, the input distribution is given by

\[
P_{\textsf{input}} = \sum_{i=1}^4 \frac{1}{4} \, \mathcal{L}(m_i, 2),
\]
where the \( m_i \) are defined as above.

From Example~\ref{example: Gaussian Laplace universe}, we know that \( P_{\textsf{input}} \in \tilde{\mathcal{P}} \), where
\[
\tilde{\mathcal{P}} = \left\{ \sum_{i=1}^k \lambda_i \, \mathcal{L}(m_i, b) : k \in \mathbb{N},\, \lambda_i \geq 0,\, \sum_{i=1}^k \lambda_i = 1,\, \|m_i\|_1 \leq 1 \right\},
\]
and furthermore, \( \tilde{\mathcal{P}} \subseteq \tilde{\mathcal{P}}_{e^{-1/b},\, e^{1/b},\, h_{\mathcal{L}}} \), where \( h_{\mathcal{L}} \) denotes the density of a two-dimensional Laplace distribution with mean zero and scale parameter \( b = 2 \).

For \( b = 2 \), we define the local and global minimax universes as
\[
\tilde{\mathcal{P}}_{\textsf{local}} = \tilde{\mathcal{P}}_{\frac{1}{2},\, 2,\, h_{\mathcal{L}}}
\quad \text{and} \quad
\tilde{\mathcal{P}}_{\textsf{global}} = \tilde{\mathcal{P}}_{\frac{1}{6},\, 6,\, h_{\mathcal{L}}}.
\]

These universes are chosen such that \( \tilde{\mathcal{P}}_{\textsf{local}} \subseteq \tilde{\mathcal{P}}_{\textsf{global}} \), \( P_{\textsf{input}} \in \tilde{\mathcal{P}}_{\textsf{local}} \cap \tilde{\mathcal{P}}_{\textsf{global}} \),  and the ratio \( \frac{c_2 - c_1}{1 - c_1} \in \mathbb{N} \), as required by Assumption~\ref{assumption: norm}.
Moreover, we note that \( \tilde{\mathcal{P}}_{e^{-1/b},\, e^{1/b},\, h_{\mathcal{L}}} \subseteq \tilde{\mathcal{P}}_{\frac{1}{2},\, 2,\, h_{\mathcal{L}}} \), i.e., the condition in Assumption \ref{assumption: norm} is achieved by slightly adjusting the bounds \( c_1 \) and \( c_2 \) using floor and ceiling functions.

We evaluate the performance of minimax-optimal samplers on the input distribution under two LDP settings: comparing local minimax-optimal samplers with global minimax-optimal samplers under both \( \nu \)-GLDP and \( \varepsilon \)-LDP. In the pure-LDP setting, we compare the global minimax-optimal sampler of \citet{park2024exactly} with the local minimax-optimal sampler from Theorem~\ref{thm: local pure} for \( \varepsilon = 1 \). In the \( \nu \)-GLDP setting, we compare our global minimax-optimal sampler from Corollary~\ref{cor: special case GDP} with the local minimax-optimal sampler from Theorem~\ref{thm: local functional}, using the special case \( g = G_\nu \) with \( \nu = 1.5 \).

Under both settings, Figure~\ref{fig: Laplace pure and functional local vs global} demonstrates that the local minimax-optimal sampler better preserves the input distribution compared to the global minimax-optimal sampler, given the same level of privacy.

\subsection{Gaussian mixtures}\label{appendix: gaussian ring}

\textbf{Proof of $\tilde{\mathcal{P}}_{\mathcal{N}} \subseteq \tilde{\mathcal{P}}_{0,\, 1,\, h_{\mathcal{N}}}$ (Example \ref{example: Gaussian Laplace universe})} 

let $\tilde{\mathcal{P}}_{\mathcal{N}}  = \Big\{ \sum_{i=1}^k \lambda_i \mathcal{N}(m_i, \sigma^2 I_n)  :  \lambda_i \geq 0, \, \sum_{i=1}^k \lambda_i = 1, \, \|m_i\|_2 \leq 1 \Big\}$. We want to show that  $\tilde{\mathcal{P}}_{\mathcal{N}} \subseteq \tilde{\mathcal{P}}_{0,\, 1,\, h_{\mathcal{N}}}$ , where $h_{\mathcal{N}}(x) = \frac{1}{(2\pi \sigma^2)^\frac{n}{2}} \exp\big(-\frac{[\max(0, \|x\|_2 - 1)]^2}{2\sigma^2} \big)$.

Fix a component centre \(m\in\mathbb{R}^{n}\) with \(\|m\|_{2}\le1\).
Its Gaussian density at \(x\) is
\[
\mathcal{N}_{m,\sigma^{2}I_{n}}[x]
=\frac{1}{(2\pi\sigma^{2})^{n/2}}
  \exp\Bigl(
     -\tfrac{\|x-m\|_{2}^{2}}{2\sigma^{2}}
  \Bigr).
\]
By the triangle inequality,
\[
\|x-m\|_{2}
\;\ge\;
\bigl|\|x\|_{2}-\|m\|_{2}\bigr|
\;\ge\;
\max\bigl(0,\|x\|_{2}-1\bigr),
\]
because \(\|m\|_{2}\le1\).
Squaring and dividing by \(2\sigma^{2}\) yields
\[
\exp\Bigl(
  -\tfrac{\|x-m\|_{2}^{2}}{2\sigma^{2}}
\Bigr)
\;\le\;
\exp\Bigl(
  -\tfrac{[\max(0,\|x\|_{2}-1)]^{2}}{2\sigma^{2}}
\Bigr).
\]
Multiplying by the common normalizing constant
\((2\pi\sigma^{2})^{-n/2}\) gives
\[
\mathcal{N}_{m,\sigma^{2}I_{n}}[x]
\;\le\;
h_{\mathcal{N}}(x),\qquad\forall x\in\mathbb{R}^{n}.
\]

Now take an arbitrary mixture
\(P=\sum_{i=1}^{k}\lambda_{i}\,\mathcal{N}(m_{i},\sigma^{2}I_{n})
\in\tilde{\mathcal{P}}_{\mathcal{N}}\).
Its density is
\[
p(x)=\sum_{i=1}^{k}\lambda_{i}\,
      \mathcal{N}_{m_{i},\sigma^{2}I_{n}}[x].
\]
Because each component satisfies
\(\mathcal{N}_{m_{i},\sigma^{2}I_{n}}[x]\le h_{\mathcal{N}}(x)\) and the
weights obey \(\lambda_{i}\ge0,\ \sum_{i}\lambda_{i}=1\), we have
\[
0\le p(x)\le h_{\mathcal{N}}(x)\qquad\forall x\in\mathbb{R}^{n}.
\]
Therefore \(P\in\tilde{\mathcal{P}}_{0,\,1,\,h_{\mathcal{N}}}\), as
claimed.

It is worth noting that $h_\mathcal{N}$ is not a probability density. In order to make $h_{\mathcal{N}}$ a valid probability distribution, we normalize it by its integral and define
\[
c_2 = \int h_{\mathcal{N}}(x) \, dx, \quad \text{and} \quad h(x) = \frac{h_{\mathcal{N}}(x)}{c_2}.
\]
Accordingly, we have \( \tilde{\mathcal{P}}_{0,1,h_{\mathcal{N}}} = \tilde{\mathcal{P}}_{0,c_2,h} \).

\subsection{Finite sample space numerical results under pure LDP}\label{appendix: finite space}

\textbf{Experimental details of Figure \ref{fig: finite worst-case k = 20}}

We compare local and global minimax-optimal samplers in the finite setting \( \mathcal{X} = [k] \), where the global universe is defined as \( \tilde{\mathcal{P}}_{\textsf{global}} = \mathcal{P}([k]) = \tilde{\mathcal{P}}_{0, k, \mu_k} \), with \( \mu_k \) denoting the uniform distribution over \([k]\). The local neighborhood is specified as \( \tilde{\mathcal{P}}_{\textsf{local}} = \mathcal{N}_\gamma(\mu_k) = \tilde{\mathcal{P}}_{\frac{1}{\gamma},\, \gamma,\, \mu_k} \), where \( \gamma = \frac{k}{2} - 1 \), consistent with the finite-space version of Theorem~\ref{thm: local pure}.

For the selected values \( k \in \{10, 20, 100\} \), it is easy to verify that \( \gamma \in \mathbb{N} \), and the uniform distribution \( \mu_k \) is \( (\alpha, \frac{1}{\alpha}, 1) \)-decomposable, where
\[
\alpha = \frac{1 - c_1}{c_2 - c_1}
= \frac{1 - \frac{1}{\gamma}}{\gamma - \frac{1}{\gamma}}
= \frac{1}{\gamma + 1}
= \frac{2}{k}.
\]
Therefore, the decomposability condition in Theorem~\ref{thm: local pure} is satisfied, and we can apply its finite-space version by setting \( \mathcal{X} = [k] \).

The local minimax risk is then given by
\[
\mathcal{R}\big(\mathcal{Q}_{[k], \mathcal{P}([k]), \varepsilon}, N_\gamma(\mu_k), f\big) 
= \frac{1 - r_1}{r_2 - r_1} \, f(r_2) + \frac{r_2 - 1}{r_2 - r_1} \, f(r_1),
\]
where
\[
b \coloneqq \frac{\gamma - \tfrac{1}{\gamma}}{(e^\varepsilon - 1)(1 - \tfrac{1}{\gamma}) + \gamma - \tfrac{1}{\gamma}} 
= \frac{\gamma + 1}{\gamma + e^\varepsilon}, \quad
r_1 \coloneqq \frac{1}{\gamma b} = \frac{e^\varepsilon + \gamma}{\gamma(\gamma + 1)}, \quad
r_2 \coloneqq \frac{\gamma}{b e^\varepsilon} = \frac{\gamma(e^\varepsilon + \gamma)}{e^\varepsilon(\gamma + 1)}.
\]

On the other hand, the global minimax risk is given by~\citep[Theorem 3.1]{park2024exactly}:
\[
\mathcal{R}\big(\mathcal{Q}_{[k], \mathcal{P}([k]), \varepsilon}, \mathcal{P}([k]), f\big)
= \frac{e^{\varepsilon}}{e^{\varepsilon} + k - 1} 
  f\left( \frac{e^{\varepsilon} + k - 1}{e^{\varepsilon}} \right)
+ \frac{k - 1}{e^{\varepsilon} + k - 1} f(0).
\]

Therefore, in Figure~\ref{fig: finite worst-case k = 20}, we compare the theoretical worst-case \( f \)-divergence losses achieved by the local minimax-optimal sampler and the global minimax-optimal sampler:
\[
\mathcal{R}\big(\mathcal{Q}_{[k], \mathcal{P}([k]), \varepsilon}, N_\gamma(\mu_k), f\big)
\quad \text{and} \quad
\mathcal{R}\big(\mathcal{Q}_{[k], \mathcal{P}([k]), \varepsilon}, \mathcal{P}([k]), f\big).
\]

\textbf{Additional numerical results}

We replicated the procedure used in Section~\ref{experiment: finite} to generate Figure~\ref{fig: finite worst-case k = 20} , but now with sample sizes \(k = 10\) and \(k = 100\). As shown in Figures~\ref{fig: finite worst-case k = 10} and~\ref{fig: finite worst-case k = 100} and comparing with Figure \ref{fig: finite worst-case k = 20}, the local minimax sampler consistently attains a smaller worst-case loss than the global minimax sampler across different $\eps$ values. Also, the performance gap decreases for larger values of $k$.

\begin{figure}[ht]
  \centering
  \includegraphics[width=0.7\linewidth]{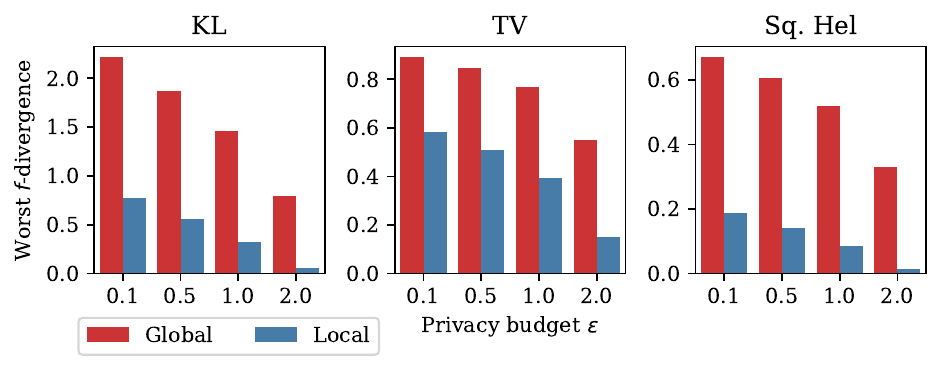}
  \caption{Theoretical worst-case 
$f$-divergences of global and local minimax samplers under the pure LDP setting with uniform reference distribution $\mu_k$ over finite space ($k = 10$).
}
  \label{fig: finite worst-case k = 10}
\end{figure}

\begin{figure}[ht]
  \centering
  \includegraphics[width=0.7\linewidth]{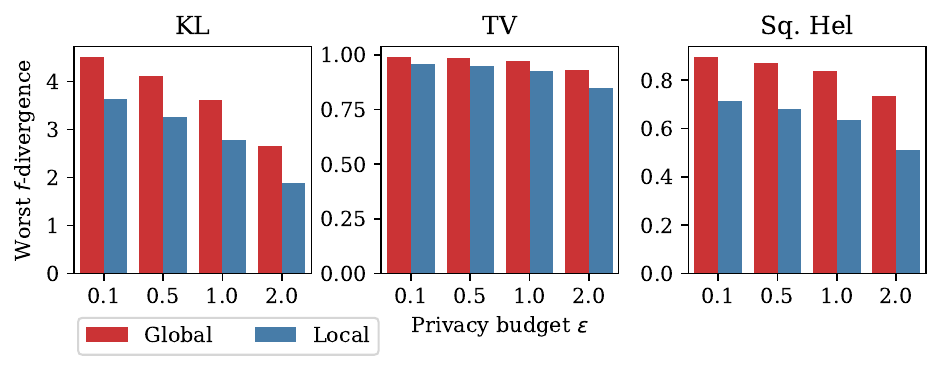}
  \caption{Theoretical worst-case 
$f$-divergences of global and local minimax samplers under the pure LDP setting with uniform reference distribution $\mu_k$ over finite space ($k = 100$).
}
  \label{fig: finite worst-case k = 100}
\end{figure}

\subsection{Experimental details of continuous sample space}\label{appendix: details of continuous}

In the continuous setting with \( \mathcal{X} = \mathbb{R} \), we fix the universe \( \tilde{\mathcal{P}}_{\textsf{local}} \) and evaluate the empirical worst-case \( f \)-divergence over 100 randomly generated client distributions \( P_1, \ldots, P_{100} \in \tilde{\mathcal{P}}_{\textsf{local}} \). Each \( P_j \) represents a client and is constructed as a mixture of a random number of one-dimensional Laplace components with scale parameter \( b = 1 \). The full generation procedure is described below.

We define the mixture class of Laplace distributions with fixed scale parameter \( b \) as
\[
\tilde{\mathcal{P}}_{\mathcal{L}} = \left\{ \sum_{i=1}^k \lambda_i \, \mathcal{L}(m_i, b) : k \in [K],\, \lambda_i \geq 0,\, \sum_{i=1}^k \lambda_i = 1,\, |m_i| \leq 1 \right\},
\]
where \( \mathcal{L}(m, b) \) denotes the one-dimensional Laplace distribution with mean \( m \in \mathbb{R} \) and scale parameter \( b > 0 \). To prevent an unbounded number of components in each mixture, we impose an upper bound \( K \) on the number of Laplace components per client.

Each \( P_j \in \tilde{\mathcal{P}}_{\mathcal{L}} \) is generated by randomly selecting \( k \), \( \lambda_i \), and \( m_i \) as follows:  
First, sample \( \tilde{k} \) from a Poisson distribution with mean \( k_0 \), and set \( k = \min(\tilde{k} + 1, K) \). Then, sample each \( m_1, \ldots, m_k \) independently from the uniform distribution on \( [-1, 1] \), and sample weights \( (\lambda_1, \ldots, \lambda_k) \) from the uniform distribution on \( \mathcal{P}([k]) \). In this experiment, to maintain consistency with~\citet{park2024exactly}, we use \( K = 10 \) and \( k_0 = 2 \).

We define the local and global universes as
\[
\tilde{\mathcal{P}}_{\textsf{local}} = \tilde{\mathcal{P}}_{1/3,\, 3,\, h_{\mathcal{L}}}
\quad \text{and} \quad
\tilde{\mathcal{P}}_{\textsf{global}} = \tilde{\mathcal{P}}_{1/9,\, 9,\, h_{\mathcal{L}}},
\]
where \( h_{\mathcal{L}} \) denotes the density of the Laplace distribution with mean zero and scale \( b = 1 \).

We evaluate the empirical worst-case divergence of each sampler using the maximum
\[
\max_{j \in [100]} D_f\bigl(P_j \,\|\, \mathbf{Q}(P_j)\bigr).
\]
The local minimax sampler is instantiated from Theorem~\ref{thm: local pure}, while the global minimax sampler corresponds to the optimal sampler from~\citep[Theorem 3.3]{park2024exactly}.

\subsection{Finite sample space numerical results under  GLDP}  \label{appendix: experiment GDP local global discrete}

In this section, we adopt the same experimental setup as in Section~\ref{experiment: finite} to evaluate the worst-case $f$-divergence of the local and global minimax samplers under $\nu$-GLDP constraints. To this end, we follow the same procedure using the same global and local universes. However, the computation of minimax risk differs: for the global minimax risk, we use the optimal value corresponding to the optimal sampler characterized in Corollary~\ref{cor: special case GDP}, while for the local minimax risk, we instantiate  Theorem~\ref{thm: local functional} with $g = G_\nu$ (both the finite sample space version of the results).

\begin{figure}[ht]
  \centering
  \includegraphics[width=0.7\linewidth]{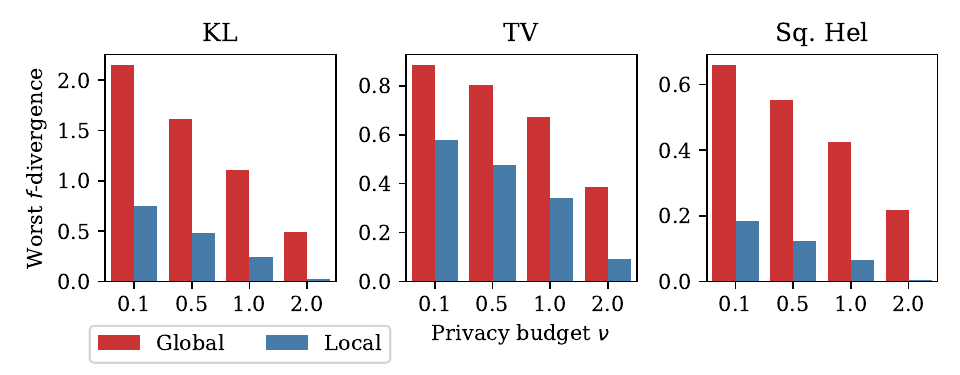}
  \caption{Theoretical worst-case 
$f$-divergences of global and local minimax samplers under the $\nu$-GLDP setting with uniform reference distribution $\mu_k$ over finite space ($k = 10$).
}
  \label{fig: finite worst-case GDP k = 10}
\end{figure}

\begin{figure}[ht]
  \centering
  \includegraphics[width=0.7\linewidth]{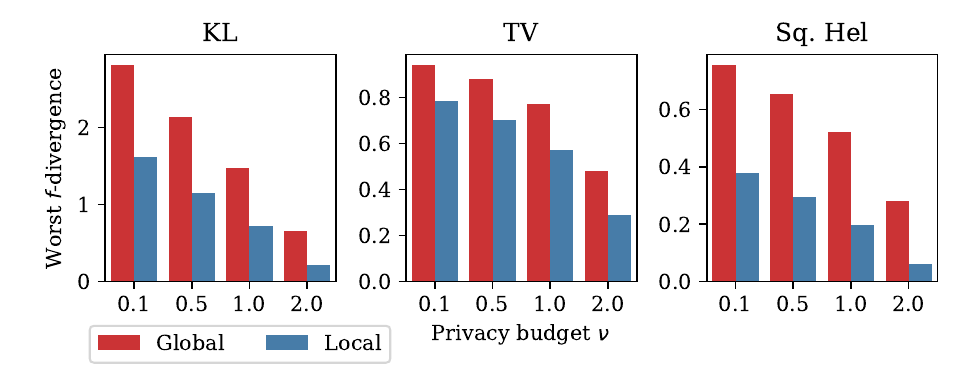}
  \caption{Theoretical worst-case 
$f$-divergences of global and local minimax samplers under the $\nu$-GLDP setting with uniform reference distribution $\mu_k$ over finite space ($k = 20$).
}
  \label{fig: finite worst-case GDP k = 20}
\end{figure}

\begin{figure}[ht]
  \centering
  \includegraphics[width=0.7\linewidth]{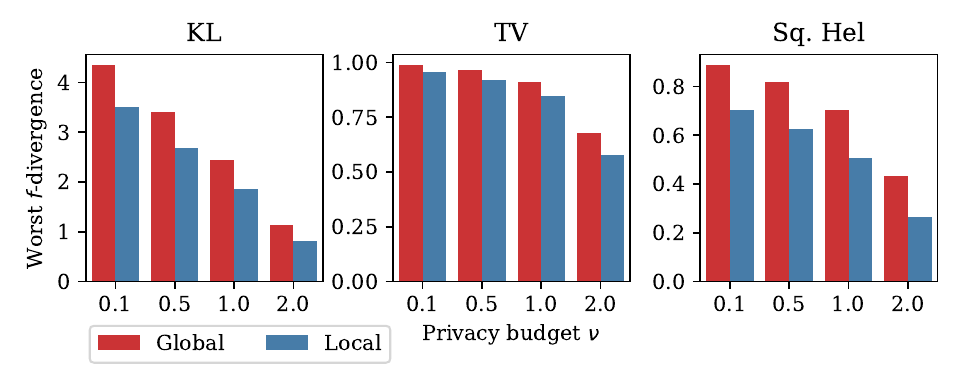}
  \caption{Theoretical worst-case 
$f$-divergences of global and local minimax samplers under the $\nu$-GLDP setting with uniform reference distribution $\mu_k$ over finite space ($k = 100$).
}
  \label{fig: finite worst-case GDP k = 100}
\end{figure}

We compare the worst-case $f$-divergence under $\nu$-GDP for local and global minimax-optimal samplers across different values of $\nu \in \{0.1, 0.5, 1, 2\}$. The numerical results in Figures~\ref{fig: finite worst-case GDP k = 10}, \ref{fig: finite worst-case GDP k = 20}, and \ref{fig: finite worst-case GDP k = 100} demonstrate that for various sample sizes $k \in \{10, 20, 100\}$, the local minimax-optimal sampler consistently achieves lower minimax risk than the global sampler across all privacy parameter values.

\subsection{Continuous sample space numerical results under  GLDP} \label{appendix: experiment GDP local global continuous}

We follow the same experimental procedure described in Appendix~\ref{appendix: details of continuous}, with the only difference being that we now compare the local and global minimax-optimal samplers under $\nu$-GDP instead of $\varepsilon$-LDP. The local and global universes, as well as the distribution generation process, remain unchanged. For the global minimax sampler, we use the construction provided in Corollary~\ref{cor: special case GDP}, while the local minimax sampler is instantiated from Theorem~\ref{thm: local functional} with \( g = G_\nu \). We conduct the comparison across various values of \( \nu \in \{0.1,0.5,1,2\} \), evaluating the resulting samplers using three $f$-divergences: KL divergence, total variation distance, and squared Hellinger distance.

\begin{figure}[ht]
  \centering
  \includegraphics[width=0.7\linewidth]{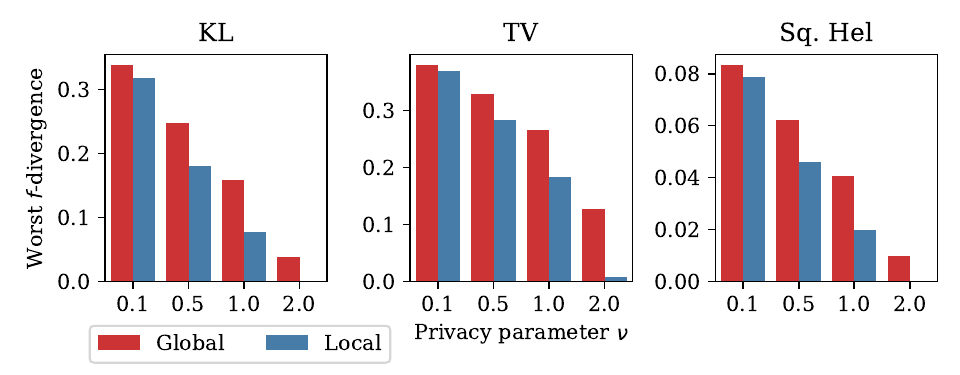}
  \caption{Empirical worst-case 
$f$-divergences of global and local minimax samplers under $\nu$-GLDP setting, over 100 experiments on a 1-D Laplace mixture.}
  \label{fig: continuous worst-case GDP}
\end{figure}

As illustrated in Figure~\ref{fig: continuous worst-case GDP}, in the $\nu$-GDP setting—similar to the $\varepsilon$-LDP case—the local minimax sampler consistently achieves a smaller worst-case $f$-divergence than the global minimax sampler across all three divergences and privacy parameter values.

\subsection{Additional numerical results on continuous sample space }\label{appendix: high dimension laplace}
We extend the experiment from Section~\ref{subsec: experiment 1d continuous} to the two-dimensional setting \( \mathcal{X} = \mathbb{R}^2 \). Figure~\ref{fig: continuous worst-case 2-D} compares the worst-case empirical \( f \)-divergence between the global sampler of~\citet{park2024exactly} and the local sampler described in Theorem~\ref{thm: local functional} for $g = g_\eps$.

We follow the same experimental procedure as in Section~\ref{subsec: experiment 1d continuous}, detailed in Appendix~\ref{appendix: details of continuous}, to generate 100 client distributions. Each distribution is constructed as a mixture of 2-D Laplace components with fixed scale parameter \( b = 1 \). The only difference from the one-dimensional case is the extension to 2-D Laplace mixtures.

The mixture class of 2-D Laplace distributions with scale parameter \( b \) is defined as
\[
\tilde{\mathcal{P}}_{\mathcal{L}} = \left\{ \sum_{i=1}^k \lambda_i \, \mathcal{L}(m_i, b) : k \in [K],\, \lambda_i \geq 0,\, \sum_{i=1}^k \lambda_i = 1,\, \|m_i\|_1 \leq 1 \right\},
\]
where \( \mathcal{L}(m, b) \) denotes a two-dimensional Laplace distribution with mean \( m \in \mathbb{R}^2 \) and scale \( b > 0 \). To control the complexity of each mixture, we impose an upper bound \( K \) on the number of components per distribution.

Each distribution \( P_j \in \tilde{\mathcal{P}}_{\mathcal{L}} \) is generated by randomly sampling \( k \), \( \lambda_i \), and \( m_i \) as follows:  
First, sample \( \tilde{k} \sim \text{Poisson}(k_0) \), and set \( k = \min(\tilde{k} + 1, K) \). Then, sample each \( m_i \in \mathbb{R}^2 \) independently from the uniform distribution on the \(\ell_1\) ball \( \{ x \in \mathbb{R}^2 : \|x\|_1 \leq 1 \} \), and draw the weights \( (\lambda_1, \ldots, \lambda_k) \) uniformly from the probability simplex \( \mathcal{P}([k]) \). In this experiment, following the setup in~\citet{park2024exactly}, we use \( K = 10 \) and \( k_0 = 2 \).

The local and global universes are defined as
$
\tilde{\mathcal{P}}_{\textsf{local}} = \tilde{\mathcal{P}}_{1/3,\, 3,\, h_{\mathcal{L}}}
\quad $ and $
\tilde{\mathcal{P}}_{\textsf{global}} = \tilde{\mathcal{P}}_{1/9,\, 9,\, h_{\mathcal{L}}}$,
where \( h_{\mathcal{L}} \) denotes the density of the two-dimensional Laplace distribution with mean zero and scale parameter \( b = 1 \).


\begin{figure}[ht]
  \centering
  \includegraphics[width=0.7\linewidth]{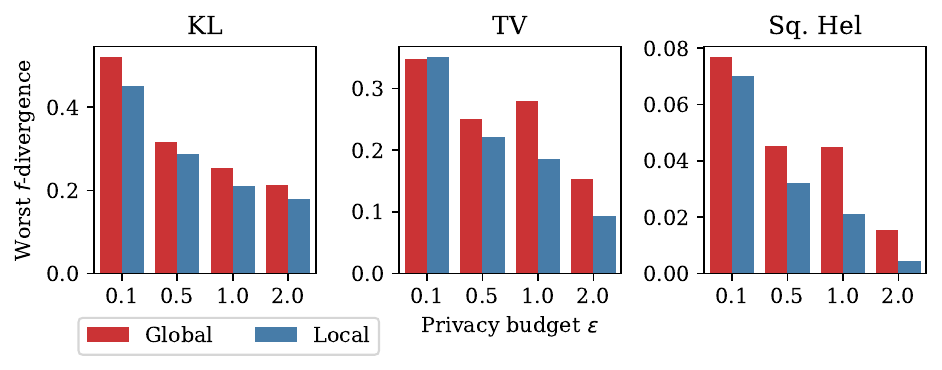}
  \caption{Empirical worst-case 
$f$-divergences of global and local minimax samplers under the pure LDP setting, over 100 experiments on a 2-D Laplace mixture.
}
  \label{fig: continuous worst-case 2-D}
\end{figure}

In the two-dimensional case, similar to the one-dimensional setting, the local sampler outperforms the global sampler for nearly all fixed values of the privacy parameter.

\section{Instructions for reproducing results}\label{appendix: reproduce}

In this appendix, we provide instructions for reproducing the experiments and figures in the paper. For a detailed description of the code, please refer to the provided file \texttt{README.md}. For tasks that require substantial runtime, we specify the running times. Tasks that complete in less than 5 seconds are excluded from such reporting.

All experiments were conducted on a system running Ubuntu 22.04.4 LTS, equipped with an Intel(R) Xeon(R) CPU @ 2.20GHz and 16GB of RAM.

\subsection{Instructions for reproducing Figure \ref{fig: Laplace pure and functional local vs global}}

From the repository root, run 
\texttt{python -m experiments.exp\_LapMixture\_visual} 
to generate the output distributions, and then 
\texttt{python -m plotting.plot\_LapMixture\_visual} 
to produce Figure~\ref{fig: Laplace pure and functional local vs global}.

The measured running time in our environment for applying the minimax-optimal sampler to the original input distribution is approximately 300 seconds.

\subsection{Instructions for reproducing results for finite space}

The finite space results are organized into six figures, grouped into two categories. 
Figures \ref{fig: finite worst-case k = 20}, \ref{fig: finite worst-case k = 10}, and \ref{fig: finite worst-case k = 100} 
present the worst-case divergences under pure LDP for different values of \( k \in \{10, 20, 100\} \). 
In contrast, Figures~\ref{fig: finite worst-case GDP k = 10}, \ref{fig: finite worst-case GDP k = 20}, 
and \ref{fig: finite worst-case GDP k = 100} show the corresponding results under \( \nu \)-GLDP for the same values of \( k \).

To generate Figures~\ref{fig: finite worst-case k = 20}, \ref{fig: finite worst-case k = 10}, and \ref{fig: finite worst-case k = 100}, 
use the script \texttt{plotting/plot\_finite\_pure.py} with the \texttt{--k} argument to specify the desired value of \( k \). 
For example, the following commands can be used to generate the respective plots:
\begin{verbatim}
python -m plotting.plot_finite_pure --k 20
python -m plotting.plot_finite_pure --k 10
python -m plotting.plot_finite_pure --k 100
\end{verbatim}

To produce Figures~\ref{fig: finite worst-case GDP k = 10}, \ref{fig: finite worst-case GDP k = 20}, 
and \ref{fig: finite worst-case GDP k = 100}, run the script \texttt{plotting/plot\_finite\_GLDP.py} with the corresponding \( k \) values:
\begin{verbatim}
python -m plotting.plot_finite_GLDP --k 10
python -m plotting.plot_finite_GLDP --k 20
python -m plotting.plot_finite_GLDP --k 100
\end{verbatim}

\subsection{Instructions for reproducing results for continuous space}

To reproduce Figure~\ref{fig: continuous worst-case}, first run the script \texttt{experiments/exp\_1DLaplaceMix\_pure.py} with different values of the privacy parameter \( \eps \). The following commands correspond to \( \eps \in \{0.1, 0.5, 1.0, 2.0\} \), respectively:

\begin{verbatim}
python -m experiments.exp_1DLaplaceMix_pure --eps 0.1 --scale 1 --seed 1 
python -m experiments.exp_1DLaplaceMix_pure --eps 0.5 --scale 1 --seed 2 
python -m experiments.exp_1DLaplaceMix_pure --eps 1.0 --scale 1 --seed 3 
python -m experiments.exp_1DLaplaceMix_pure --eps 2.0 --scale 1 --seed 4 
\end{verbatim}

These scripts can be executed independently and in any order, or run in parallel. In our environment, running all four in parallel required approximately 600 seconds. After completion, run \texttt{python -m plotting.plot\_1DLaplaceMix\_pure} to generate the final plots.

To reproduce Figure~\ref{fig: continuous worst-case GDP}, execute the script \texttt{experiments/exp\_1DLaplaceMix\_GLDP.py} with various values of the privacy parameter \( \nu \). The following commands correspond to \( \nu = 0.1, 0.5, 1.0, \) and \( 2.0 \), respectively:

\begin{verbatim}
python -m experiments.exp_1DLaplaceMix_GLDP --nu 0.1 --scale 1 --seed 1 
python -m experiments.exp_1DLaplaceMix_GLDP --nu 0.5 --scale 1 --seed 2
python -m experiments.exp_1DLaplaceMix_GLDP --nu 1.0 --scale 1 --seed 3 
python -m experiments.exp_1DLaplaceMix_GLDP --nu 2.0 --scale 1 --seed 4 
\end{verbatim}

These scripts can be executed in any order or in parallel. In our environment, running all four in parallel required approximately 80 seconds. After completion, run \texttt{python -m plotting.plot\_1DLaplaceMix\_GLDP} to generate the final figure.

To reproduce Figure~\ref{fig: continuous worst-case 2-D}, run the script \texttt{experiments/exp\_nDLaplaceMix\_pure.py} with different values of the privacy parameter \( \eps \). The following commands correspond to \( \eps \in \{0.1, 0.5, 1.0, 2.0\} \):

\begin{verbatim}
python -m experiments.exp_nDLaplaceMix_pure --eps 0.1 --seed 1 --dim 2
python -m experiments.exp_nDLaplaceMix_pure --eps 0.5 --seed 2 --dim 2
python -m experiments.exp_nDLaplaceMix_pure --eps 1.0 --seed 3 --dim 2
python -m experiments.exp_nDLaplaceMix_pure --eps 2.0 --seed 4 --dim 2
\end{verbatim}

These can be run independently or in parallel. In our environment, running all four in parallel required approximately 40 seconds. Once the first step is completed, run \texttt{python -m plotting.plot\_nDLaplaceMix\_pure ----dim 2} to generate the final figure.



\fi

\end{document}